\documentclass[times,sort&compress,3p]{elsarticle}
\journal{Journal of Multivariate Analysis}
\usepackage{amsmath, amssymb, enumerate, amsthm}

\usepackage{times}
\usepackage{bm}
\usepackage{natbib}
\usepackage{graphicx}
\usepackage{subcaption}
\usepackage{textcomp}
\usepackage{color}
\usepackage[plain,noend]{algorithm2e}
\usepackage{lineno,hyperref}

\DeclareMathOperator{\row}{row}

\newtheorem{theorem}{Theorem}
\newtheorem{lemma}{Lemma}

\newtheorem{corollary}{Corollary}

\DeclareMathOperator{\var}{\rm var}
\DeclareMathOperator{\corr}{\rm corr}
\DeclareMathOperator{\cov}{\rm cov}

\newcommand{\argmax}{\operatornamewithlimits{\rm argmax}}
\newcommand{\argmin}{\operatornamewithlimits{\rm argmin}}

\begin{document}

\begin{frontmatter}

\title{Angle-based joint and individual variation explained}

\author[]{Qing Feng}
\ead{qing.feng1014@gmail.com}

\author[]{Meilei Jiang\corref{cor1}}
\cortext[cor1]{Corresponding author.}
\ead{jiangm@live.unc.edu}

\author[]{Jan Hannig}
\ead{jan.hannig@unc.edu}

\author[]{J. S. Marron}
\ead{marron@unc.edu}

\address{Department of Statistics and Operations Research, University of North Carolina at Chapel Hill, Chapel Hill, NC 27599, USA}

\begin{abstract}
Integrative analysis of disparate data blocks measured on a common set of experimental subjects is a major challenge in modern data analysis. This data structure naturally motivates the simultaneous exploration of the joint and individual variation within each data block resulting in new insights. For instance, there is a strong desire to integrate the multiple genomic data sets in The Cancer Genome Atlas to characterize the common and also the unique aspects of cancer genetics and cell biology for each source. In this paper we introduce Angle-Based Joint and Individual Variation Explained  capturing both joint and individual variation within each data block. This is a major improvement over earlier approaches to this challenge in terms of a new conceptual understanding, much better adaption to data heterogeneity and a fast linear algebra computation. Important mathematical contributions are the use of score subspaces as the principal descriptors of variation structure and the use of perturbation theory as the guide for variation segmentation. This leads to an exploratory data analysis method which is insensitive to the heterogeneity among data blocks and does not require separate normalization. An application to cancer data reveals different behaviors of each type of signal in characterizing tumor subtypes. An application to a mortality data set reveals interesting historical lessons. Software and data are available at GitHub \url{https://github.com/MeileiJiang/AJIVE_Project}. 
\end{abstract}


\end{frontmatter}

\section{Introduction}
\label{c:jive-s:intro}
A major challenge in modern data analysis is data integration, combining diverse information from disparate data sets measured on a common set of experimental subjects. Simultaneous variation decomposition has been useful in many practical applications. For example, \citet{kuhnle2011integration}, \citet{lock2013bayesian}, and \citet{mo2013pattern} performed integrative clustering on multiple sources to reveal novel and consistent cancer subtypes based on understanding of joint and individual variation. The Cancer Genome Atlas (TCGA)~\citep{cancer2012comprehensive} provides a prototypical example for this problem. TCGA contains disparate data types generated from high-throughput technologies. Integration of these is fundamental for studying cancer on a molecular level. Other types of application include analysis of multi-source metabolomic data~\citep{kuligowski2015analysis}, extraction of commuting patterns in railway networks~\citep{jere2014extracting}, recognition of brain-€"computer interface~\citep{zhang2015ssvep}, etc.	

A unified and insightful understanding of the set of data blocks is expected from simultaneously exploring the joint variation representing the inter-block associations and the individual variation specific to each block.
\citet{lock2013joint} formulated this challenge into a matrix decomposition problem. Each data block is decomposed into three matrices modeling different types of variation, including a low-rank approximation of the joint variation across the blocks, low-rank approximations of the individual variation for each data block, and residual noise. Definitions and constraints were proposed for the joint and individual variation together with a method named JIVE; see \url{https://genome.unc.edu/jive/} and \citet{OConnell:2016du} for Matlab and \textsf{R} implementations of JIVE, respectively.

JIVE was a promising framework for studying multiple data matrices. However, \citet{lock2013joint} algorithm and its implementation was iterative (thus slow) and performed rank selection based on a permutation test. It had no guarantee of achieving a solution that satisfied the definitions of JIVE, especially in the case of some correlation between individual components. The example in Figure~\ref{fig:jive:toyoldjiveoutput} in \ref{s:toy_details} shows that this can be a serious issue. An important related algorithm named COBE was developed by~\citet{zhou2015group}. COBE considers a JIVE-type decomposition as a quadratic optimization problem with restrictions to ensure identifiability. While COBE removed many of the shortcomings of the original JIVE, it was still iterative and often required longer computation time than the~\citet{lock2013joint} algorithm. Neither~\citet{zhou2015group} nor~\citet{lock2013joint} provided any theoretical basis for selection of a thresholding parameter used for separation of the joint and individual components. 

A novel solution, {\em Angle-based Joint and Individual Variation Explained (AJIVE)}, is proposed here for addressing this matrix decomposition problem. It provides an efficient {\em angle-based algorithm} ensuring an identifiable decomposition and also an insightful new interpretation of extracted variation structure. The key insight is the use of row spaces, i.e., a focus on scores, as the principal descriptor of the joint and individual variation, assuming columns are the $n$ data objects, e.g., vectors of measurements on patients. This focuses the methodology on variation patterns across data objects, which gives straightforward definitions of the components and thus provides identifiability. These variation patterns are captured by the  \emph{score subspaces} of $\mathbb{R}^n$. Segmentation of joint and individual variation is based on studying the relationship between these score subspaces and using perturbation theory to quantify noise effects~\citep{stewart1990matrix}.

The main idea of AJIVE is illustrated in the flowchart of Figure~\ref{fig:flow}. AJIVE works in three steps. First we find a low-rank approximation of each data block (shown as the far left color blocks in the flowchart) using SVD. This is depicted (using blocks with colored dashed line boundaries) on the left side of Figure~\ref{fig:flow} with the black arrows signifying thresholded SVD. Next, in the middle of the figure, SVD of the concatenated bases of row spaces from the first step (the gray blocks with colored boundaries) gives a joint row space (the gray box next to the circle), using a mathematically rigorous threshold derived using perturbation theory in Section~\ref{c:jive-s:method-subsec:step2}. This SVD is a natural extension of Principal Angle Analysis, which is also closely related to the multi-block extension of Canonical Correlation Analysis~\citep{nielsen2002multiset} as well as to the flag means of the row spaces~\citep{draper2014flag}; see Section~\ref{s:CCAPPA_optimization} for details. Finally, the joint and individual space approximations are found using projection of the joint row space and its orthogonal complements on the data blocks as shown as colored boundary gray squares on the right with the three joint components at the top and the individual components at the bottom.
\begin{figure}[htp!]
	\centering
	\includegraphics[scale = 0.52]{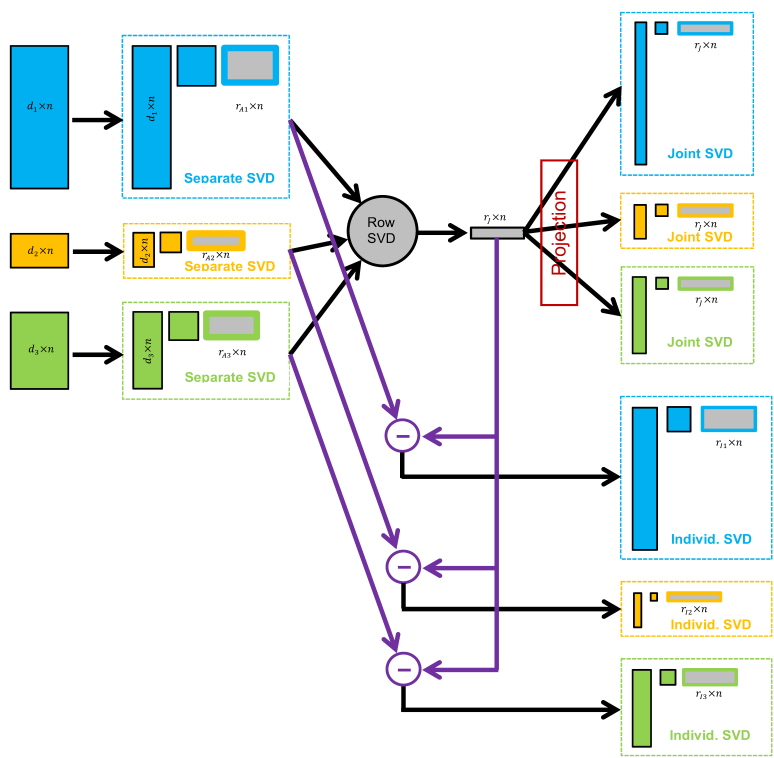}
	\caption{Flow chart demonstrating the main steps of AJIVE. First low-rank approximation of each data block is obtained on the right. Then in the middle joint structure between the low-rank approximations is extracted using SVD of the stacked row basis matrices. Finally, on the right, the joint components (upper) are obtained by projection of each data block onto the joint basis (middle) and the individual components (lower) come from orthonormal basis subtraction.}
	\label{fig:flow}
\end{figure}

Using score subspaces to describe variation contained in a matrix not only empowers the interpretation of analysis but also improves understanding of the problem and the efficiency of the algorithm. An identifiable decomposition can now be obtained with all definitions and constraints satisfied even in situations when individual spaces are somewhat correlated. Moreover, the need to select a tuning parameter used to distinguish joint and individual variation is eliminated based on theoretical justification using perturbation theory.  A consequence is an algorithm which uses a fast built-in singular value decomposition to replace lengthy iterative algorithms. For the example in Section~\ref{c:jive-s:intro-subsec:toy}, implemented in Matlab, the computational time of AJIVE (10.8 seconds) is about 11 times faster than the old JIVE (121 seconds) and 39 times faster than COBE (422 seconds). The computational advantages of AJIVE get even more pronounced on data sets with higher dimensionality and more complex heterogeneity such as the TCGA data analyzed in Section~\ref{subsec:TCGA}. For a very successful application of AJIVE on integrating fMRI imaging and behavioral data, see~\citet{yu2017jive}.

Other methods that aim to study joint variation patterns and/or individual variation patterns have also been developed. \citet{westerhuis1998analysis} discuss two types of methods. One main type extends traditional Principal Component Analysis (PCA), including Consensus PCA and Hierarchical PCA first introduced by \citet{wold1987multi, wold1996hierarchical}. An overview of extended PCA methods is discussed in~\citet{smilde2003framework}. \citet{abdi2013multiple} discuss a multiple block extension of PCA called multiple factor analysis. This type of method computes the block scores, block loadings, global loadings and global scores.

The other main type of method is extensions of Partial Least Squares (PLS)~\citep{wold1985partial} or Canonical Correlation Analysis (CCA)~\citep{hotelling1936relations} that seek associated patterns between the two data blocks by maximizing covariance/correlation. For example,~\citet{wold1996hierarchical} introduced multi-block PLS and hierarchical PLS (HPLS) and~\citet{trygg2003o2} proposed \emph{O2-PLS} to better reconstruct joint signals by removing structured individual variation. A multi-block extension can be found in \citet{lofstedt2013global}.

\citet{Yang:2015dt} provide a very nice integrative joint and individual component analysis based on non-negative matrix factorization. \citet{ray2014bayesian} do integrative analysis using factorial models in the Bayesian setting. \citet{schouteden2013sca, schouteden:2014ce} propose a method called DISCO-SCA that is a low-rank approximation with rotation to sparsity of the concatenated data matrices. 

A connection between extended PCA and extended PLS methods is discussed in~\citet{hanafi2011connections}. Both types of methods provide an integrative analysis by taking the inter-block associations into account. These papers recommend use of normalization to address potential scale heterogeneity, including normalizing by the Frobenius norm, or the largest singular value of each data block, etc. However, there are no consistent criteria for normalization and some of these methods have convergence problems. An important point is that none of these approaches provide simultaneous decomposition highlighting joint and individual modes of variation with the goal of contrasting these to reveal new insights.

\subsection{Toy example}
\label{c:jive-s:intro-subsec:toy}

We give a toy example to provide a clear view of multiple challenges brought by potentially very disparate data blocks. This toy example has two data blocks, $X$ ($100\times 100$) and $Y$ ($\mbox{10,000} \times 100$), with patterns corresponding to joint and individual structures. Such data set sizes are reasonable in modern genetic studies, as seen in Section~\ref{subsec:TCGA}. Figure~\ref{fig:jive:toyrawdata} shows colormap views of these matrices, with the value of each matrix entry colored according to the color bar at the bottom of each subplot. The data have been simulated so expected row means are $0$. Therefore mean centering is not necessary in this case. A careful look at the color bar scalings shows the values are almost four orders of magnitude larger for the top matrices. Each column of these matrices is regarded as a common data object and each row is considered as one feature. The number of features is also very different as labeled in the $y$-axis. Each of the two raw data matrices, $X$ and $Y$ in the left panel of Figure~\ref{fig:jive:toyrawdata}, is the sum of joint, individual and noise components shown in the other panels. 

The joint variation for both blocks, second column of panels, presents a contrast between the left and right halves of the data matrix, thus having the same rank-1 score subspace. If for example the left half columns were male and right half were female, this joint variation component could be interpreted as a contrast of gender groups which exists in both data blocks for those features where color appears. 

The $X$ individual variation, third column of panels, partitions the columns into two other groups of size $50$ that are arranged so the row space is orthogonal to that of the joint score subspace. The individual signal for $Y$ contains two variation components, each driven by half of the features. The first component, displayed in the first $5000$ rows, partitions the columns into three groups. The other component is driven by the bottom half of the features and partitions the columns into two groups, both with row spaces orthogonal to the joint. Note that these two individual score subspaces for $X$ and $Y$ are different but not orthogonal. The smallest angle between the individual subspaces is~$45$\textdegree. 

This example presents several challenging aspects, which also appear in real data sets such as TCGA, as studied in Section~\ref{subsec:TCGA}. One is that both the values and the number of the features are orders of magnitude different between $X$ and $Y$. Another important challenge is that because the individual spaces are not orthogonal, the individual signals are correlated.  Correctly handling these in an integrated manner is a major improvement of AJIVE over earlier methods. In particular, normalization is no longer an issue because AJIVE only uses the low-rank initial {\it scores} (represented as the gray boxes in the SVD shown on the left of Figure~\ref{fig:flow}), while signal power appears in the central subblocks and the features only in the left subblocks. Appropriate handing of potential correlation among individual components is done using perturbation theory in Section~\ref{c:jive-s:method}.

The noise matrices, the right panels of Figure~\ref{fig:jive:toyrawdata}, are standard Gaussian random matrices (scaled by 5000 for $X$) which generates a noisy context for both data blocks and thus a challenge for analysis, as shown in the left panels of Figure~\ref{fig:jive:toyrawdata}.

\begin{figure}[htp!]
	\begin{center}
		\includegraphics[scale = 0.75]{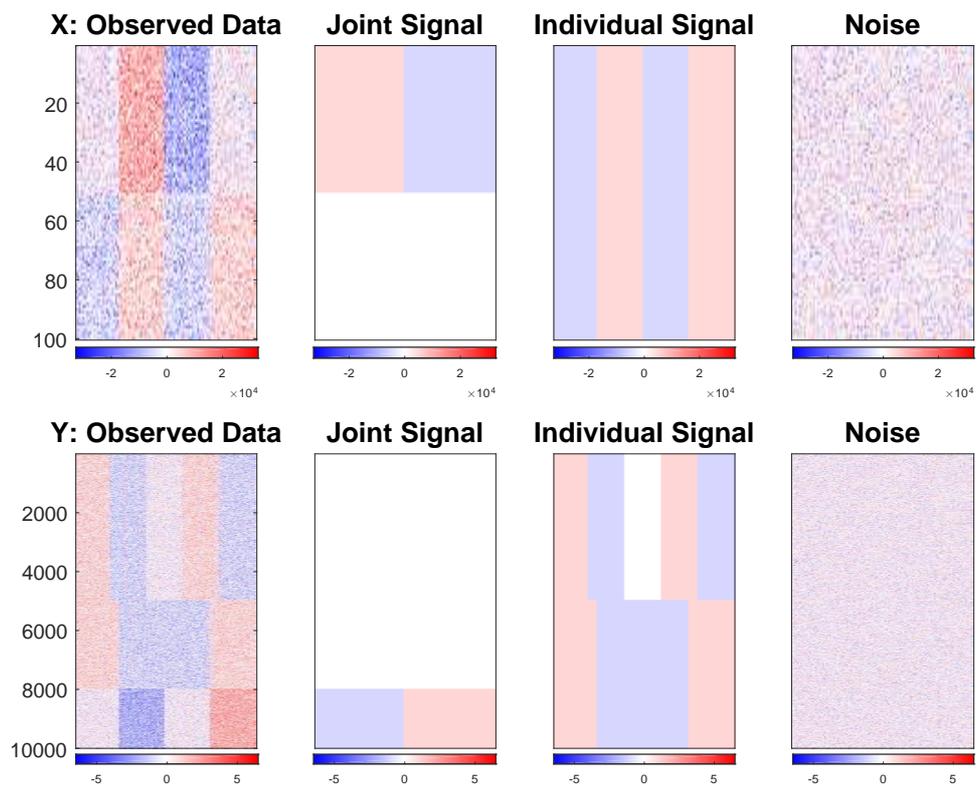}
		\caption[Data blocks $X$ and $Y$ in the toy example]{Data blocks $X$ (top) and $Y$ (bottom) in the toy example. The left panels present the observed data matrices which are a sum of the signal and noise matrices depicted in the remaining panels. Scale is indicated by color bars at the bottom of each sub-plot. These structures are challenging to capture using conventional methods due to very different orders of magnitude and numbers of features.}
		\label{fig:jive:toyrawdata}
	\end{center}
\end{figure}

Simply concatenating $X$ and $Y$ on columns and performing a singular value decomposition on this concatenated matrix completely fails to give a meaningful joint analysis. PLS and CCA might be used to address the magnitude difference in this example and capture the signal components. However, they target common relationships between two data matrices and therefore are unable to simultaneously extract and distinguish the two types of variation. Moreover, because of its sensitivity to the strength of the signal PLS misclassifies correlated individual components as joint components.  The original JIVE of \citet{lock2013joint} also fails on this toy example. Details on all of these can be found in \ref{s:toy_details}. 

In this toy example, the selection of the initial low-rank parameters $r_X = 2$ and $r_Y = 3$ is unambiguous. The left panel of Figure~\ref{fig:jive:toyjiveoutput} shows this AJIVE-approximation well captures the signal variations within both $X$ and $Y$. What's more, our method correctly distinguishes the types of variation showing its robustness against heterogeneity across data blocks and correlation between individual data blocks. The approximations of both joint and individual signal are depicted in the remaining panels. A careful study of the impact of initial rank misspecification on the AJIVE results for this toy example is in Sections~\ref{c:jive-s:method-subsec:step1} and \ref{c:jive-s:method-subsec:step2}.

\begin{figure}[htp!]
	\centering
	\includegraphics[scale = 0.75]{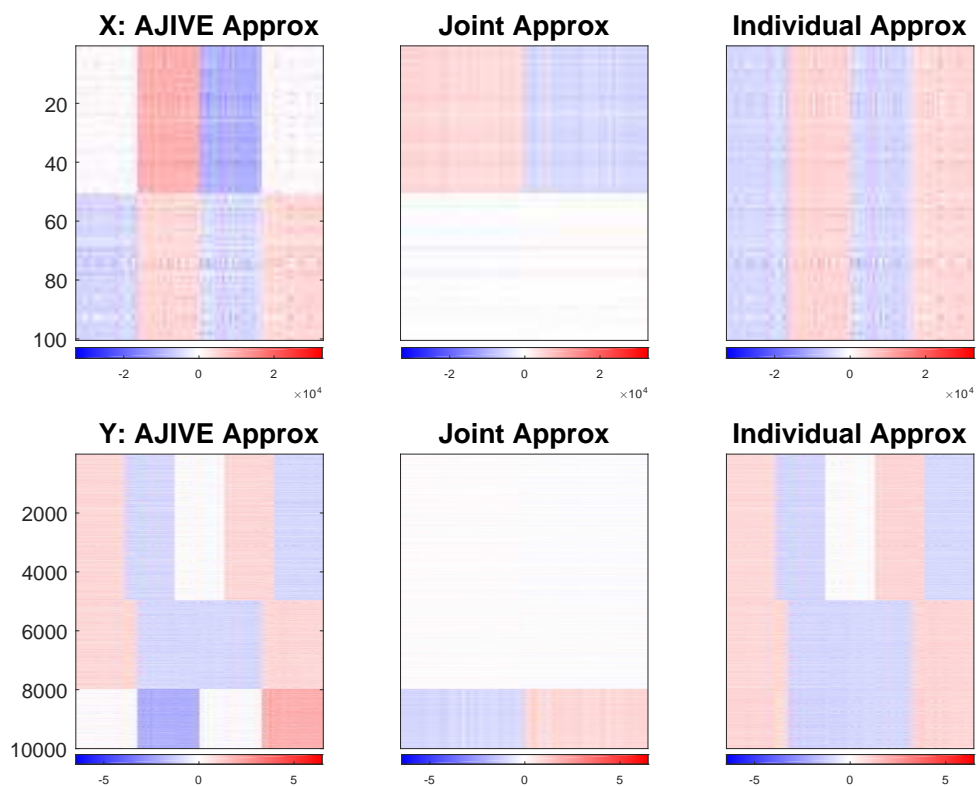}
	\caption[AJIVE approximation of the toy data]{ AJIVE approximation of the data blocks $X$ and $Y$ in the toy example are shown in the first column, with the joint and individual signal matrices depicted in the remaining columns. Both quite diverse types of variations are well captured for each data block by AJIVE, in contrast to other usual methods as seen in \ref{s:toy_details}.}
	\label{fig:jive:toyjiveoutput}
\end{figure}

The rest of this paper is organized as follows. Section~\ref{c:jive-s:method} describes the population model and mathematical details of the estimation approach. 
Results of application to a TCGA breast cancer data set and a mortality data set are presented in Section~\ref{c:jive-s:datanalaysis}. 
Relationships between the proposed AJIVE and other methods from an optimization point of view are discussed in Section~\ref{s:optimization_analysis}. The AJIVE Matlab software, the related Matlab scripts and associated datasets, which can be used to reproduce all the results in this paper, are available at the GitHub repository \url{https://github.com/MeileiJiang/AJIVE_Project}.

\section{Proposed method}
\label{c:jive-s:method}

In this section the details of the new proposed AJIVE are presented. The population model is proposed in Sections~\ref{c:jive-s:method-subsec:model}. The estimation approaches are given in Section \ref{c:jive-s:method-subsec:step1}, \ref{c:jive-s:method-subsec:step2}, and \ref{c:jive-s:method-subsec:step3}. 

\subsection{Population model}
\label{c:jive-s:method-subsec:model}

Matrices $X_1, \ldots,X_K$ each of size $d_k \times n$ are a set of data blocks for study, e.g., the colored blocks on the left of Figure~\ref{fig:flow}. The columns are regarded as data objects, one vector of measurements for each experimental subject, while rows are considered as features. All $X_k$s therefore have the same number of columns and perhaps a different number of rows.

Each $X_k$ is modeled as low-rank true underlying signals $A_k$ perturbed by additive noise matrices $E_k$. Each low-rank signal $A_k$ is the sum of two matrices containing joint and individual variation, denoted as $J_k$ and $I_k$ respectively for each block, viz.
\begin{equation*}
\begin{bmatrix}
X_1 \\
\vdots \\
X_K \\
\end{bmatrix} = \begin{bmatrix}
A_1 \\
\vdots \\
A_K \\
\end{bmatrix}  + \begin{bmatrix}
E_1 \\
\vdots \\
E_K \\
\end{bmatrix}  = \begin{bmatrix}
J_1 \\
\vdots \\
J_K \\
\end{bmatrix}  + \begin{bmatrix}
I_1 \\
\vdots \\
I_K \\ 
\end{bmatrix}  + \begin{bmatrix}
E_1 \\
\vdots \\
E_K \\
\end{bmatrix}.
\end{equation*}
Our approach focuses on the vectors in the row space of our matrices. In this context these vectors are often called \emph{score vectors} and the row space of the matrix is often called \emph{score subspace} ($\subset \mathbb{R}^n$). Therefore, the row spaces of the matrices capturing joint variation, i.e., joint matrices, are defined as sharing a common score subspace denoted as $\row(J)$, viz.
\[
\row(J_1) = \cdots =  \row(J_K) = \row(J).
\]
The individual matrices are individual in the sense that they are orthogonal to the joint space, i.e., $\row(I_{k}) \perp \row(J)$ for all $k \in\{ 1, \ldots, K\}$, and the intersection of their score subspaces is the zero vector space, i.e.,
\[
\bigcap\limits_{k=1}^{K} \row(I_k) = \{\vec{0}\}.
\]
This means that there is no non-trivial common row pattern in every individual score subspaces across blocks. 

To ensure an identifiable variation decomposition we assume $\row(J)\subset \row(A_k)$, which also implies $\row(I_k)\subset \row(A_k)$, for all $k\in \{1,\ldots, K\}$. Note that orthogonality between individual matrices $\{I_1, \ldots, I_K\}$ is \textit{not} assumed as it is not required for the model to be uniquely determined. 

Under these assumptions, the model is identifiable in the following sense.
\begin{lemma}
	\label{lemma-exist}
	Given a set $\{A_1, \ldots, A_K\}$ of matrices, there are unique sets $\{J_1, \ldots, J_K\}$ and $\{I_1, \ldots, I_K\}$ of matrices so that
	\begin{itemize}
		\item [(i)] $A_k$ = $J_k + I_k$, for all $k \in \{ 1, \ldots, K\}$;
		\item [(ii)] $\row(J_k) = \row(J) \subset \row(A_k)$, for all $k \in \{ 1, \ldots, K\}$;
		\item [(iii)] $\row(J) \perp \row(I_k)$, for all $k \in \{ 1, \ldots, K\}$;
		\item [(iv)] $\bigcap\limits_{k=1}^{K} \row(I_k) = \{\vec{0}\}.$
	\end{itemize}
\end{lemma}
The proof is provided in \ref{s:proofs}. Lemma~\ref{lemma-exist} is very similar to Theorem 1.1 in~\citet{lock2013joint}. The main difference is that the rank conditions are replaced by conditions on row spaces. In our view, this provides a clearer mathematical framework and more precise understanding of the different types of variation. 

The additive noise matrices, $E_1, \ldots, E_K$, are assumed to follow an isotropic error model where the energy of projection is invariant to direction in both row and column spaces. Important examples include the multivariate standard normal distribution and the matrix multivariate Student $t$ distribution~\citep{kotz2004multivariate}. All singular values of each noise matrix are assumed to be smaller than the smallest singular values of each signal to give identifiability. This assumption on the noise distribution here is weaker than the classical iid Gaussian random matrix, and only comes into play when determining the number of joint components. 

The estimation algorithm, which segments the data into joint and individual components in the presence of noise, has three main steps, as follows.

\begin{description}
\item
\textbf{Step 1: Signal Space Initial Extraction.} Low-rank approximation of each data block, as shown on the left in Figure~\ref{fig:flow}. A novel approach together with careful assessment of accuracy using matrix perturbation theory from linear algebra~\citep{stewart1990matrix}, is provided in Sections~\ref{c:jive-s:method-subsec:step1} and \ref{c:jive-s:method-subsec:step2}. 

\item
\textbf{Step 2: Score Space Segmentation.} Initial determination of joint and individual components, as shown in the center of Figure~\ref{fig:flow}. Our approach to this is based on an extension of Principal Angle Analysis, and an inferential based graphical diagnostic tool. The two block case is discussed in Section~\ref{c:jive-s:method-subsec:step2-2blocks}, with the multi-block case appearing in Section~\ref{c:jive-s:method-subsec:step2-multiblocks}.

\item
\textbf{Step 3: Final Decomposition and Outputs.} Check segmented components still meet initial thresholds in Step 1, and reproject for appropriate outputs, as shown in the right of Figure~\ref{fig:flow}. Details of this are in Section~\ref{c:jive-s:method-subsec:step3}.
\end{description}

\subsection{Step 1: Signal space initial extraction}
\label{c:jive-s:method-subsec:step1}

Even though the signal components $A_1, \ldots, A_K$ are low-rank, the data matrices $X_1, \ldots, X_K$ are usually of full rank due to the presence of noise. SVD works as a signal extraction device in this step, keeping components with singular values greater than selected thresholds individually for each data block, as discussed in Section~\ref{c:jive-s:method-subsec:step1-threshold}. The accuracy of this SVD approximation will be carefully estimated in Section~\ref{c:jive-s:method-subsec:step1-accuracy}, and will play an essential role in segmenting the joint space in Step 2.

\subsubsection{Initial low-rank approximation}
\label{c:jive-s:method-subsec:step1-threshold}

Each signal block $A_k$ is estimated using SVD of $X_k$. Given a threshold $t_k$, the estimator $\tilde{A}_k$ (represented in Figure~\ref{fig:flow} as the boxes with dashed colored boundaries on the left) is defined by setting all singular values below $t_k$ to 0. The resulting rank $\tilde{r}_k$ of $\tilde{A}_k$ is an initial estimator of the signal rank $r_k$. The reduced-rank decompositions of the $\tilde{A}_k$s are
\begin{equation}
\label{equ:step1}
\tilde{A}_k = \tilde{U}_k\tilde{\Sigma}_k\tilde{V}_k^\top,
\end{equation}
where $\tilde{U}_{k}$ contains the left singular vectors that correspond to the largest $\tilde{r}_k$ singular values respectively for each data block. The initial estimate of the signal score space, denoted as $\row(\tilde{A}_{k})$, is spanned by the right singular vectors in $\tilde{V}_{k}$ (shown as gray boxes with colored boundaries on the left of Figure~\ref{fig:flow}).

When selecting these thresholds, one needs to be aware of a bias/variance like trade-off. Setting the threshold too high will provide an accurate estimation of the parts of the joint space that are included in the low-rank approximation. The downside is that significant portions of the joint signal might be thresholded out. This could be viewed as a low-variance high-bias situation. If the threshold is set low, then it is likely that the joint signal is included in all of the blocks. However, the precision of the segmentation in the next step can deteriorate to the point that individual components, or even worse, noise components, can be selected in the joint space. This can be viewed as the low-bias high-variance situation.

Most off-the-shelf automatic procedures for low-rank matrix approximation have as their stated goal signal reconstruction and prediction, which based on our experience tends toward thresholds that are too small, i.e., input ranks that are too large. This is sensible as adding a little bit more noise usually helps prediction but it has bad effects on signal segmentation. We therefore recommend taking a multi-scale perspective and trying several threshold choices, e.g., by considering several relatively big jumps in a scree plot. A useful inferential graphical device to assist with this choice is developed in Section~\ref{c:jive-s:method-subsec:step2}. 

Figure~\ref{fig:toy1singularvalue} shows the scree plots of each data block for the toy example in Section~\ref{c:jive-s:intro}. The left scree plot for $X$ clearly indicates a selection of rank $\tilde{r}_1 = 2$ and the right scree plot for $Y$ points to rank $\tilde{r}_2 = 3$; in both cases those components stand out while the rest of the singular values decay slowly showing no clear jump.

\begin{figure}[t!]
		\centering
		\includegraphics[scale=0.3]{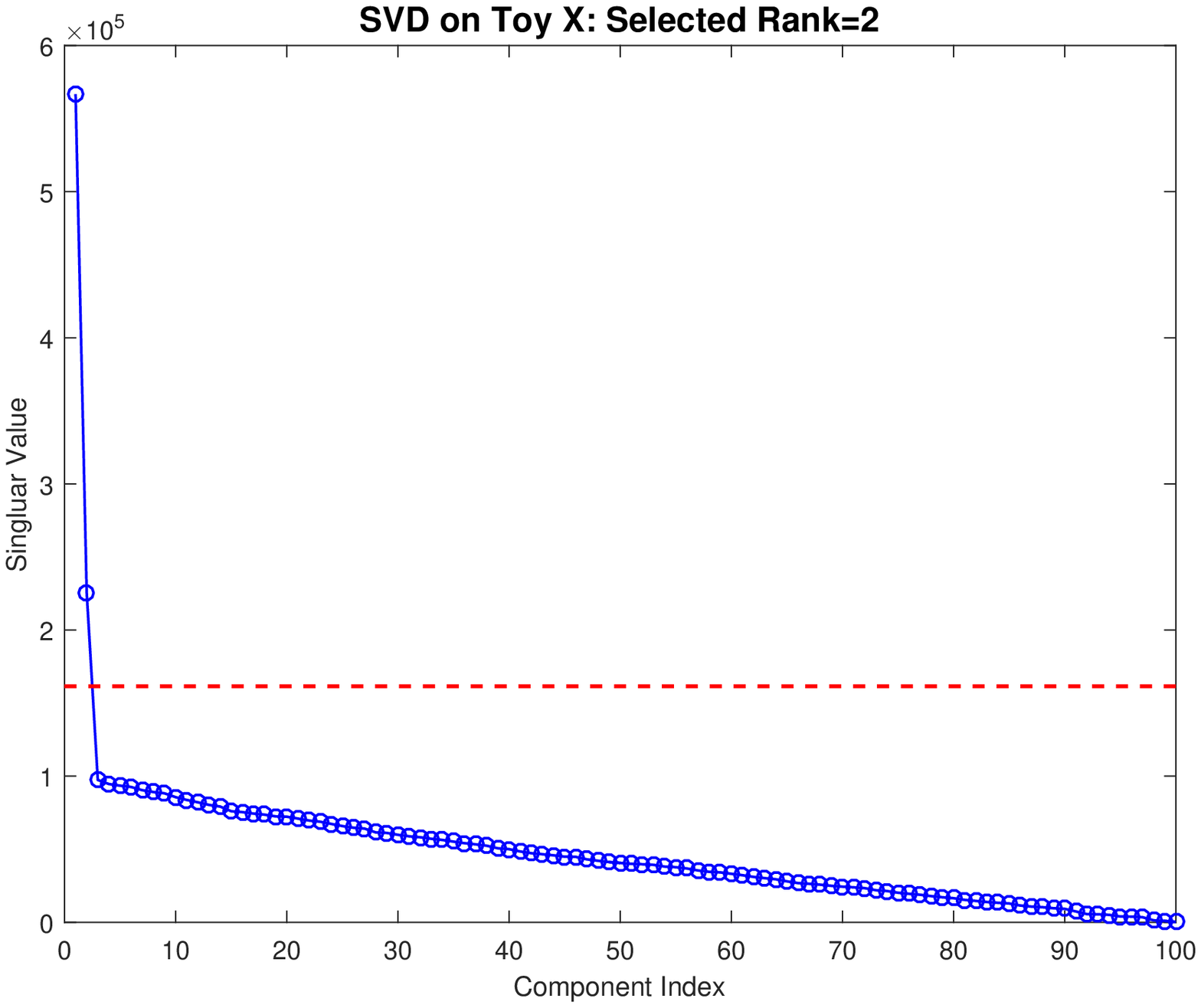}
		\vspace{4ex}
		\centering
		\includegraphics[scale=0.3]{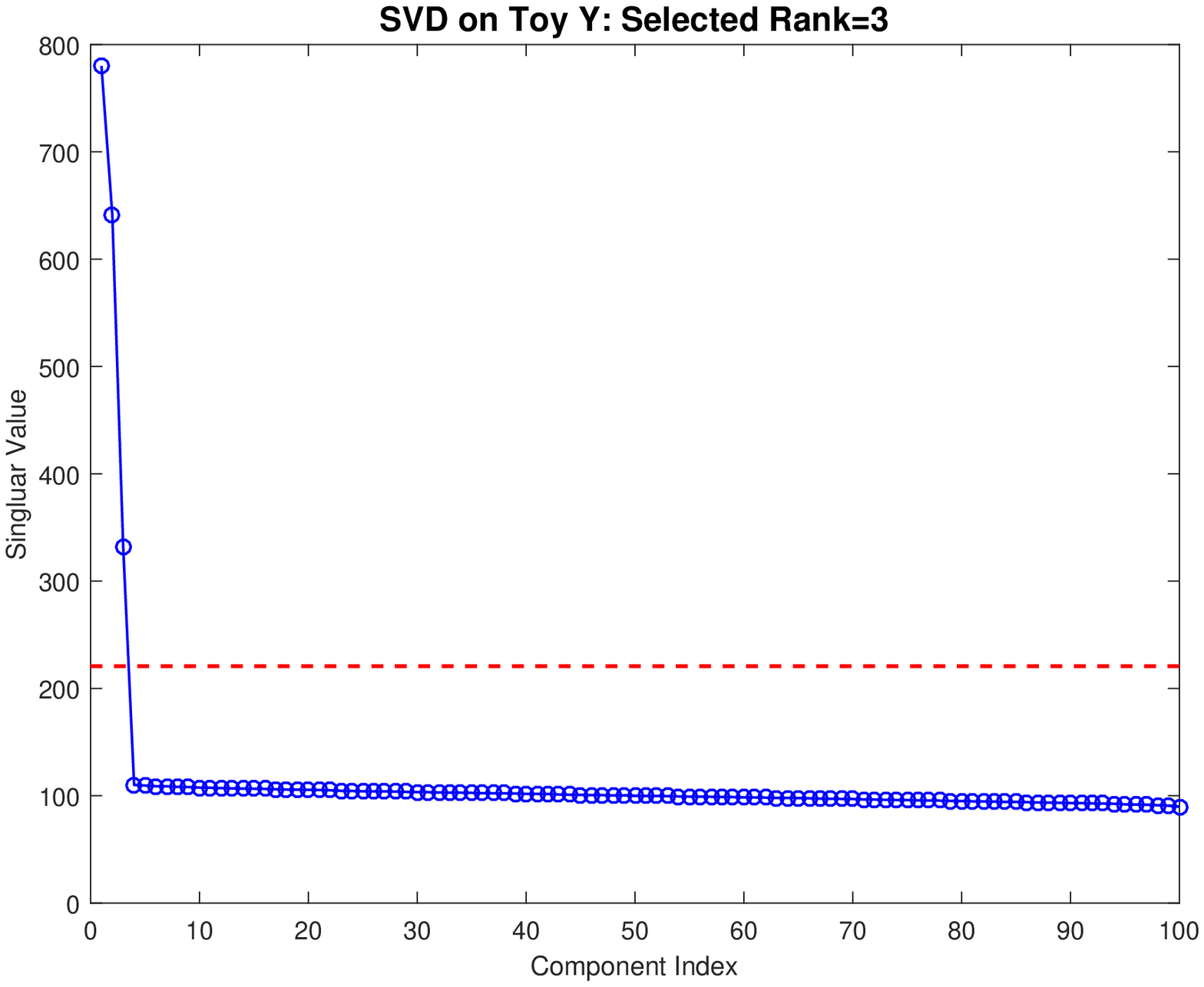}
		\vspace{4ex}
	\caption[Scree plots for the toy data]{Scree plots for the toy data sets $X$ (left) and $Y$ (right). Both plots display the singular values associated with a component in descending order versus the index of the component. The components with singular values above the dashed red threshold line are regarded as the initial signal components in the first step of AJIVE.}
	\label{fig:toy1singularvalue}
\end{figure}

\subsubsection{Approximation accuracy estimation}
\label{c:jive-s:method-subsec:step1-accuracy}

A major challenge is segmentation of the joint and individual variation in the presence of noise which individually perturbs each signal. A first step towards addressing this is a careful study of how well $A_k$ is approximated by $\tilde{A}_k$ using the \emph{Generalized $\sin \theta$ Theorem}~\citep{wedin1972perturbation}.

\paragraph{Pseudometric between subspaces}

To apply the Generalized $\sin\theta$ Theorem, we use the following pseudometric as a notion of distance between theoretical and perturbed subspaces. Recall that $\row(A_k)$, $\row(\tilde{A}_k)$ are respectively the $r_k$- and $ \tilde{r}_k$-dimensional score subspaces of $\mathbb{R}^{n}$ for the matrix $A_k$ and its approximation $\tilde{A}_k$. The corresponding projection matrices are $P_{A_k}$ and $P_{\tilde{A}_k}$, respectively. A pseudometric between the two subspaces can be defined as the difference of the projection matrices under the operator $L^2$ norm, i.e., $\rho \{ \row(A_k), \row(\tilde{A}_k)\} = \|P_{A_k} - P_{\tilde{A}_k}\|$~\citep{stewart1990matrix}. When $r_k = \tilde{r}_k$, this pseudometric is also a distance between the two subspaces.

An insightful understanding of this pseudometric $\rho \{ \row(A_k), \row(\tilde{A}_k)\}$ comes from a principal angle analysis~\citep{jordan1875essai, hotelling1936relations} of the subspaces $\row(A_k)$ and  $\row(\tilde{A}_k)$. Denote the principal angles between $\row(A_k)$ and $\row(\tilde{A}_k)$ as
\begin{equation}
\label{equ:perturb_angle}
	\Theta\{ \row(A_k), \row(\tilde{A}_k)\} = \{\theta_{k,1}, \ldots, \theta_{k, r_k \wedge \tilde{r}_k}\}
\end{equation} 
with $\theta_{k,1} \leq \cdots \leq \theta_{k,r_k \wedge \tilde{r}_k}$. The pseudometric $\rho \{ \row(A_k), \row(\tilde{A}_k)\} $ is equal to the sine of the maximal principal angle, i.e., $\sin\theta_{k,r_k \wedge \tilde{r}_k}$. Thus the largest principal angle between two subspaces measures their closeness, i.e., distance. 

The pseudometric $\rho\{\row(A_k), \row(\tilde{A}_k)\}$ can be also written as 
$$
\rho \{ \row(A_k), \row(\tilde{A}_k)\} = \|(I-P_{A_k})P_{\tilde{A}_k}\| = \|(I-P_{\tilde{A}_k})P_{A_k}\|,
$$
which gives another useful understanding of this definition. It measures the relative deviation of the signal variation from the theoretical subspace. Accordingly, the similarity/closeness between the subspaces and its perturbation can be written as $\|P_{A_k}P_{\tilde{A}_k}\|$ and is equal to the cosine of the maximal principal angle defined above, i.e., $\cos\theta_{k,r_k \wedge \tilde{r}_k}$. Hence, $\sin^2\theta_{k,r_k \wedge \tilde{r}_k}$ indicates the proportion of signal deviation and $\cos^2\theta_{k,r_k \wedge \tilde{r}_k}$ tells the proportion of remaining signal in the theoretical subspace.

\paragraph{Wedin bound}

For a signal matrix $A_k$ and its perturbation $X_k = A_k + E_k$, the generalized $\sin \theta$ theorem provides a bound for the distance between the rank $\tilde{r}_k$ $(\leq r_k)$ singular subspaces of $A_k$ and $X_k$. This bound quantifies how the theoretical singular subspaces are affected by noise.

\begin{theorem}[Wedin, 1972]
	\label{thm-wedin}
	Let $A_k$ be a signal matrix with rank $r_k$. Letting $A_{k, 1} = U_{k, 1} \Sigma_{k, 1} V_{k, 1}^{\top} $ denote the rank $\tilde{r}_k$ SVD of $A_k$, where $\tilde{r}_k \leq r_k$, write $A_k = A_{k, 1} + A_{k, 0}$. For the perturbation $X_k = A_k + E_k$, a corresponding decomposition can be made as $X_k = \tilde{A}_{k, 1} + \tilde{E}_k$, where $\tilde{A}_{k, 1} = \tilde{U}_{k, 1} \tilde{\Sigma}_{k, 1} \tilde{V}_{k, 1}^{\top}$ is the rank $\tilde{r}_k$ SVD of $X_k$. Assume that there exists an $\alpha \geq 0$ and a $\delta > 0$ such that for $\sigma_{\min}(\tilde{A}_{k, 1})$ and $\sigma_{\max}(A_{k, 0})$ denoting appropriate minimum and maximum singular values
	$$ \sigma_{\min}(\tilde{A}_{k, 1}) \geq \alpha + \delta \quad \text{and} \quad \sigma_{\max}(A_{k, 0}) \leq \alpha.$$
	Then the distance between the row spaces of $\tilde{A}_{k, 1}$ and $A_{k, 1}$ is bounded by
	$$ \rho \{ \row(\tilde{A}_{k, 1}), \row(A_{k, 1})\} \leq \frac{\max{ (\|E_k \tilde{V}_{k, 1}\|, \|E_k^{\top} \tilde{U}_{k, 1}\|} )}{\delta} \wedge 1.$$ 
\end{theorem}

In practice we do not observe $A_{k, 0}$ thus $\delta$ cannot be estimated in general. A special case of interest for AJIVE is $\tilde{r}_k = r_k$, in which case $A_{k, 0} = 0, A_k = A_{k, 1}$. The following is an adaptation of the generalized $\sin \theta$ theorem to this case.

\begin{corollary}[Bound for correctly specified rank]
	\label{thm-wedin-full}
	For each $k \in \{ 1, \ldots , K\}$, the signal matrix $A_k$ is perturbed by additive noise $E_k$. Let $\theta_{k, \tilde{r}_k}$ be the largest principal angle for the subspace of signal $A_k$ and its approximation $\tilde{A}_k$, where $\tilde{r}_k = r_k$. Denote the SVD of $\tilde{A}_k$ as $\tilde{U}_k\tilde{\Sigma}_k\tilde{V}_k^\top$. The distance between the subspaces of $A_k$ and $\tilde{A}_k$, $\rho\{ \row(A_k), \row(\tilde{A}_k)\}$, i.e., sine of $\theta_{k,\tilde{r}_k}$, is bounded above by
	\begin{equation}
	\label{equ:wedin-correct}
	\rho \{ \row(A_k), \row(\tilde{A}_k)\} = \sin\theta_{k,\tilde{r}_k} \leq \frac{\textnormal{max}(\|E_k\tilde{V}_k\|, \|E_k^\top\tilde{U}_k\|)}{\sigma_{\textnormal{min}}(\tilde{A}_k)} \wedge 1.
	\end{equation}
\end{corollary}	
	
In this case the bound is driven by the maximal value of noise energy in the column and row spaces and by the estimated smallest signal singular value. This is consistent with the intuition that a deviation distance, i.e., a largest principal angle, is small when the signal is strong and perturbations are weak. 

In general, it can be very challenging to correctly estimate the true rank of $A_k$. If the true rank $r_k$ is not correctly specified, then different applications of the Wedin bound are useful. In particular, when $A_{k, 0}$ is not 0, i.e., $\tilde{r}_k < r_k$, insights come from replacing $E_k$ by $E_k + A_{k, 0}$ in the Wedin bound.

\begin{corollary}[Bound for under-specified rank]
	\label{thm-wedin-under}
	For each $k \in \{ 1, \ldots , K\}$, the signal matrix $A_k$ with rank $r_k$ is perturbed by additive noise $E_k$. Let $\tilde{A}_k = \tilde{U}_k\tilde{\Sigma}_k\tilde{V}_k^\top$ be the rank $\tilde{r}_k$ SVD approximation of $A_k$ from the perturbed matrix, where $\tilde{r}_k < r_k$.
	Denote $A_k =  A_{k, 1} + A_{k, 0}$, where $A_{k, 1}$ is the rank $\tilde{r}_k$ SVD of $A$. Then the distance between $\row(A_{k, 1})$ and $\row(\tilde{A}_k)$ is bounded above by
	$$
	\rho \{ \row(A_{k, 1}), \row(\tilde{A}_k)\} \leq \frac{\max{ (\|(E_k + A_{k, 0}) \tilde{V}_k\|, \|(E_k + A_{k, 0})^{\top} \tilde{U}_k\| )}}{\sigma_{\min}(\tilde{A}_k)} \wedge 1.
	$$
\end{corollary}

For the other type of initial rank misspecification, $\tilde{r}_k > r_k$, we augment $A_k$ with appropriate noise components to be able to use the Wedin bound.

\begin{corollary}[Bound for over-specified rank]
	\label{thm-wedin-over}
	For each $k \in \{ 1, \ldots , K\}$, the signal matrix $A_k = U_k \Sigma_k V_k^\top$ with rank $r_k$ is perturbed by additive noise $E_k$. Let $\tilde{A}_k = \tilde{U}_k\tilde{\Sigma}_k\tilde{V}_k^\top$ be the rank $\tilde{r}_k$ SVD of $X_k$, where $\tilde{r}_k > r_k$.
    Let $E_0$ be the rank $\tilde{r}_k - r_k$ SVD of $(I - U_k U_k^\top) E_k (I - V_k V_k^\top)$. Then the pseudometric between $\row(A_k)$ and $\row(\tilde{A}_k)$ is bounded above by
	$$
	\rho \{ \row(A_k), \row(\tilde{A}_k)\} \leq \frac{\max{(\|(E_k - E_0) \tilde{V}_k\|, \|(E_k - E_0)^{\top} \tilde{U}_k\|)}}{\sigma_{\min}(\tilde{A}_k)} \wedge 1.
	$$
\end{corollary}

The bounds in Corollaries~\ref{thm-wedin-full}--\ref{thm-wedin-over} provide many useful insights. However, these bounds still cannot be used directly since we do not observe the error matrices $E_1, \ldots, E_K$. A re-sampling based estimator of the Wedin bounds is provided in the next paragraph. As seen in Figure~\ref{fig:wedin_bound_toydata}, this estimator appropriately adapts to each of the above three cases. Moreover, Figure~\ref{fig:wedin_bound_toydata} also indicate that the Wedin bound for over-specified rank is usually very conservative. 

\begin{figure}
	\vspace{0.1in}
	\begin{minipage}[b]{\linewidth}
		\centering
		\includegraphics[scale= 0.45]{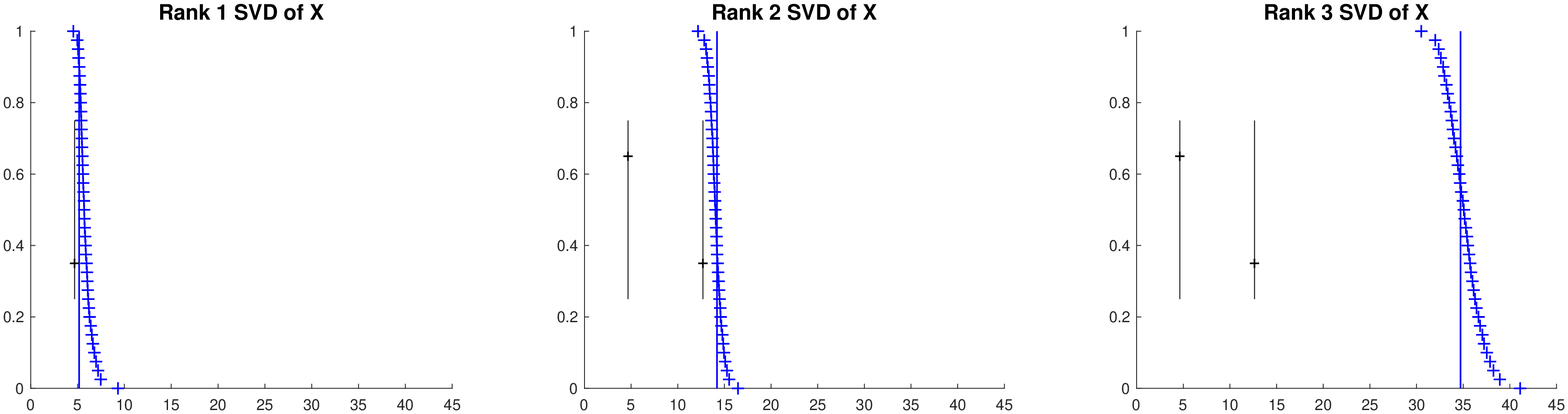}
		\vspace{4ex}
	\end{minipage}
	\begin{minipage}[b]{\linewidth}
		\centering
		\includegraphics[scale= 0.45]{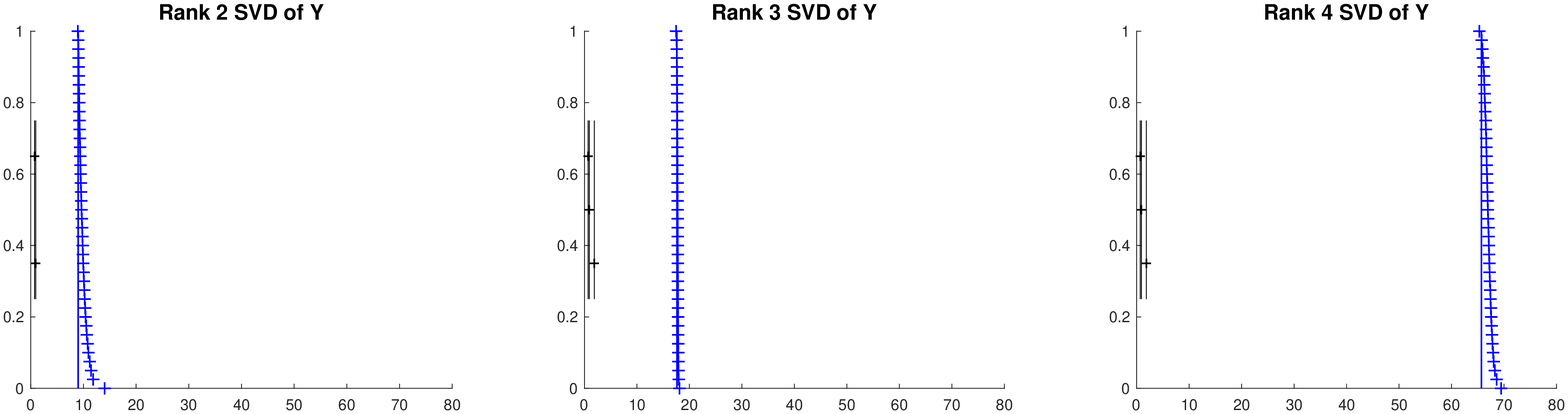}
		\vspace{4ex}
	\end{minipage}
	\caption{Principal angle plots between each singular subspace of the signal matrix $A_{k,1}$ and its estimator $\tilde{A}_k$ for the toy dataset. Graphics for $X$ are on the upper row, with $Y$ on the lower row. The left, middle and right columns are the under-specified, correctly specified and over-specified signal matrix rank cases respectively. Each $x$-axis represents the angle. The $y$-axis shows the values of the survival function of the resampled distribution, which are shown as blue plus signs in the figure. The vertical blue solid line is the theoretical Wedin bound, showing this bound is well estimated. The vertical black solid line segments represent the principal angles $\theta_{k,1}, \ldots, \theta_{k, r_k \wedge \tilde{r}_k}$ between $\row(A_{k, 1})$ and $\row(\tilde{A}_k)$. The distance between the black and blue lines reveals when the Wedin bound is tight.}
	\label{fig:wedin_bound_toydata}
\end{figure}

\paragraph{Estimation and evaluation of the Wedin bound}

As mentioned above, the perturbation bounds of each $\theta_{k, r_k \wedge \tilde{r}_k}$ require the estimation of terms $\|E_k\tilde{V}_k\|$, $\|E_k^\top\tilde{U}_k\|$ for $k \in \{ 1, 2\}$. These terms are measurements of energies of the noise matrices projected onto the signal column and row spaces. Since an isotropic error model is assumed, the \emph{distributions} of energy of the noise matrices in arbitrary fixed directions are equal. Thus, if we sample random subspaces of dimension $\tilde{r}_k$, that are orthogonal to the estimated signal $\tilde{A}_k$, and use the observed residual $\tilde{E}_k = X_k - \tilde{A}_k$, this should provide a good estimator of the distribution of the unobserved terms $\|E_k\tilde{V}_k\|$, $\|E_k^\top\tilde{U}_k\|$.

In particular, consider the estimation of the term $\|E_k\tilde{V}_k\|$. We draw a random subspace of dimension $\tilde{r}_k$ that is orthogonal to $\tilde{V}_k$, denoted as $V_k^{\star}$. The observed data block $X_k$ is projected onto the subspace spanned by $V_k^{\star}$, written as $X_k V_k^{\star}$. The distribution (with respect to the $V_k^{\star}$ variation) of the operator $L^2$ norm $\|X_k V_k^{\star}\| = \|\tilde{E}_k V_k^{\star}\|$ approximates the distribution of the unknown $\|E_k\tilde{V}_k\|$ because both measure noise energy in essentially random directions. Similarly the estimation of $\|E_k^\top \tilde{U}_k\|$ is approximated by $\|X_k^\top U_k^{\star}\|$, where $U_k^{\star}$ is a random $\tilde{r}_k$-dimensional subspace orthogonal to $\tilde{U}_k$. These distributions are used to estimate the Wedin bound by generating 1000 replications of $\|X_k V_k^{\star}\|$ and $\|X_k^\top U_k^{\star}\|$, and plugging these into \eqref{equ:wedin-correct}. The quantiles of the resulting distributions are used as prediction intervals for the unknown theoretical Wedin bound. Note this random subspace sampling scheme provides a distribution with smaller variance than simply sampling from the remaining singular values of $X_k$, i.e., using 1000 subspaces each generated by a random sample of $\tilde{r}_k$ remaining singular vectors.

There are two criteria for evaluating the effectiveness of the estimator. First is how well the resampled distributions approximate the underlying theoretical Wedin bounds. This is addressed in Figure~\ref{fig:wedin_bound_toydata}, which is based on the toy example in Section~\ref{c:jive-s:intro-subsec:toy}. For each of the matrices $X$ and $Y$ (top and bottom rows), the under, correctly, and over specified signal rank cases (Corollaries \ref{thm-wedin-under}, \ref{thm-wedin-full} and \ref{thm-wedin-over}, respectively) are carefully investigated. In each case the theoretical Wedin bound (calculated using the true underlying quantities, that are only known in a simulation study) are shown as vertical blue lines. Our resampling approach provides an estimated distribution, the survival function of which is shown using blue plus signs. This indicates remarkably effective estimation of the Wedin bound in all cases.

The second more important criterion is how well the prediction interval covers the actual principal angles between $\row(A_k)$ and $\row(\tilde{A}_k)$. These angles are shown as vertical black line segments in Figure~\ref{fig:wedin_bound_toydata}. For the square matrix $X$, in the under and correctly specified case (top, left, and center), the Wedin bound seems relatively tight. In all other cases, the Wedin bound is conservative.

Figure~\ref{fig:wedin_bound_toydata} shows one realization of the noise in the toy example. A corresponding simulation study is summarized in Table~\ref{t:smallsimulation}. For this we generated 10,000 independent copies of the data sets $X$ ($100\times 100$, true signal rank $r_1 = 2$) and $Y$ ($\mbox{10,000} \times 100$, true signal rank $r_2 = 3$). Then for several low-rank approximations (columns of Table~\ref{t:smallsimulation}) we calculated the estimate of the angle between the true signal and the low-rank approximation. Table~\ref{t:smallsimulation} reports the percentage of the times the corresponding quantile of the resampled estimate is bigger than the true angle for the matrix $X$. When the rank is correctly specified, i.e., $\tilde{r}_1 = r_1 = 2$, we see that the performance for the square matrix $X$ is satisfactory as the empirical percentages are close to the nominal values. When the rank is misspecified, the empirical upper bound is conservative. Corresponding empirical percentages for the high dimension, low sample size data set $Y$ are all 100\%, and thus are not shown. This is caused by the fact that Wedin bound can be very conservative if the matrix is far from square. As seen in Figure~\ref{fig:diagnostic_angle_plot} this can cause identification of spurious joint components. This motivates our development of a diagnostic plot in Section~\ref{c:jive-s:method-subsec:step2}. Recent works of \citet{cai2016rate} and \citet{o2013random} may provide potential approaches for improvement of the Wedin bound.
\begin{table}[tb]
	\caption{Coverages of the prediction intervals of the true angle between the signal $\row(A_{k, 1})$ and its estimator $\row(\tilde{A}_k)$ for the matrix $X$ in the toy example. Rows are nominal levels. Columns are ranks of approximation (where 2 is the correct rank). The simulation based on 10000 realizations of $X$ shows good performance for this square matrix.}\label{t:smallsimulation}
	\begin{center}
		\begin{tabular}{c|ccc}
			&1 & \bf{2} & 3 \\
			\hline
			50\%  & 91.9\% &   \bf{63.6\%} &  100.0\% \\
			90\% &100.0\% &  \bf{89.6\%} & 100.0\%\\
			95\% & 100.0\% &  \bf{93.7\%} & 100.0\%\\
			99\% & 100.0\% &  \bf{98.0\%} & 100.0\%\\
			\hline
		\end{tabular}
	\end{center}
\end{table}

\subsection{Step 2: Score space segmentation}
\label{c:jive-s:method-subsec:step2}

\subsubsection{Two-block case}
\label{c:jive-s:method-subsec:step2-2blocks}

For a clear introduction to the basic idea of score space segmentation into joint and individual components, the two-block special case ($K=2$) is first studied. The goal is to use the low-rank approximations $\tilde{A}_k$ from Eq.~(\ref{equ:step1}) to obtain estimates of the common joint and individual score subspaces. Due to the presence of noise, the components of $\row(\tilde{A}_{1})$ and $\row(\tilde{A}_{2})$ corresponding to the underlying joint space, no longer are the same, but should have a relatively small angle. Similarly, the components corresponding to the underlying individual spaces are expected to have a relatively large angle. This motivates the use of principal angle analysis to separate the joint from the individual components. 

\paragraph{Principal angle analysis}

One of the ways of  computing the principal angles between $\row(\tilde{A}_{1})$ and $\row(\tilde{A}_{2})$ is to perform SVD on a concatenation of their right singular vector matrices~\citep{miao1992principal}, i.e.,
\begin{equation}
\label{equ:JointM2}
M \triangleq \begin{bmatrix}
\tilde{V}_{1}^\top \\
\tilde{V}_{2}^\top \\
\end{bmatrix} = U_{M}\Sigma_{M}V_{M}^\top,
\end{equation}
where the singular values, $\sigma_{M,i}$, on the diagonal of $\Sigma_M$, determine the principal angles, $\Phi \{\row(\tilde{A}_1), \row(\tilde{A}_2)\}$ = $\{\phi_1, \ldots, \phi_{\tilde{r}_1 \wedge \tilde{r}_2}\}$, where, for each $i \in \{1, \ldots, \tilde{r}_1 \wedge \tilde{r}_2\}$,
\begin{equation}
\label{equ: principalangle}
\phi_i = \arccos\{(\sigma_{M,i})^2-1\}.
\end{equation}

This SVD decomposition can be understood as a tool that finds pairs of directions in the two subspaces $\row(\tilde{A}_1)$ and $\row(\tilde{A}_2)$ of minimum angle, sorted in increasing order. These angles are shown as vertical black line segments in our main diagnostic graphic introduced in Figure~\ref{fig:diagnostic_angle_plot}.\ The first $\tilde{r}_J$ column vectors in $V_M$ will form the orthonormal basis of the estimated joint space, $\row(J) \subseteq \mathbb{R}^n$. A deeper investigation of the relationship between $V_M$ and the canonical correlation vectors in $U_M$ appears in Section~\ref{s:CCAPPA_optimization}. Next we determine which angles are small enough to be labeled as joint components, i.e., the selection of $\tilde{r}_J$. 

\paragraph{Random direction bound}

In order to investigate which principal angles correspond to random directions, we need to estimate the distribution of principal angles generated by random subspaces. This distribution only depends on the initial input ranks of each data block, $\tilde{r}_k$, and the dimension of the row spaces, $n$. We obtain this distribution by simulation. In particular, $\tilde{V}_1$ and $\tilde{V}_2$ are replaced in \eqref{equ: principalangle} by random subspaces, i.e., each is right multiplied by an independent random orthonormal matrix. The distribution of the smallest principal angle, corresponding to the largest singular value, indicates angles potentially driven by pure noise. We recommend the 5th percentile of the angle distribution as cutoff in practice. Principal angles larger than this are not included in the joint component, which provide 95\% confidence that the selected joint space does not have pure noise components. This cutoff is prominently shown in Figure~\ref{fig:diagnostic_angle_plot} as the vertical dot-dashed red line. The cumulative distribution function of the underlying simulated distribution is shown as red circles.

When the individual spaces are not orthogonal, a sharper threshold based on the Wedin bounds is available.

\paragraph{Threshold based on the Wedin bound}

The following lemma provides a bound on the largest allowable principal angle of the joint part of the initial estimated spaces. 

\begin{lemma}
	\label{lemma-bound}
	Let $\phi$ be the largest principal angle between two subspaces that are each a perturbation of the common row space within $\row(\tilde{A}_1)$ and $\row(\tilde{A}_2)$. That angle is bounded by 
$	\sin\phi \leq \sin(\theta_{1,\tilde{r}_1 \wedge r_1} + \theta_{2,\tilde{r}_2 \wedge r_2} )
$
	in which $\theta_{1,\tilde{r}_1 \wedge r_1}$ and $\theta_{2,\tilde{r}_2 \wedge r_2}$ are the angles given in Eq.~\eqref{equ:perturb_angle}.
\end{lemma}
The proof is provided in \ref{s:proofs}. As with the theoretical Wedin bound, the unknown $\theta_{1,\tilde{r}_1 \wedge r_1}$ and $\theta_{2,\tilde{r}_2 \wedge r_2}$ are replaced by distribution estimators of the Wedin bounds. The survival function of the distribution estimator of this upper bound on $\phi$ is shown in Figure~\ref{fig:diagnostic_angle_plot} using blue plus signs. The vertical dashed blue line is the 95th percentile of this distribution, giving 95\% confidence that angles larger do not correspond to joint components of the lower rank approximations in Step 1. The joint rank $\tilde{r}_J$ is selected to be the number of principal angles, $\phi_i$ in \eqref{equ: principalangle}, that are smaller than both the 5th percentile of the random direction distribution and the 95th percentile of the resampled Wedin bound distribution.

Figure~\ref{fig:diagnostic_angle_plot} illustrates how this diagnostic graphic provides many insights that are useful for initial rank selection. This considers several candidates of initial ranks. Recall for Section~\ref{c:jive-s:intro-subsec:toy}, this toy example has one joint component, one individual $X$ component, and two individual $Y$ components. The row subspaces of their individual components are not orthogonal and the true principal angle (only known in simulation study) is 45\textdegree. Furthermore, PCA of $Y$ reveals that 79.6\% of the joint component appears in the third principal component.
	
The upper left panel of Figure~\ref{fig:diagnostic_angle_plot} shows the under specified rank case of $\tilde{r}_1 = \tilde{r}_2 = 2$. The principal angles (black lines) are larger than the Wedin bound (blue dashed line), so we conclude neither is joint variation. This is sensible since the true joint signal is mostly contained in the 3rd $Y$ component. However, both are smaller than the random direction bound (red dashed line), so we conclude each indicates presence of correlated individual spaces.

The correctly specified rank case of $\tilde{r}_1 = 2, \tilde{r}_2 = 3$ is studied in the upper right panel of Figure~\ref{fig:diagnostic_angle_plot}. Now the smallest angle is smaller than the blue Wedin bound, suggesting a joint component. The second principal angle is about 45\textdegree, which is the angle between the individual spaces. This is above the blue Wedin bound, so it is not joint structure.

The lower left panel considers the over specified initial rank of $\tilde{r}_1 = \tilde{r}_2 = 3$. The over specification results in a loosening of the blue Wedin bound, so that now we can no longer conclude 45\textdegree~is not joint, i.e., $\tilde{r}_J = 2$ cannot be ruled out for this choice of ranks. Note that there is a third principal angle, larger than the red random direction bound, which thus cannot be distinguished from pure noise, which make sense because $A_1$ has only rank $r_1 = 2$.

A case where the Wedin bound is useless is shown in the lower right. Here the initial ranks are $\tilde{r}_1 = 2$ and $\tilde{r}_2 = 4$, which results in the blue Wedin bound being actually larger than the red random direction bound. In such cases, the Wedin bound inference is too conservative to be useful. While not always true, the fact that this can be caused by over specification gives a suggestion that the initial ranks may be too large. Further analysis of this is an interesting open problem.

\begin{figure}[t!]
	\vspace{0.1in}
	\begin{minipage}[b]{0.5\linewidth}
		\centering
		\includegraphics[width=0.8\linewidth]{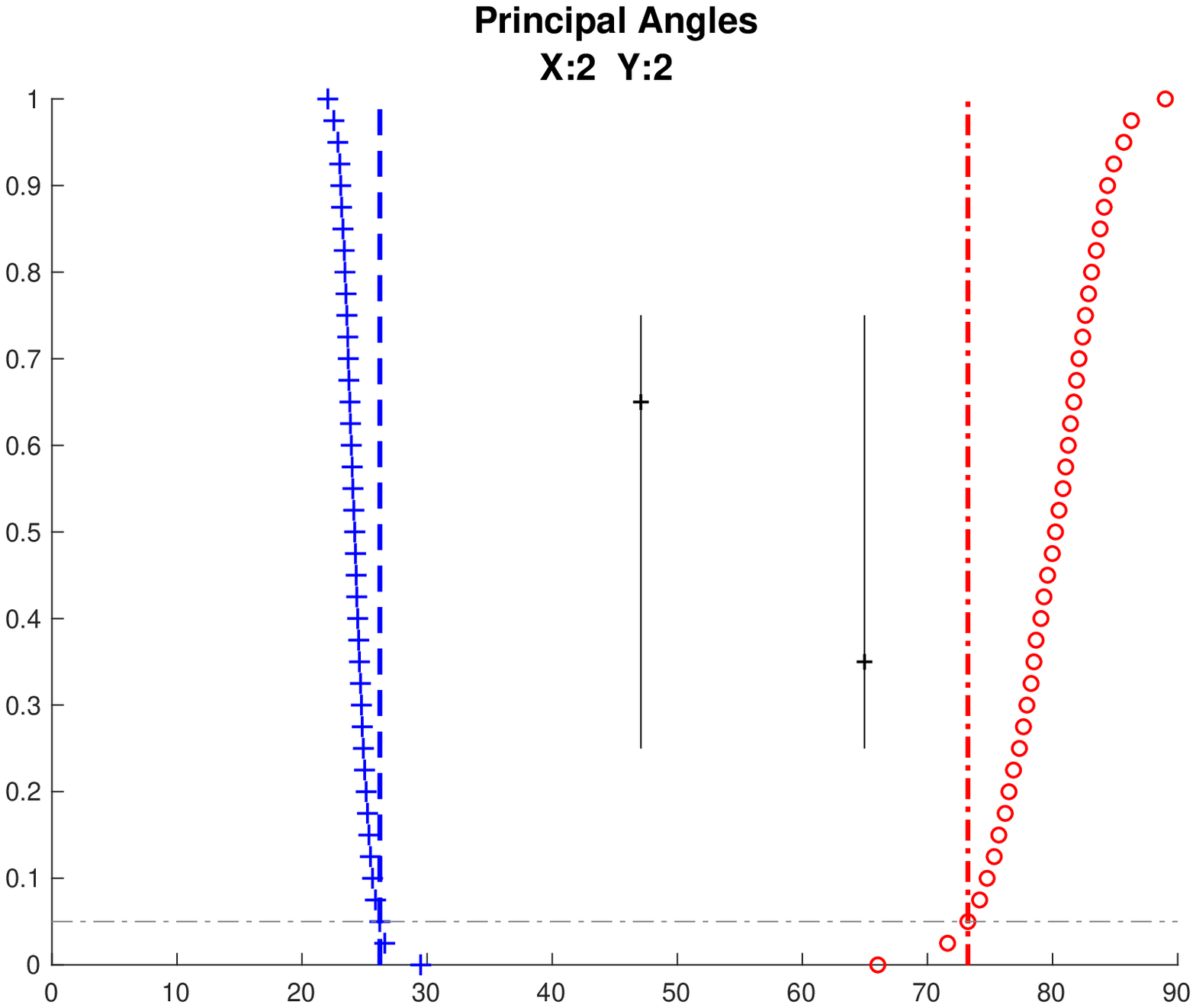}
		\vspace{4ex}
	\end{minipage}
	\begin{minipage}[b]{0.5\linewidth}
		\centering
		\includegraphics[width=0.8\linewidth]{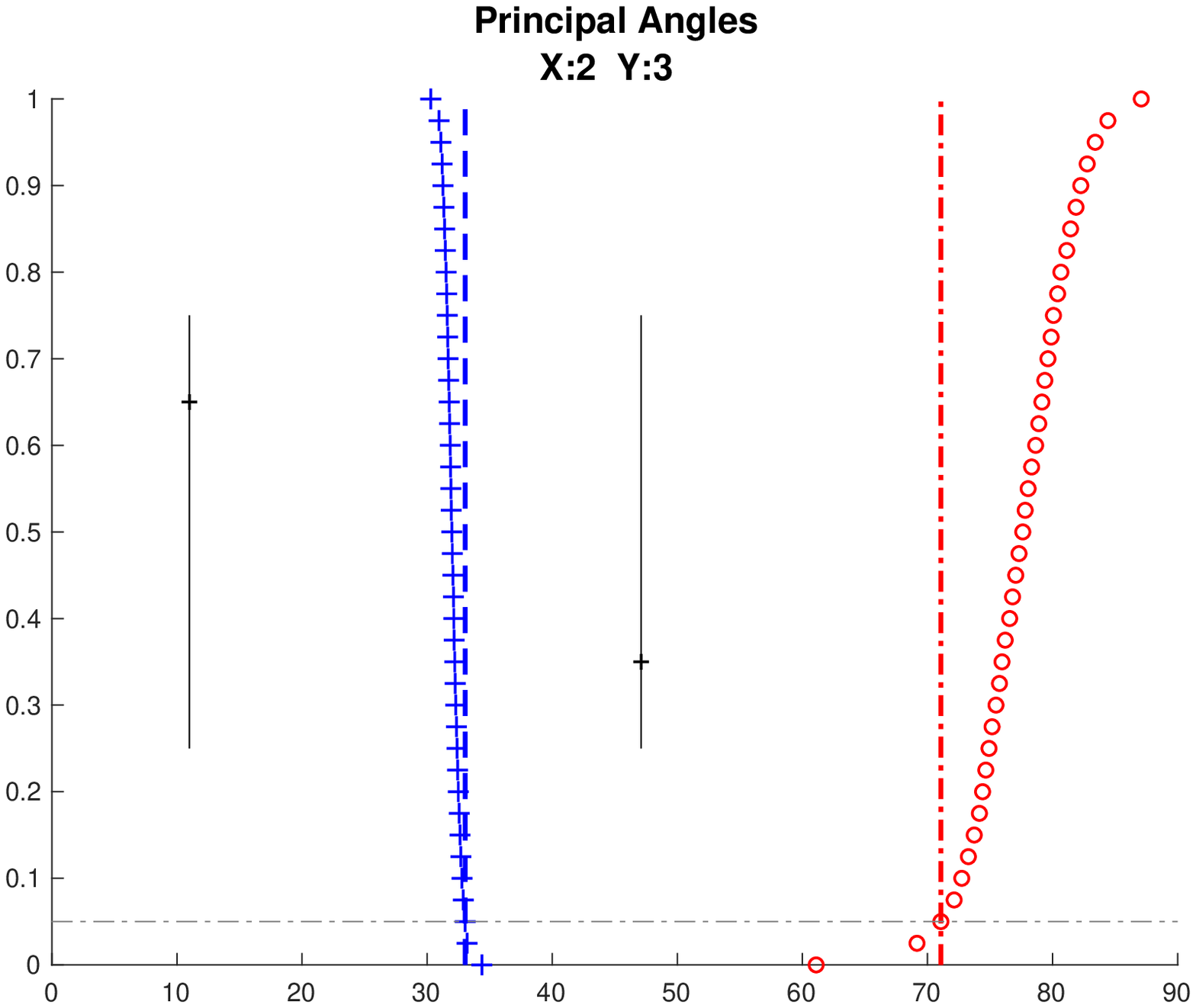}
		\vspace{4ex}
	\end{minipage}
	
	\begin{minipage}[b]{0.5\linewidth}
		\centering
		\includegraphics[width=0.8\linewidth]{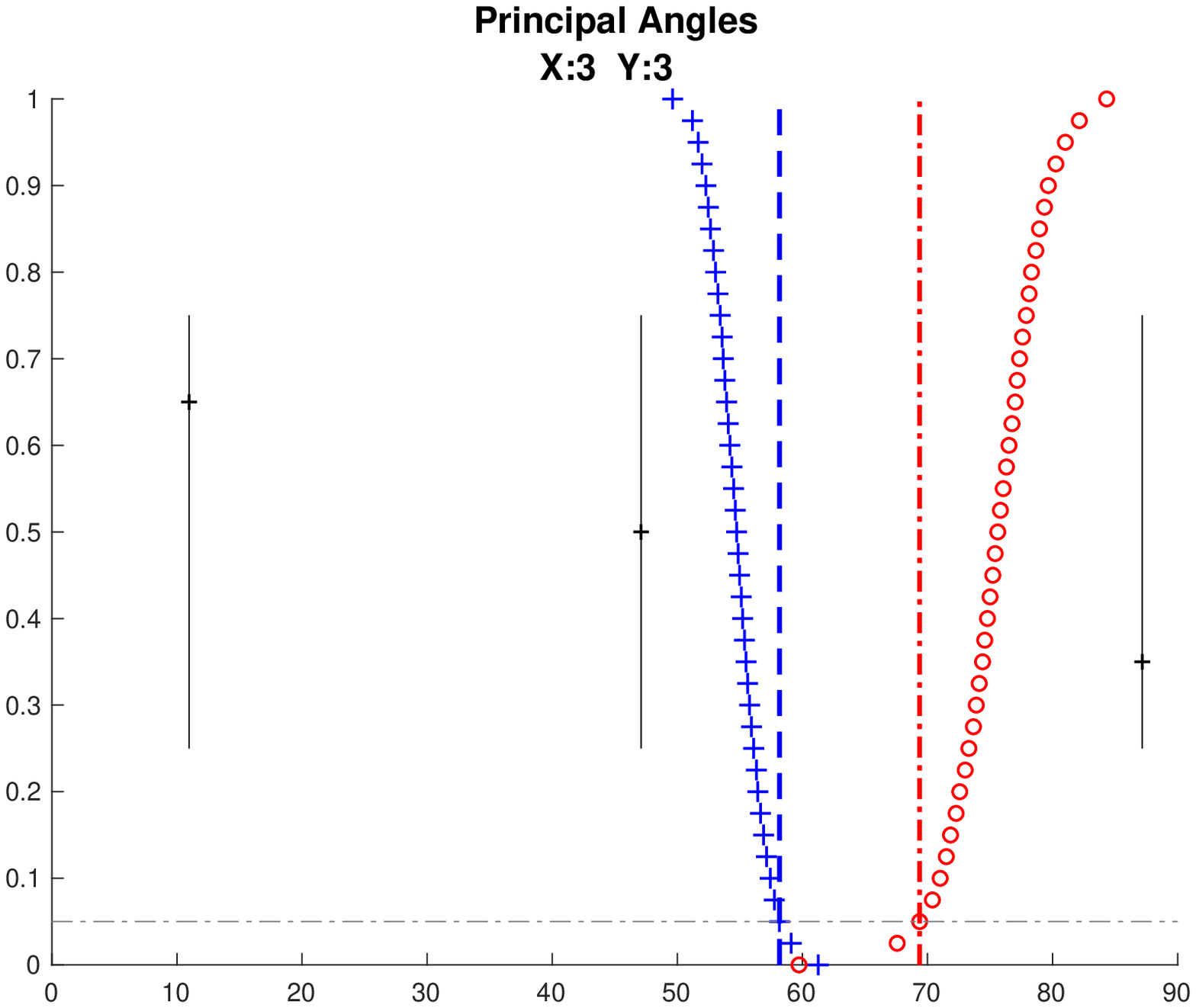}
		\vspace{4ex}
	\end{minipage}
	\begin{minipage}[b]{0.5\linewidth}
		\centering
		\includegraphics[width=0.8\linewidth]{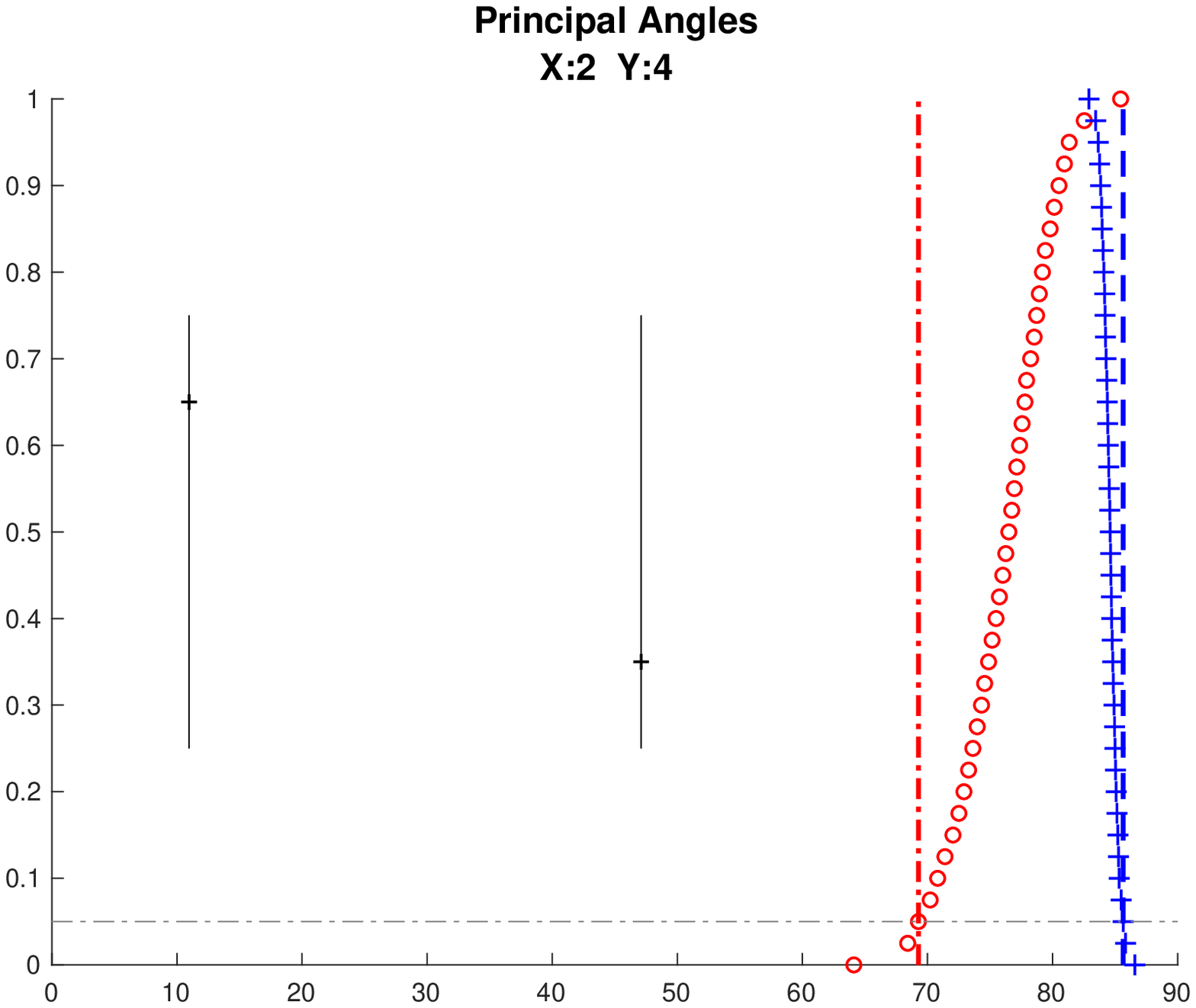}
		\vspace{4ex}
	\end{minipage}
	\caption{Principal angles and angle bounds used for segmentation in Step 2 of AJIVE for various input ranks. In each subfigure, the $x$-axis shows the angle and the $y$-axis shows the probabilities of the simulated distributions. The vertical black line segments are the values of the principal angles between $\row(\tilde{A}_1)$ and $\row(\tilde{A}_2)$, $\phi_1, \ldots, \phi_{\tilde{r}_1 \wedge \tilde{r}_2}$. The red circles show the values of the cumulative distribution function of the random direction distribution; the red dot-dashed line shows the 5th percentile of these angles. The blue plus signs show the values of the survival functions of the resampled Wedin bounds; the blue dashed line is the 95th percentile of the distribution. This figure contains several diagnostic plots, which provide guidance for rank selection. See Section \ref{c:jive-s:method-subsec:step2-2blocks} for details.}
	\label{fig:diagnostic_angle_plot}
\end{figure}

\subsubsection{Multi-block case}
\label{c:jive-s:method-subsec:step2-multiblocks}
To generalize the above idea to more than two blocks, we focus on singular values rather than on principal angles in Eq.~(\ref{equ: principalangle}). In other words, instead of finding an upper bound on an angle, we will focus on a corresponding lower bound on the remaining energy as expressed by the sum of the squared singular values. Hence, an analogous SVD will be used for studying the closeness of multiple initial signal score subspace estimates. 

For the vertical concatenation of right singular vector matrices
\begin{equation}
\label{equ:JointM}
M \triangleq \begin{bmatrix}
\tilde{V}_1^\top\\
 \vdots\\ 
 \tilde{V}_K^\top
\end{bmatrix} = U_{M}\Sigma_{M}V_{M}^\top,
\end{equation} 
SVD sorts the directions within these $K$ subspaces in increasing order of amount of deviation from the theoretical joint direction. The squared singular value $\sigma_{M,i}^2$ indicates the total amount of variation explained in the common direction $V_{M,i}^\top$ in the score subspace of $\mathbb{R}^n$. A large value of $\sigma_{M,i}^2$ (close to $K$) suggests that there is a set of $K$ basis vectors within each subspace that are close to each other and thus are potential noisy versions of a common joint score vector. As in Section~\ref{c:jive-s:method-subsec:step2-2blocks}, the random direction bound and the Wedin bound for these singular values are used for segmentation of this joint and individual components in the multi-block case. 

\paragraph{Random direction bound}

The extension of the random direction bound in Section~\ref{c:jive-s:method-subsec:step2-2blocks} is straightforward. The distribution of the largest squared singular value in \eqref{equ:JointM} generated by random subspaces is also obtained by simulation. As in the two block case, each $\tilde{V}_k$ in $M$ is replaced by an independent random subspace, i.e., right multiplied by an independent orthonormal matrix. The simulated distribution of the largest singular value of $M$ indicates singular values potentially driven by pure noise. For the toy example, the values of the survival function of this distribution are shown as red circles in Figure~\ref{fig:diagnostic_ssvbound_plot}, a singular value analog of Figure~\ref{fig:diagnostic_angle_plot}. The 5th percentile of this distribution, shown as the vertical red dot-dashed line in Figure~\ref{fig:diagnostic_ssvbound_plot}, is used as the random direction bound for squared singular value, which provides 95\% confidence that the squared singular values larger than this cutoff are not generated by random subspaces.

\paragraph{Threshold based on the Wedin bound}

Next is the lower bound for segmentation of the joint space based on the Wedin bound. 

\begin{lemma}
	\label{lemma-bound-multi}
	Let $\theta_{k, \tilde{r}_k \wedge r_k}$ be the largest principal angle between the theoretical subspace $\row(A_k)$ and its estimation $\row(\tilde{A}_k)$ for $K$ data blocks from Eq.~\eqref{equ:perturb_angle}. The squared singular values ($\sigma_{M,i}^2$) corresponding to the estimates of the joint components satisfy
	\begin{equation}
	\label{equ:Mbound}
	\sigma_{M,i}^2 \geq K - \sum_{k=1}^{K}\sin^2\theta_{k, \tilde{r}_k \wedge r_k} \geq K - \sum_{k=1}^{K}\left\{\frac{\max(\|E_k\tilde{V}_k\|, \|E_k^\top\tilde{U}_k\|)}{\sigma_{\min}(\tilde{A}_k)} \wedge 1\right\}^2.
	\end{equation}
\end{lemma}
The proof is provided in \ref{s:proofs}. This lower bound is independent of the variation magnitudes. This property makes AJIVE insensitive to scale heterogeneity across each block when extracting joint variation information.   

As in Section~\ref{c:jive-s:method-subsec:step1-accuracy}, all the terms $\|E_k\tilde{V}_k\|$, $\|E_k^\top\tilde{U}_k\|$ are resampled to derive a distribution estimator for the lower bound in \eqref{equ:Mbound}, which can provide a prediction interval as well. Figure~\ref{fig:diagnostic_ssvbound_plot} shows the values of the cumulative distribution function of this upper bound as blue plus signs for the toy example. As in the two-block case, if there are $\tilde{r}_J$ singular values larger than both this lower bound and the random direction bound, the first $\tilde{r}_J$ right singular vectors are used as the basis of the estimator of $\row(J)$.  

In the two-block case, Figure~\ref{fig:diagnostic_ssvbound_plot} contains essentially the same information as Figure~\ref{fig:diagnostic_angle_plot}, thus the same insights are available. Since principal angles between multiple subspaces are not defined, Figure~\ref{fig:diagnostic_ssvbound_plot} provides appropriate generalization to the multi-block case, see Figure~\ref{fig:tcga_ssvbound_dist}.

\begin{figure}[t!]
	\vspace{0.1in}
	\begin{minipage}[b]{0.5\linewidth}
		\centering
		\includegraphics[width=0.8\linewidth]{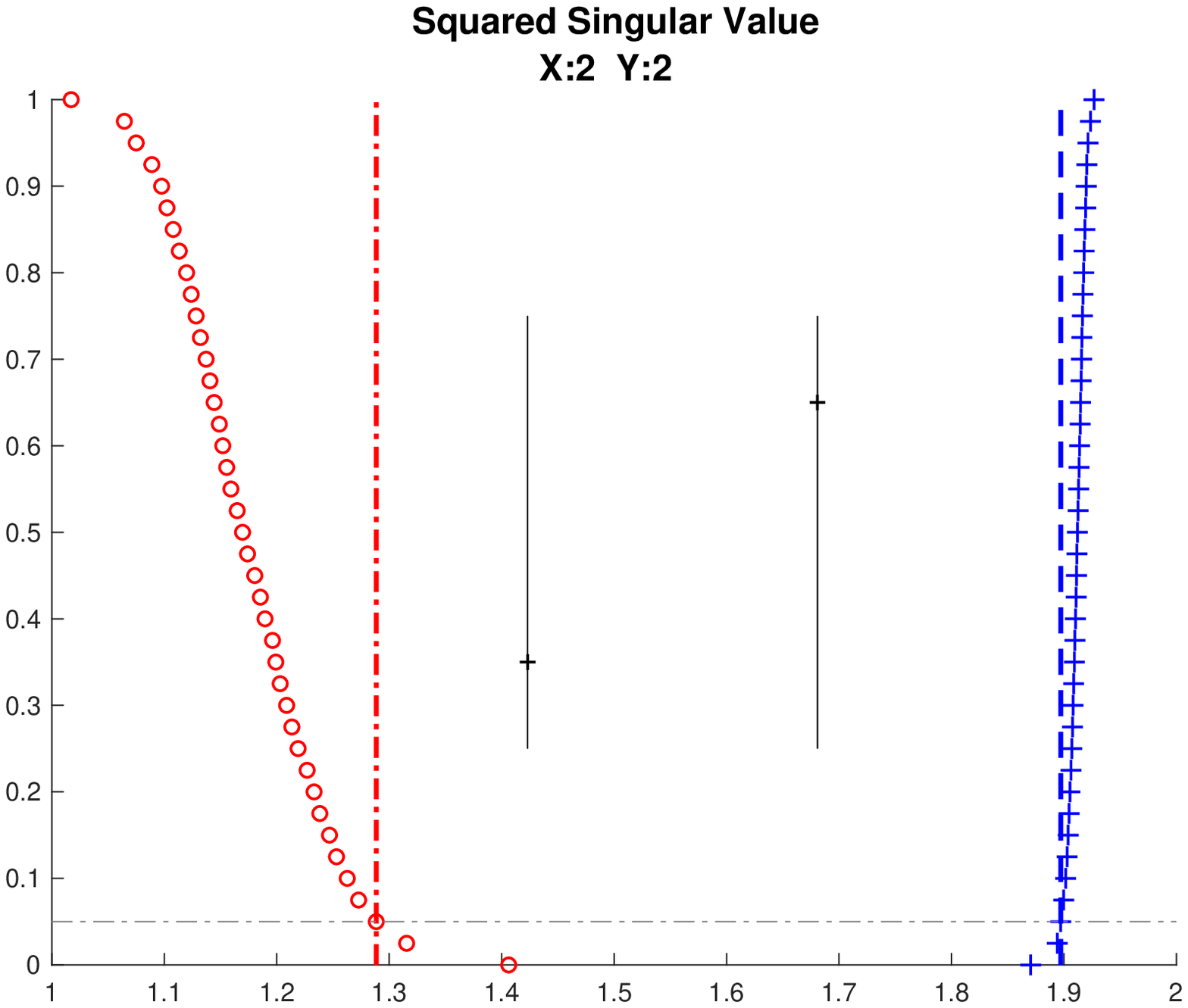}
		\vspace{4ex}
	\end{minipage}
	\begin{minipage}[b]{0.5\linewidth}
		\centering
		\includegraphics[width=0.8\linewidth]{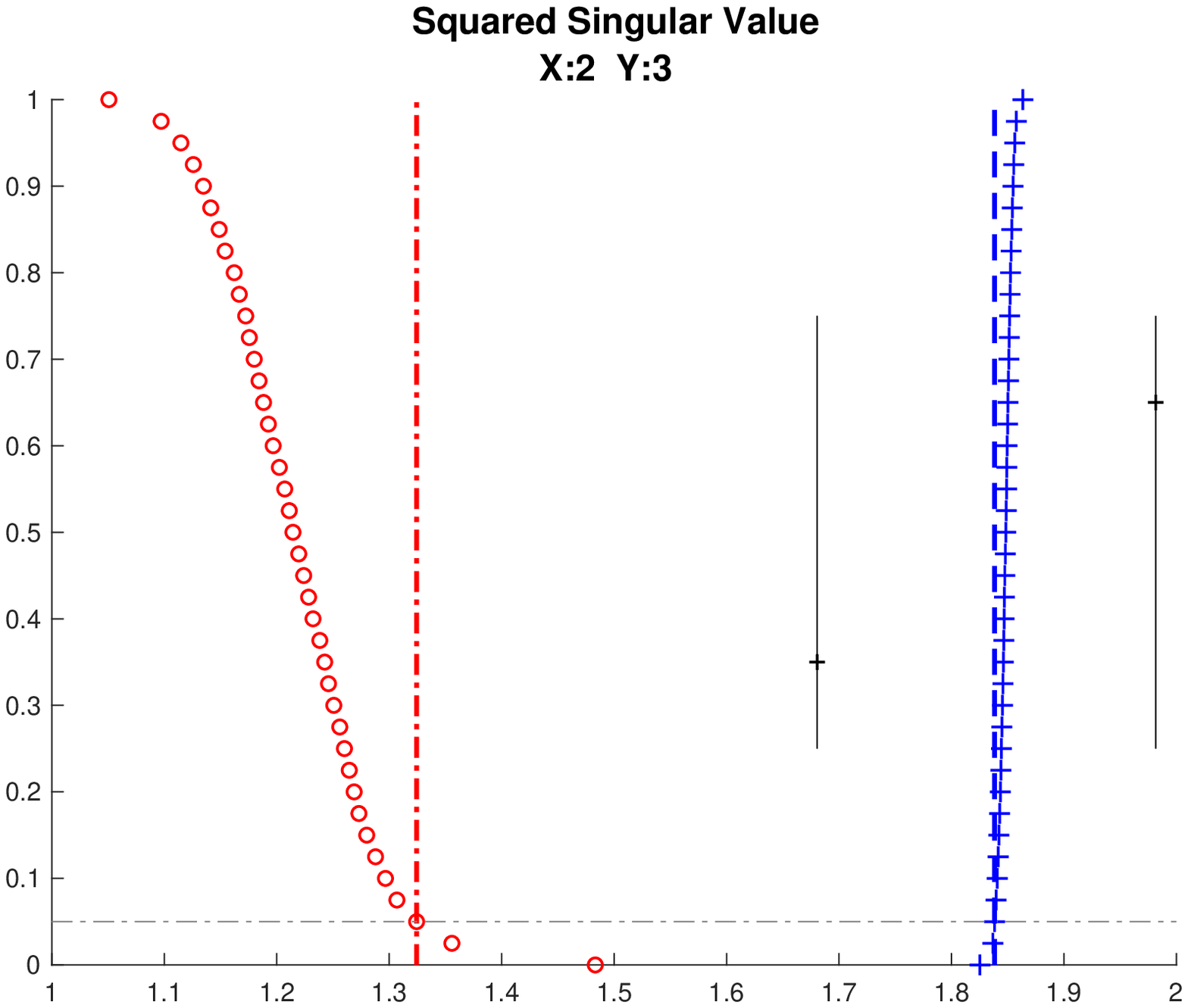}
		\vspace{4ex}
	\end{minipage}
	
	\begin{minipage}[b]{0.5\linewidth}
		\centering
		\includegraphics[width=0.8\linewidth]{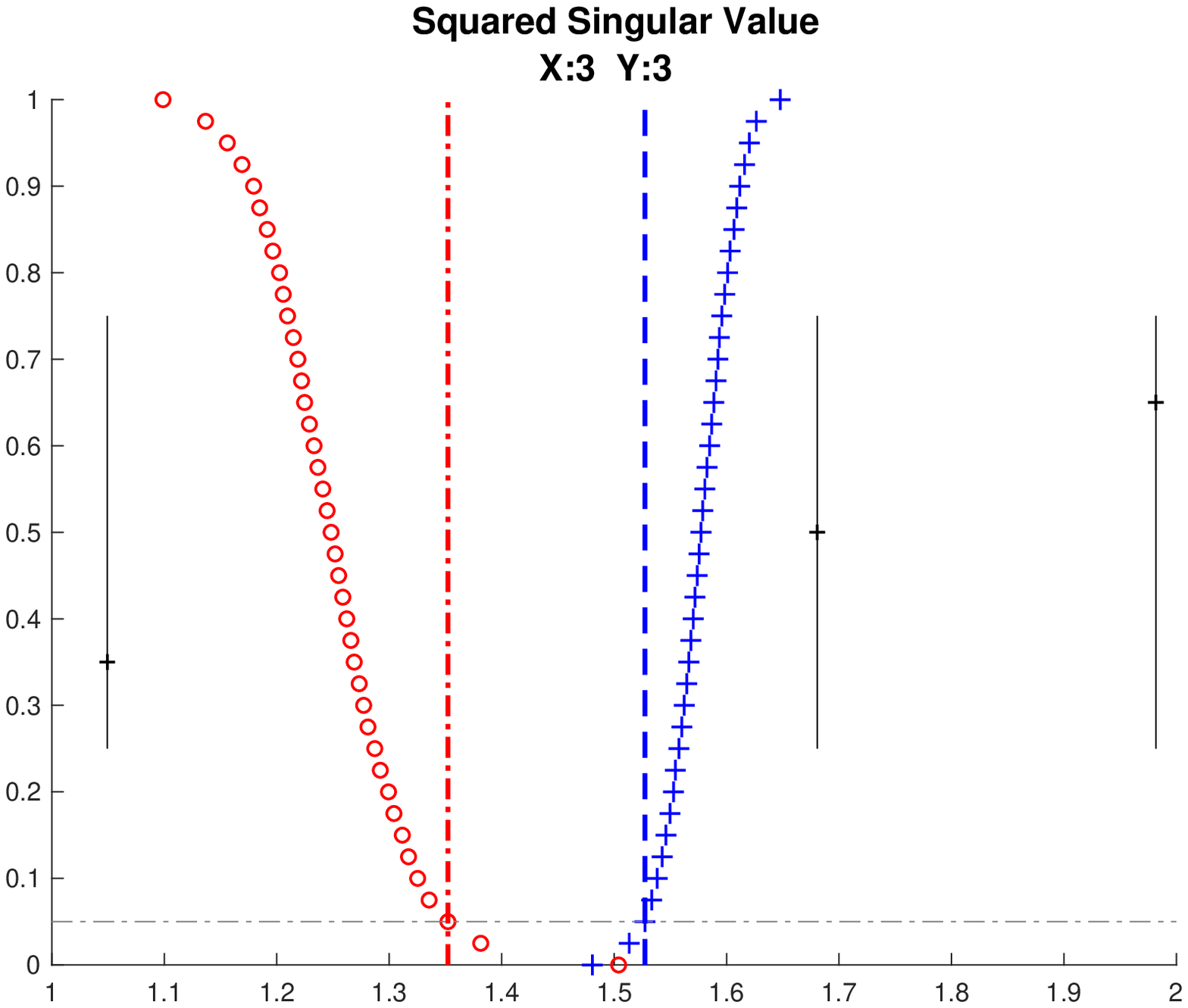}
		\vspace{4ex}
	\end{minipage}
	\begin{minipage}[b]{0.5\linewidth}
		\centering
		\includegraphics[width=0.8\linewidth]{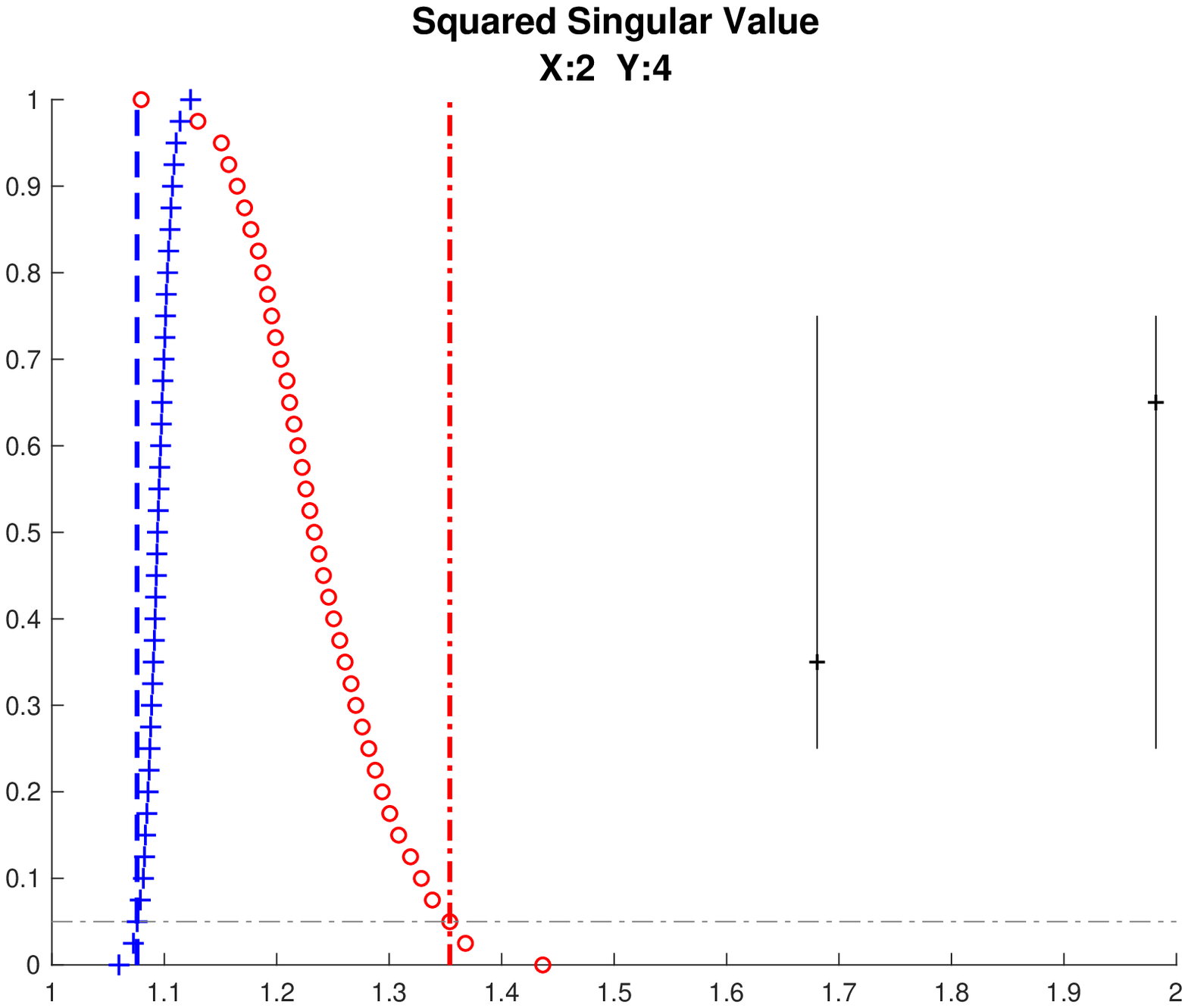}
		\vspace{4ex}
	\end{minipage}
	\caption{Squared singular values in \eqref{equ:JointM} and bounds for Step 2 of AJIVE for various rank choices. The black vertical line segments shows the first $\tilde{r}_1 \wedge \tilde{r}_2$ squared singular values of $M$ in equation~\eqref{equ:JointM}. The values of the survival function of the random direction bounds are shown as the red circles and the red dot-dashed line is the 95th percentile of this distribution, which is the random direction bound. The values of the c.d.f of the Wedin bound are shown as the blue plus signs and the 5th percentile (blue dashed line) is used for a prediction interval for the Wedin bound. In the two-block case presented here this contains the essentially same information as in Figure~\ref{fig:diagnostic_angle_plot}. For the multi-block case it is the major diagnostic graphic.}
	\label{fig:diagnostic_ssvbound_plot}
\end{figure}

\subsection{Step 3: Final decomposition and outputs}
\label{c:jive-s:method-subsec:step3}

Based on the estimate of the joint row space, matrices containing joint variation in each data block can be reconstructed by projecting $X_{k}$ onto this estimated space. Define the matrix $\tilde{V}_{J}$ as $[\vec{v}_{M,1}, \ldots, \vec{v}_{M, \tilde{r}_J}]$, where $\vec{v}_{M,i}$ is the $i$th column in the matrix $V_M$. To ensure that all components continue to satisfy the identifiability constraints from Section~\ref{c:jive-s:method-subsec:step1-threshold}, we check that, for all the blocks, each $\|X_k\vec{v}_{M,i}\|$ is also above the corresponding threshold used in Step~1. If the constraint is not satisfied for any block, that component is removed from $\tilde{V}_{J}$. A real example of this happens in Section~\ref{subsec:TCGA}. An important point is that this removal can happen even when there is a common joint structure in all but a few blocks. 

Denote $\hat{V}_J$ as the matrix $\tilde{V}_{J}$ after this removal and $\hat{r}_J$ as the final joint rank. The projection matrix onto the final estimated joint space $\row(\hat{J})$ is $P_{J}=\hat{V}_{J}\hat{V}_{J}^\top$, represented as the red rectangle in Figure~\ref{fig:flow}. The estimate of the joint variation matrices in block $k \in \{ 1, \ldots, K\}$ is
$\hat{J}_{k} = X_{k}P_{J}$.

The row space of joint structure is orthogonal to the row spaces of each individual structure. Therefore, the original data blocks are projected to the orthogonal space of $\row(\hat{J})$. The projection matrix onto the orthogonal space of $\row(\hat{J})$ is $P_{J}^{\perp} = I-P_{J}$ and the projections of each data block are denoted as $X_{k}^{\perp}$ respectively for each block, i.e.,
$X_{k}^{\perp}=X_{k}P_{J}^{\perp}$.
These projections are represented as the circled minus signs in Figure~\ref{fig:flow}.

Finally we rethreshold this projection by performing SVD on $X_{1}^{\perp}, \ldots, X_{K}^{\perp}$. The components with singular values larger than the first thresholds from Section~\ref{c:jive-s:method-subsec:step1-threshold} are kept as the individual components, denoted as $\hat{I}_{1}, \ldots, \hat{I}_{K}$. The remaining components of each SVD are regarded as an estimate of the noise matrices.

By taking a direct sum of the estimated row spaces of each type of variation, denoted by $\oplus$, the estimated signal row spaces are
\begin{equation*}
\row(\hat{A}_k) = \row(\hat{J}) \oplus \row(\hat{I}_k)
\end{equation*}
with rank $\hat{r}_k = \hat{r}_{J} + \hat{r}_{I_k}$ respectively for each $k \in\{ 1, \ldots, K\}$.

Due to this adjustment of directions of the joint components, these final estimates of signal row spaces may be different from those obtained in the initial signal extraction step. Note that even the estimates of rank $\hat{r}_k$ might also differ from the initial estimates $\tilde{r}_k$.

Given the variation decompositions of each AJIVE component, as shown on the right side of Figure~\ref{fig:flow}, several types of post AJIVE representations are available for representing the joint and individual variation patterns. There are three important matrix representations of the information in the AJIVE joint output, i.e., the boxes on the right in Figure~\ref{fig:flow}, with differing uses in post AJIVE analyses. 
\begin{enumerate}
	\item \emph{Full matrix representation}. For applications where the original features are the main focus (such as finding driving genes), the full $d_k \times n$ matrix representations $\hat{J}_{k}$ and $\hat{I}_k$ with $k \in \{ 1, \ldots, K\}$ are most useful. Thus this AJIVE output is the product of all three blocks in each dashed box on the right side of Figure~\ref{fig:flow}. Examples of these outputs are shown in the two right columns of Figure~\ref{fig:jive:toyjiveoutput}.
	
	\item \emph{Block specific representation}. For applications where the relationships between subjects are the main focus (such as discrimination between subtypes) large computational gains are available by using lower dimensional representations. These are based on SVDs as indicated in the right side of Figure~\ref{fig:flow}, i.e., for each $k \in \{ 1, \ldots, K\}$,
	\begin{equation}
	\label{equ:BSS}
	\hat{J}_k = \hat{U}^{k}_J \hat{\Sigma}^{k}_J \hat{V}^{k\top}_J, \quad \hat{I}_{k} = \hat{U}^{k}_I \hat{\Sigma}^{k}_I \hat{V}^{k\top}_I.
	\end{equation}
	
	The resulting AJIVE outputs include the joint and individual {\it Block Specific Score} (BSS) matrices $\hat{\Sigma}^{k}_J \hat{V}^{k\top}_J$ $(\hat{r}_J \times n)$, $\hat{\Sigma}^{k}_I \hat{V}^{k\top}_I$ $(\hat{r}_{I_k} \times n)$, respectively.
	This results in no loss of information when rotation invariant methods are used. The corresponding {\it Block Specific Loading} matrices are $\hat{U}^{k}_J$ $(d_k \times \hat{r}_J)$ and $\hat{U}^{k}_I$ $(d_k \times \hat{r}_{I_k})$.
	
	\item \emph{Common normalized representation}. Although $\row(\hat{V}_J^{k \top})$ in \eqref{equ:BSS} are the same, the matrices are different. In particular, the rows in \eqref{equ:BSS} can be completely different across $k$, because they are driven by the pattern of the singular values in each $\hat{\Sigma}_J^k$. In some applications, correspondence of components across data blocks is important. In this case the analysis should use a common basis of $\row(\hat{J})$, namely $\hat{V}_J^\top$ $(\hat{r}_J \times n)$, called the {\it Common Normalized Scores} (CNS). This is shown as the gray rectangular near the center of Figure~\ref{fig:flow}. To get the corresponding loadings, we regress $\hat{J}_k$ on each score vector in $\hat{V}_J^\top$ (which is computed as $\hat{J}_k \hat{V}_J$) following by normalization. By doing this, there is no guarantee of orthogonality between CNS loading vectors. However, the loadings are linked across blocks by their common scores. For studying scale free individual spaces, use the \emph{Individual Normalized Scores} (INS)  $\hat{V}^{k \top}_I$ ($\hat{r}_{I_k} \times n$). The individual loading matrices $\hat{U}^{k}_I$ are the same as the block specific individual loadings.
\end{enumerate}

The relationship between Block Specific Representation and Common Normalized Representation is analogous to that of the traditional covariance, i.e., PLS, and correlation, i.e., CCA, modes of analysis. The default output in the AJIVE software is the Common Normalized Representation.

\section{Data analysis}
\label{c:jive-s:datanalaysis}

In this section, we apply AJIVE to two real data sets, TCGA breast cancer in Section~\ref{subsec:TCGA} and Spanish mortality in Section~\ref{subsec:mortality}. 

\subsection{TCGA Data}
\label{subsec:TCGA}
A prominent goal of modern cancer research, of which The Cancer Genome Atlas~\citep{cancer2012comprehensive} is a major resource, is the combination of biological insights from multiple types of measurements made on common subjects.

TCGA provides prototypical data sets for the application of AJIVE. Here we study the 616 breast cancer tumor samples from~\citet{ciriello2015comprehensive}, which had a common measurement set. For each tumor sample, there are measurements of $16615$ gene expression features (GE), $24174$ copy number variations features (CN), $187$ reverse phase protein array features (RPPA) and $18256$ mutation features (Mutation). These data sources have very different dimensions and scalings.

The tumor samples are classified into four molecular subtypes: Basal-like, HER2, Luminal A and Luminal B. An integrative analysis targets the association among the features of these four disparate data sources that jointly quantify the differences between tumor subtypes. In addition, identification of driving features for each source and subtype is obtained from studying loadings. 

Scree plots were used to find a set of interesting candidates for the initial ranks selected in Step 1. Various combinations of them were investigated using the diagnostic graphic. Four interesting cases are shown in Figure~\ref{fig:tcga_ssvbound_dist}. The upper left panel of Figure~\ref{fig:tcga_ssvbound_dist} is a case where the input ranks are too small, resulting in no joint components being identified, i.e., all the black lines are smaller than the dashed blue estimated Wedin bound. The upper right panel shows a case where only one joint component is identified. In addition to the joint component identified in the upper right panel, the lower left panel contains a second potential joint component close to the Wedin bound. The lower right panel shows a case where the Wedin bound becomes too small since the input ranks are too large. Many components are suggested as joint here, but these are dubious because the Wedin bound is smaller than the random direction bound. Between the two viable choices, in the upper right and the lower left, we investigate the latter in detail, as it best highlights important fine points of the AJIVE algorithm. In particular, we choose low-rank approximations of dimensions $20$ (GE), $16$ (CN), $15$ (RPPA) and $27$ (Mutation). However, detailed analysis of the upper right panel results in essentially the same final joint component. After selection of the threshold in Step~1, it took AJIVE 298 seconds (5.0 minutes, on Macbook Pro Mid 2012, 2.9 GHz) to finish Steps~2 and 3.

\begin{figure}[htb!]
	\vspace{0.1in}
	\begin{minipage}[b]{0.5\linewidth}
		\centering
		\includegraphics[width=0.8\linewidth]{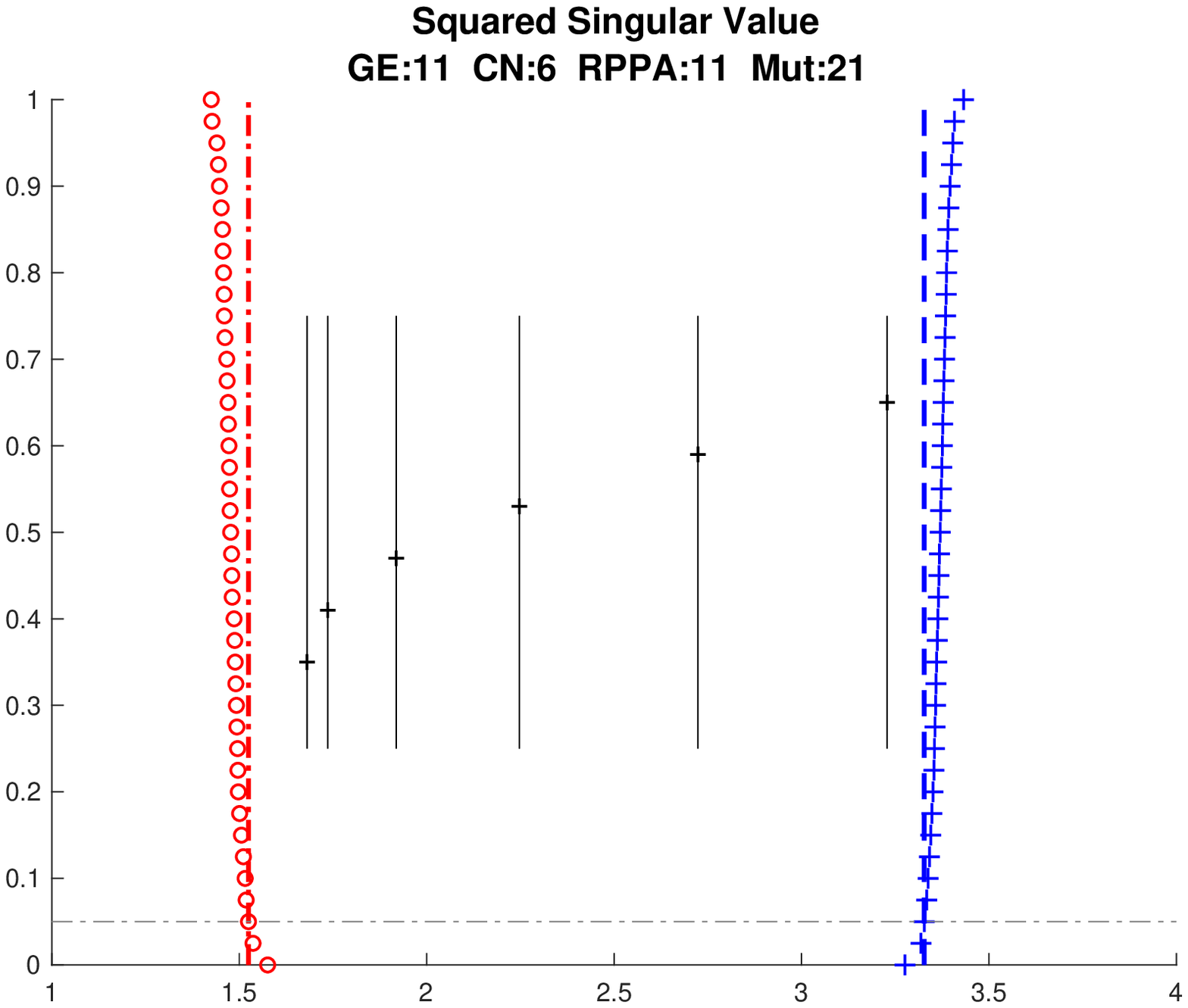}
		\vspace{4ex}
	\end{minipage}
	\begin{minipage}[b]{0.5\linewidth}
		\centering
		\includegraphics[width=0.8\linewidth]{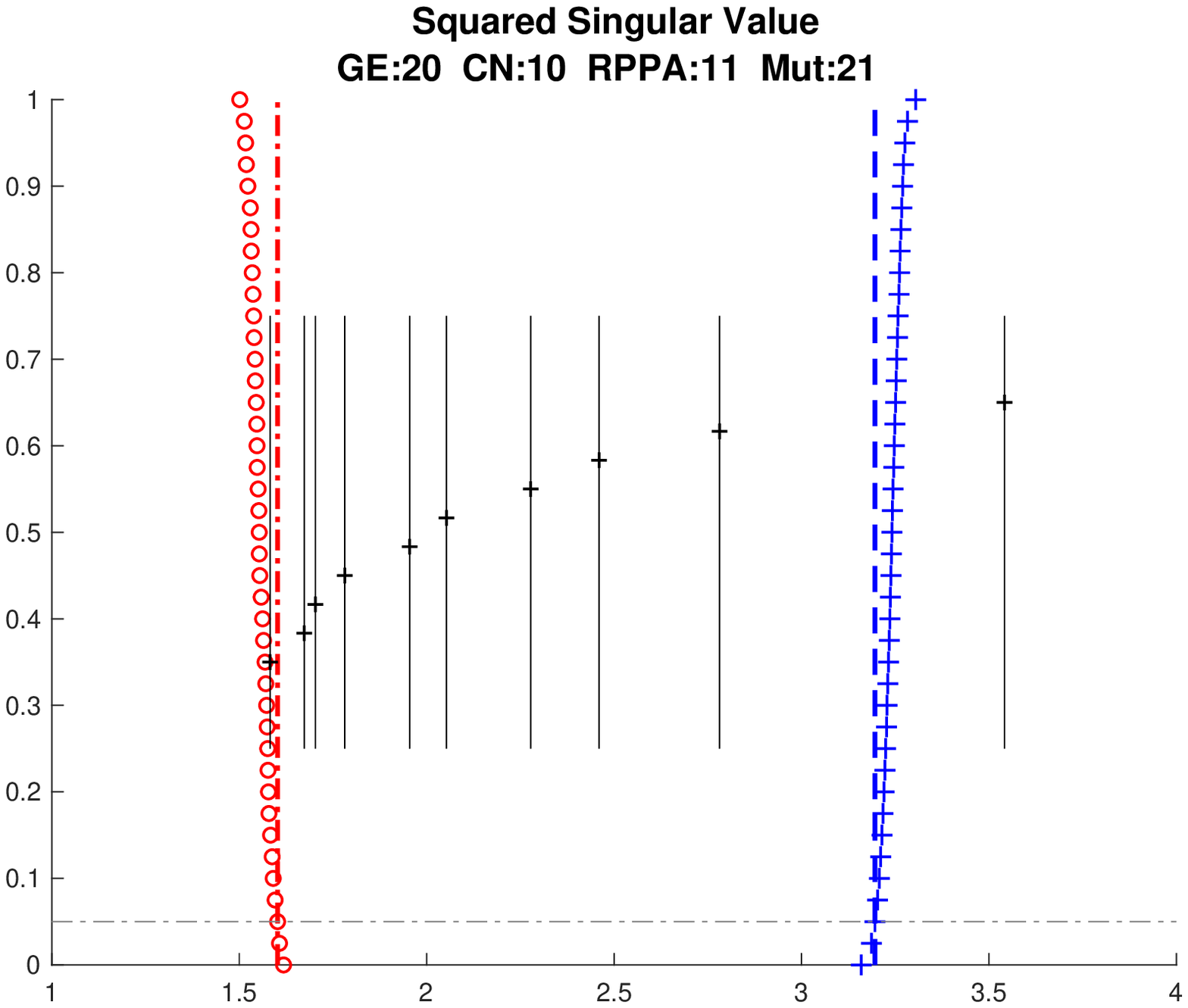}
		\vspace{4ex}
	\end{minipage}
	
	\begin{minipage}[b]{0.5\linewidth}
		\centering
		\includegraphics[width=0.8\linewidth]{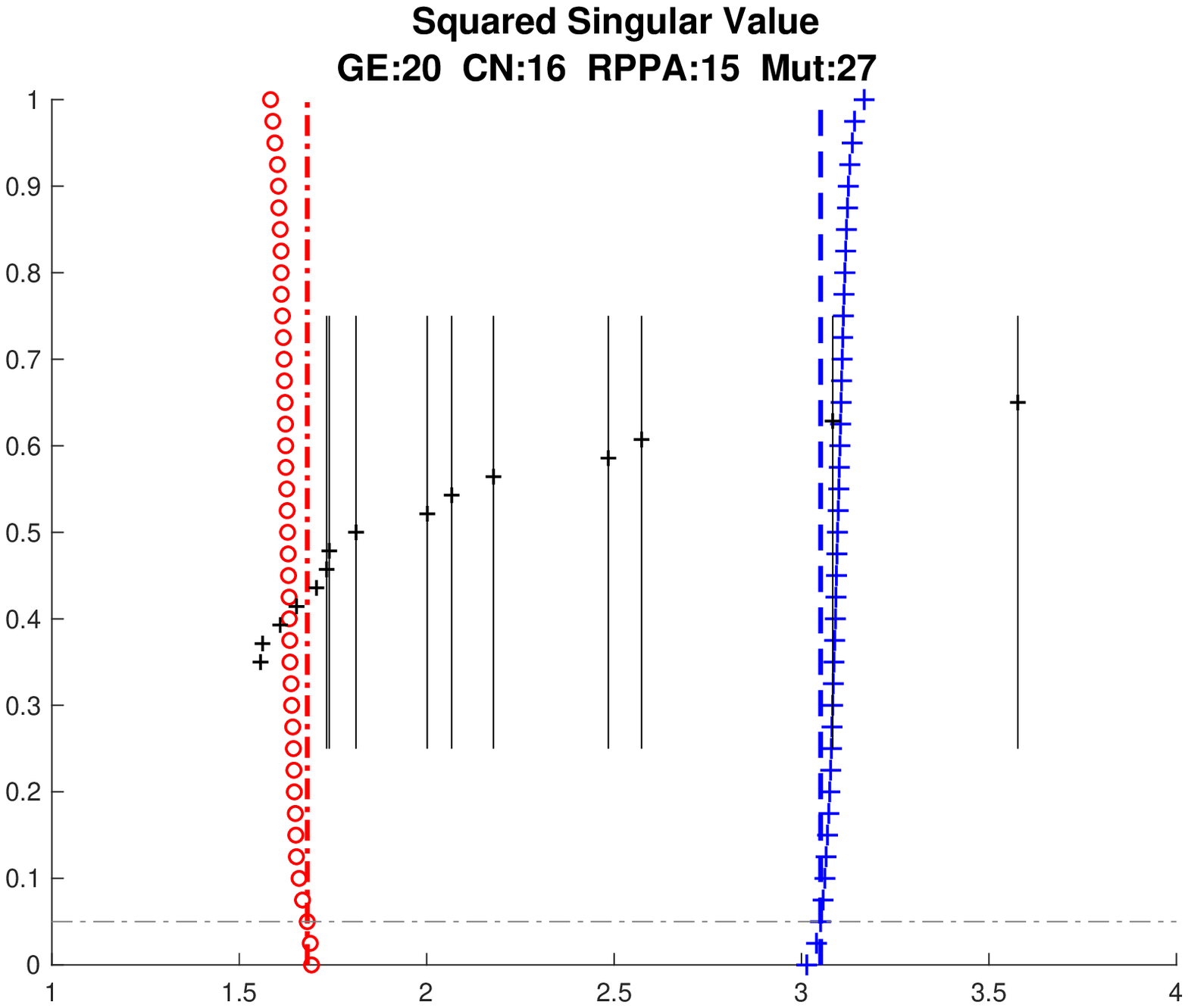}
		\vspace{4ex}
	\end{minipage}
	\begin{minipage}[b]{0.5\linewidth}
		\centering
		\includegraphics[width=0.8\linewidth]{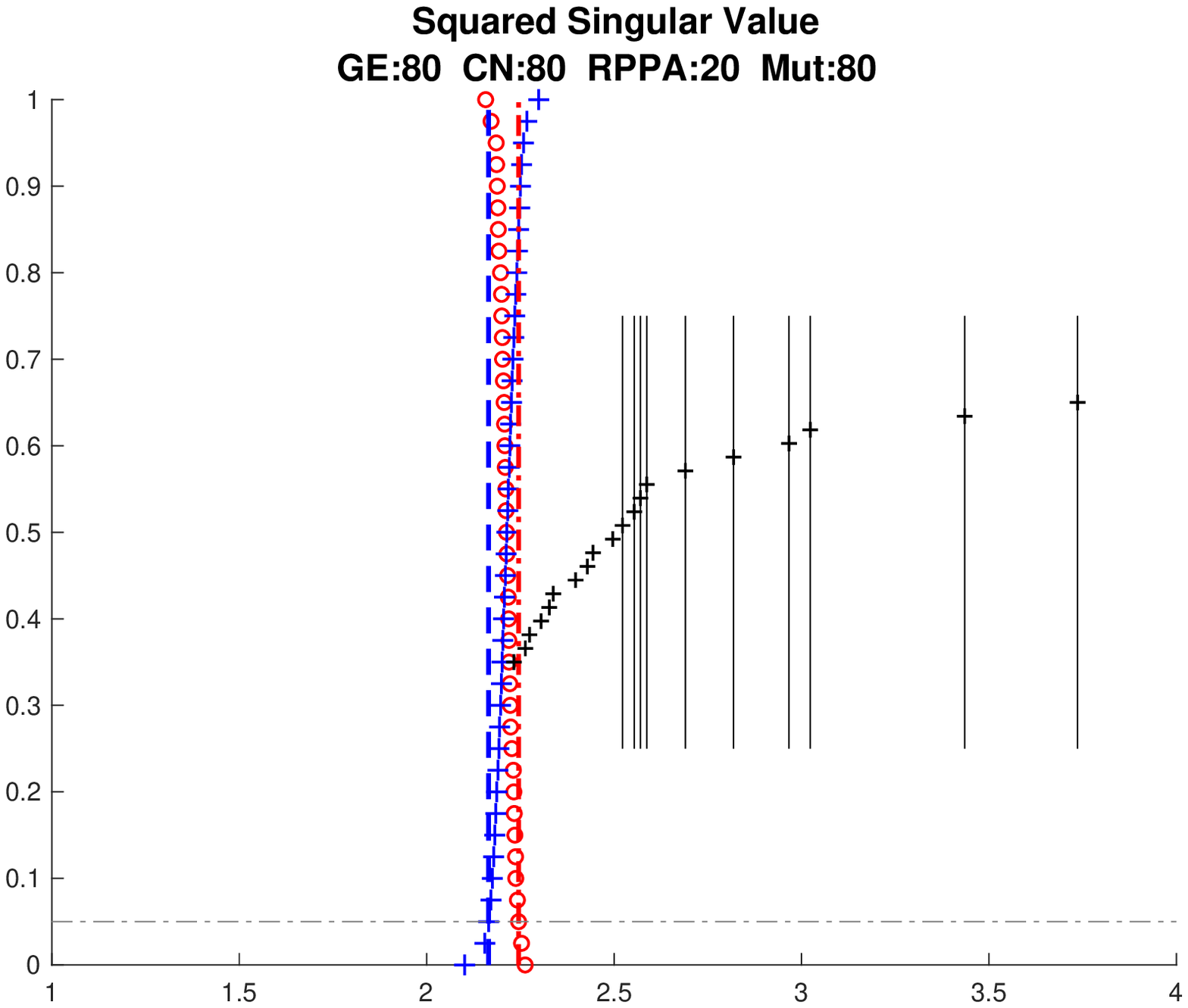}
		\vspace{4ex}
	\end{minipage}
	\caption{Squared singular value diagnostic graphics for TCGA dataset over various rank choices. Indicates that there are one joint component among four data blocks and one joint component among three data blocks.}
	\label{fig:tcga_ssvbound_dist}
\end{figure}

In the second AJIVE step, the one sided 95\% prediction interval suggested selection of two joint components. However, the third step indicated dropping one joint component, because the norm of the projection of the mutation data on that direction, i.e., the second CNS, is below the threshold from Step 1. This result of one joint component was consistent with the expectation of cancer researchers, who believe the mutation component has only one interesting mode of variation. A careful study of all such projections shows that the other data types, i.e., GE, CN and RPPA, do have a common second joint component as discussed at the end of this section. The association between the CNS and genetic subtype differences is visualized in the left panel of Figure~\ref{fig:tcgamultijive:joint1}. The dots are a jitter plot of the patients, using colors and symbols to distinguish the subtypes (Blue for Basal-like, cyan for HER2, red for Luminal A and magenta for Luminal B). Each symbol is a data point whose horizontal coordinate is the value and vertical coordinate is the height based on data ordering. The curves are Gaussian kernel density estimates, i.e., smoothed histograms, which show the distribution of the subtypes.

\begin{figure}[htb!]
	\centering
	\vspace{4ex}
	\begin{minipage}[b]{0.5\textwidth}
		\centering
		\includegraphics[scale=0.42]{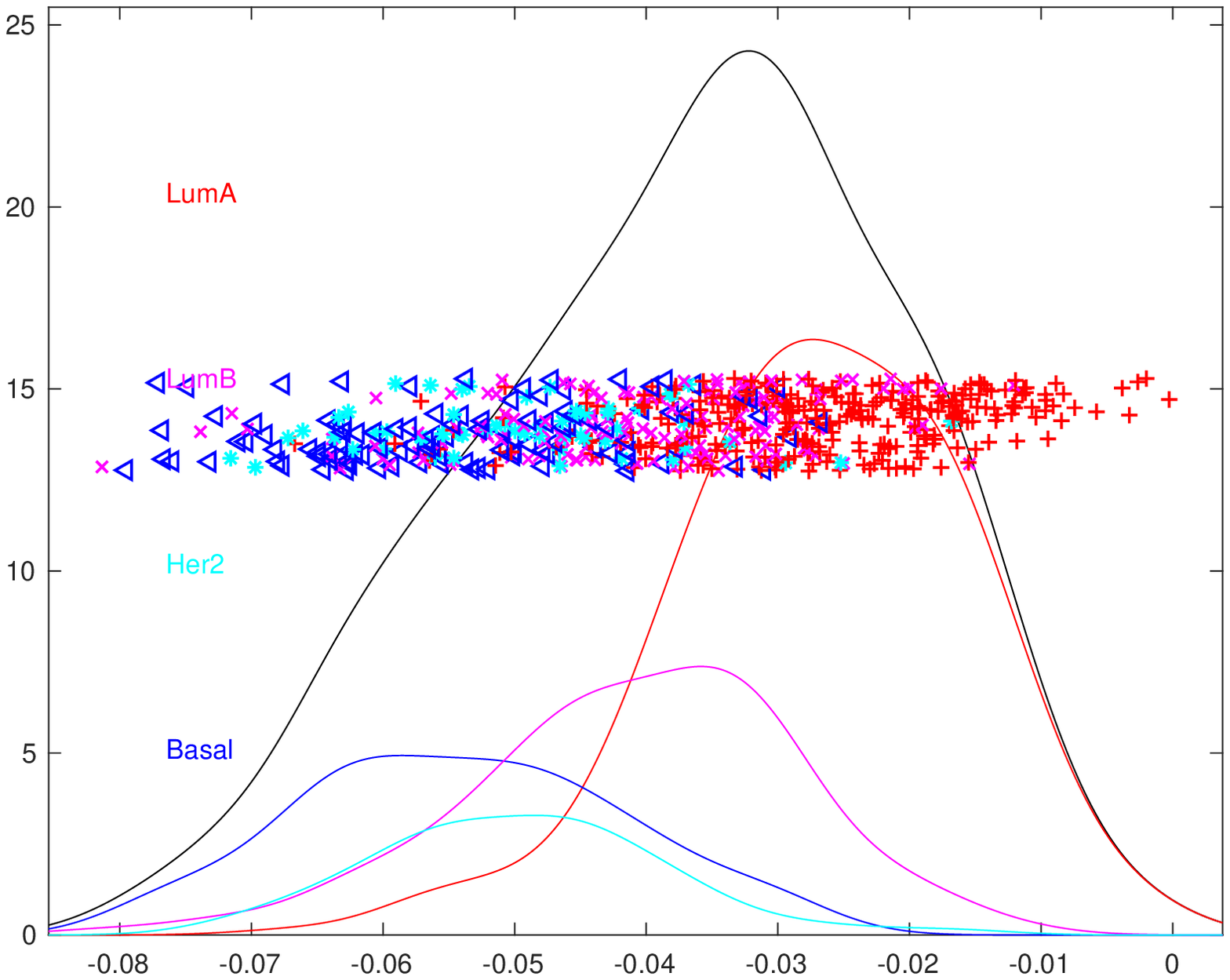}
	\end{minipage}
	\begin{minipage}[b]{0.5\textwidth}
		\centering
		\includegraphics[scale=0.42]{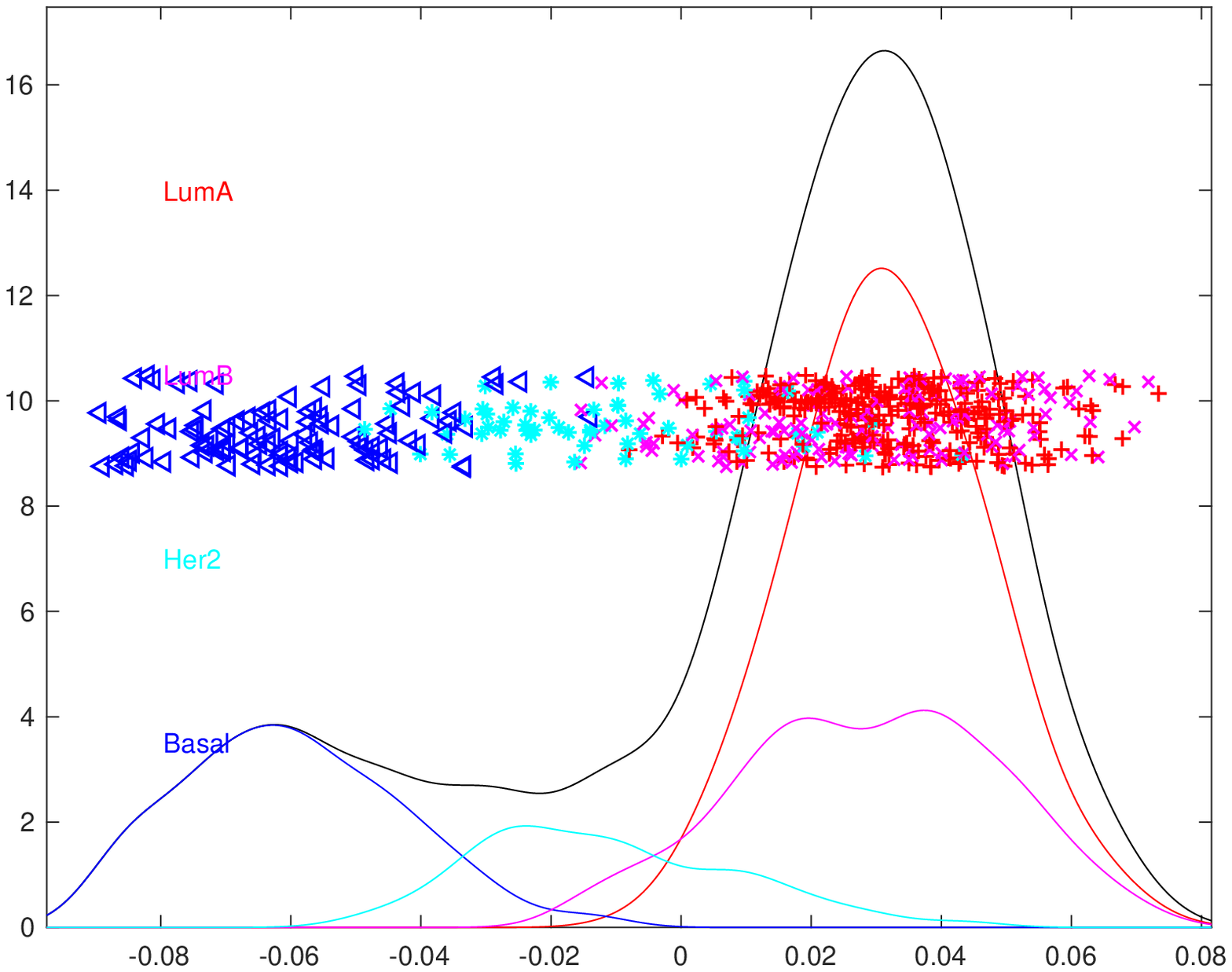}
	\end{minipage}
	\caption{Left: Kernel density estimates of the CNS among GE, CN, RPPA and mutation. The clear separation among Luminal A, versus Her2 and Basal indicates that these four data blocks share a very strong Luminal A property captured in this joint variation component; Right: The CNS from applying AJIVE to the individual matrices of GE, CN, and RPPA. The clear separation indicates that these contain a joint variation component that is consistent with the subtype difference between Basal versus the others.}
	\label{fig:tcgamultijive:joint1}
\end{figure}

The clear separation among density estimates suggest that this joint variation component is strongly connected with the subtype difference between Luminal A versus the other subtypes. To quantify this subtype difference, a test is performed using the CNS of this joint component evaluated by the DiProPerm hypothesis test~\citep{wei2016direction} based on 100 permutations. Strength of the evidence is usually measured by permutation $p$-values. However, in this context empirical $p$-values are frequently zero. Thus a more interpretable measure of strength of the evidence is the DiProPerm $z$-score. This is $26.54$ for this CNS. An area under the receiver operating characteristic (ROC) curve (AUC)~\citep{hanley1982meaning} of $0.878$, is also obtained to reflect the classification accuracy. These numbers confirm the strong Luminal~A property shared by these four data types. 

A further understanding can be obtained by identifying the feature set of each data type which jointly works with the others in characterizing the Luminal A property. By studying the loading coefficients, important mutation features TP53,TTN and PIK3CA are identified which are well known features from previous studies. Similarly the strong role played by GATA3 in RPPA is well known, and is connected with the large GATA3 mutation loading. A less well known result of this analysis is the genes appearing with large GE loadings. Many of these were not flagged in earlier studies, which had focused on subgroup separation, instead of joint behavior. 

As noted in the discussion of Step~2 above, all four data types have only one significant joint component. However, the individual components for all of GE, CN and RPPA seem to have $3$-way joint components. This is investigated by performing a second AJIVE analysis. In particular, we apply the second and third step to the three individual variation matrices from the initial analysis. Notice that all individual matrices are low-rank and thus the first step is not necessary. The AJIVE analysis results in one joint variation component which is displayed in the right panel of Figure~\ref{fig:tcgamultijive:joint1}. This joint variation component clearly shows the differences among Basal, HER2 and Luminal subtypes. In particular, a subtype difference between Basal-like versus the others is quantified using the DiProPerm $z$-score ($31.60$) and the AUC ($0.996$). Considering the fact that the AUC of the classification between Basal-like versus the others using all the original separate GE features is $0.999$, this single joint component contains almost all the variation information for separating Basal-like from the others. This hierarchical application of AJIVE reveals an important joint component that is specific to GE, CN and RPPA but not to Mutation. 

\subsection{Spanish mortality data}
\label{subsec:mortality}

A quite different data set from the Human Mortality Database is studied here, which consists of both Spanish males and females. For each gender data block, there is a matrix of \emph{mortality}, defined as the number of people who died divided by the total, for a given age group and year. Because mortality varies by several orders of magnitude, the $\log_{10}$ mortality is studied here. Each row represents an age group from 0 to 95, and each column represents a year between 1908 and 2002. In order to associate the historical events with the variations of mortality, columns, i.e., mortality as a function of age, are considered as the common set of data objects of each gender block. \citet{marron2014overview} performed analysis on the male block and showed interesting interpretations related to Spanish history. Here we are looking for a deeper analysis which integrates both males and females by exploring joint and individual variation patterns. 

AJIVE is applied to the two gender blocks centered by subtracting the mean of each age group. The principal angle diagnostic graphics introduced in Section~\ref{c:jive-s:method-subsec:step2} are provided for this mortality dataset over various rank choices in Figure~\ref{fig:mortality_angle_diagnostic} to guide the selection of initial ranks in Step 1. The upper left panel shows the case $\tilde{r}_1= \tilde{r}_2 = 1$. The only principal angle is larger than the 95th percentile of the resampled Wedin bound and thus we conclude that no joint space is identified. The upper right panel shows the effect of increasing the initial rank choices to $\tilde{r}_1= \tilde{r}_2 = 2$. In this case, the first principal angle becomes smaller. Because it is smaller than the Wedin bound it is identified as a joint component. The second principal angle is still larger than the Wedin bound. Thus we concluded that only one joint component is identified in this case. In the lower left panel we increase the input rank of male mortality to 3. The second principal angle becomes much smaller, in particular smaller than the Wedin bound, and thus is also labeled as joint component. This indicates that the third principal component of male mortality contains joint information. The lower right panel shows the case where $\tilde{r}_1 = 4, \tilde{r}_2 = 5.$ In this case the two smallest principal angle are unchanged (and still joint). Two more principal angle appear. One is larger than the random direction bound, and thus cannot be distinguished from pure noise. The other is just inside the boundary of the much increased Wedin bound suggesting correlation among individual components. Based on these, the choice $\tilde{r}_1 = 3, \tilde{r}_2 = 2$ is used in the subsequent analysis.

\begin{figure}[htb!]
	\vspace{0.1in}
	\begin{minipage}[b]{0.5\linewidth}
		\centering
		\includegraphics[width=0.8\linewidth]{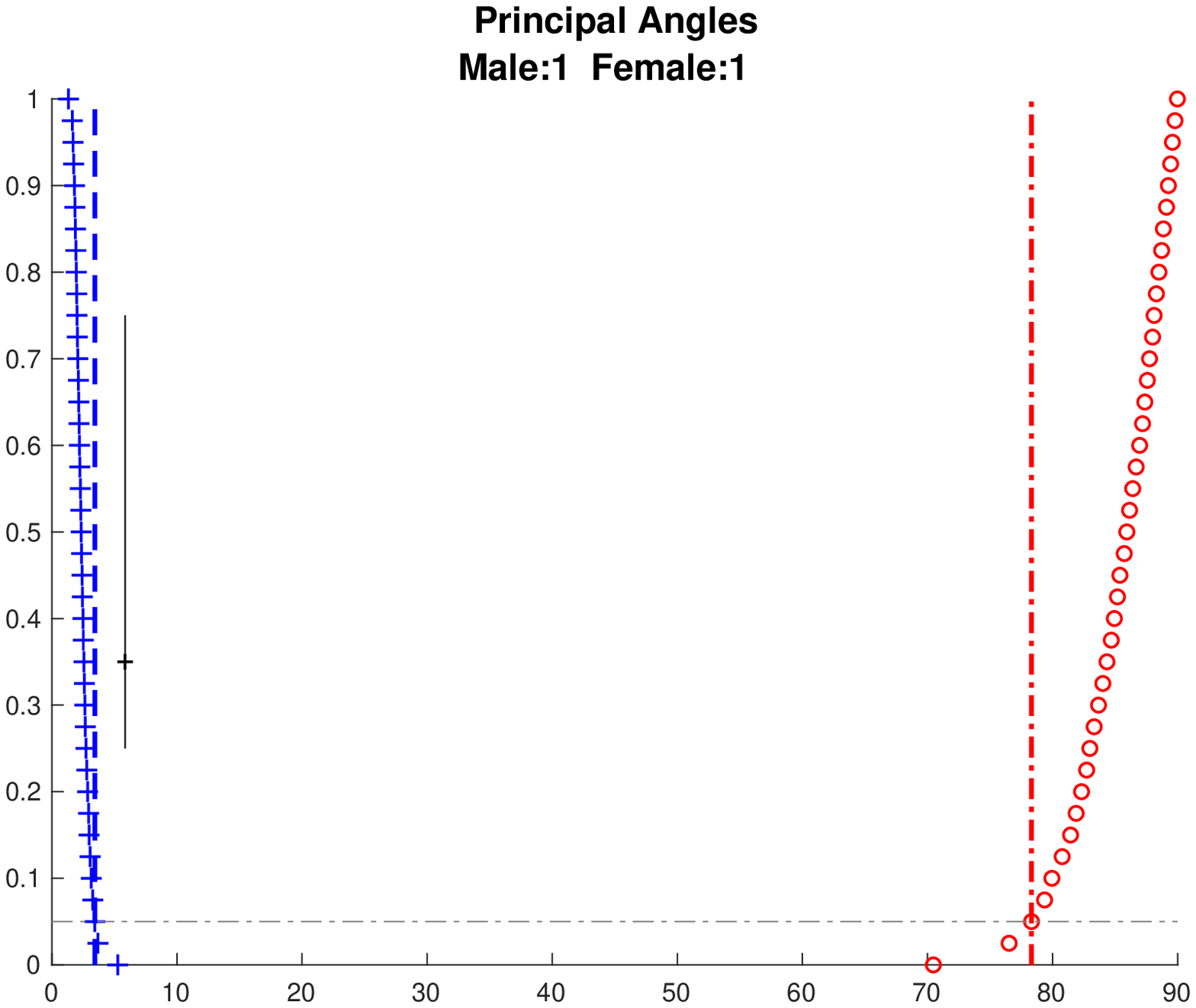}
		\vspace{4ex}
	\end{minipage}
	\begin{minipage}[b]{0.5\linewidth}
		\centering
		\includegraphics[width=0.8\linewidth]{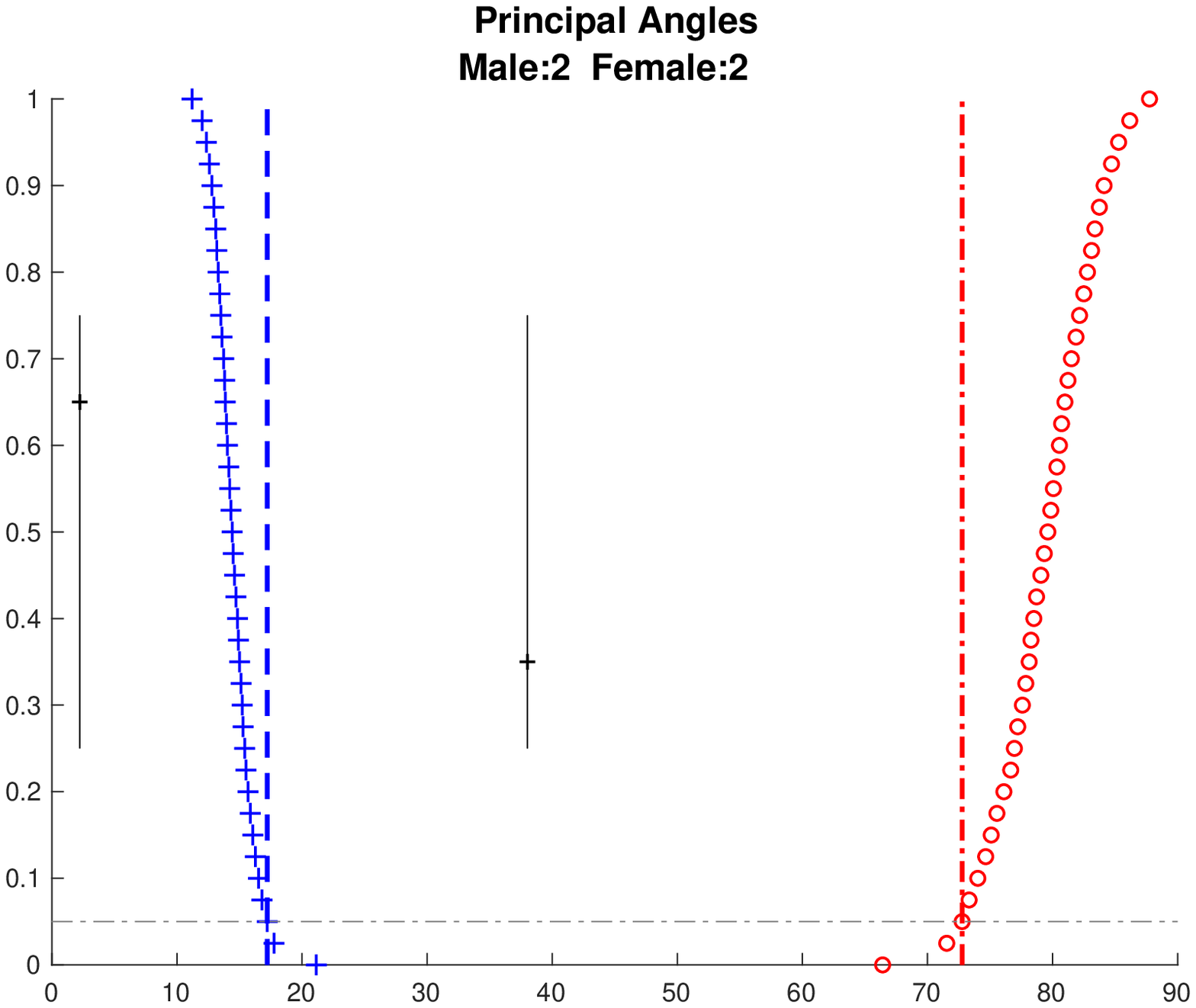}
		\vspace{4ex}
	\end{minipage}
	
	\begin{minipage}[b]{0.5\linewidth}
		\centering
		\includegraphics[width=0.8\linewidth]{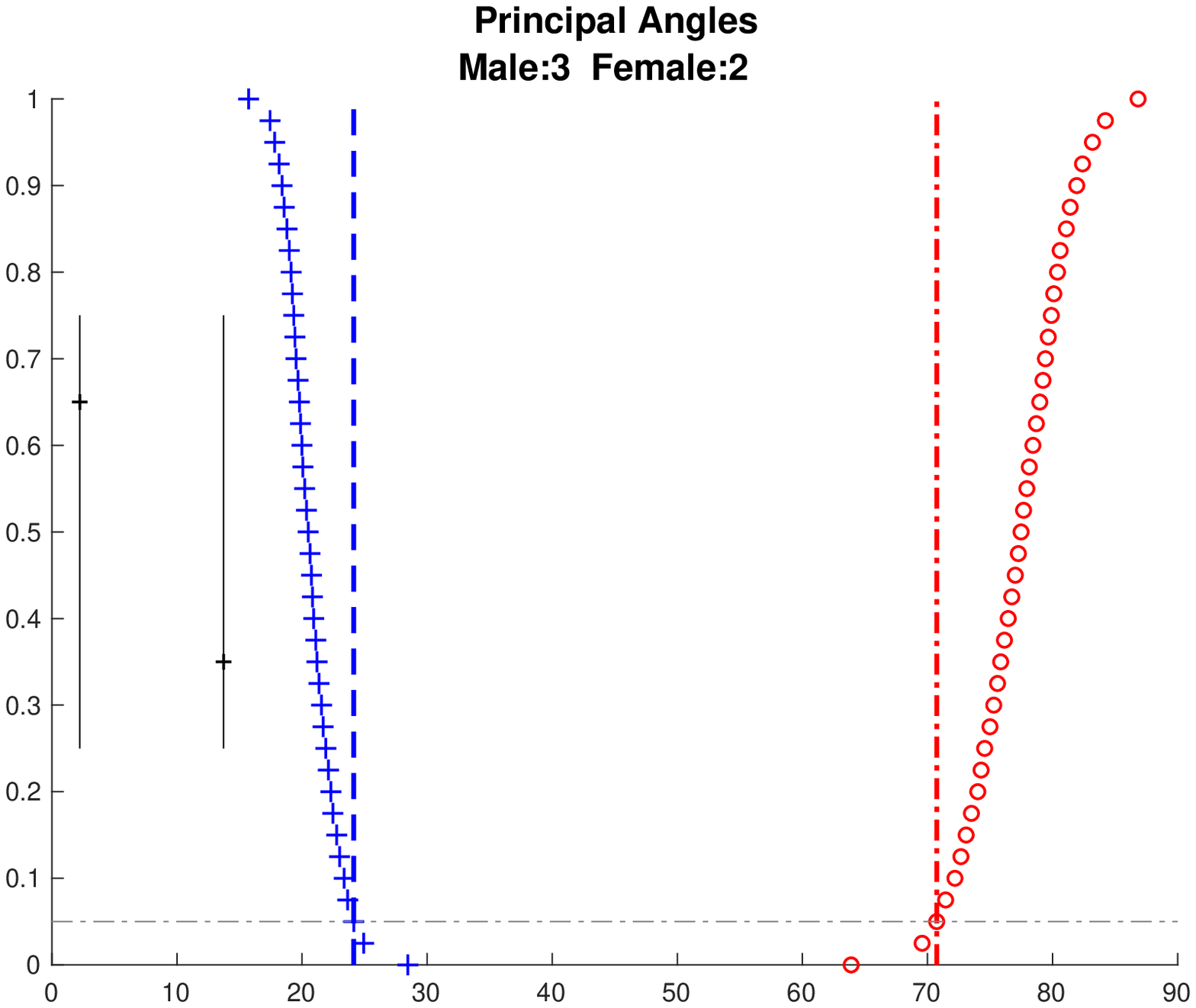}
		\vspace{4ex}
	\end{minipage}
	\begin{minipage}[b]{0.5\linewidth}
		\centering
		\includegraphics[width=0.8\linewidth]{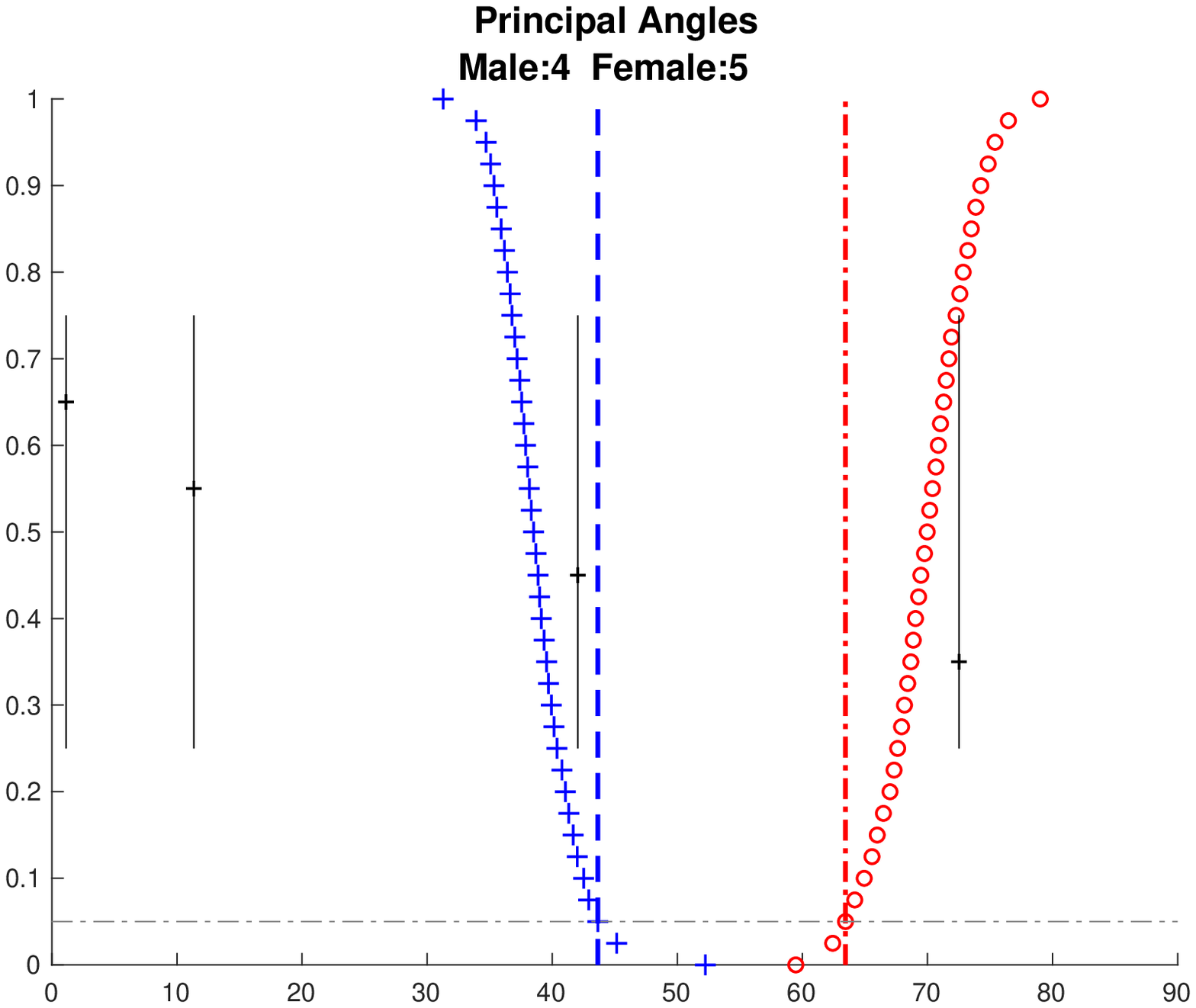}
		\vspace{4ex}
	\end{minipage}
	\caption{Principal angle diagnostic graphics for Spanish mortality data set over various rank choices. Provides the rationale of the rank choice, $\tilde{r}_1 = 3, \tilde{r}_2 = 2$.}
	\label{fig:mortality_angle_diagnostic}
\end{figure}

The resulting AJIVE gives two joint components and one individual component for the male. Since the loading matrices provide important information on the effect of different age groups, block specific analysis together with loading matrices is most informative here. 

Figure~\ref{fig:mortalityjoint1} shows a view of the first joint components for the males (left) and females (right) that is very different from the heat map views used in Section~\ref{c:jive-s:intro-subsec:toy}. While these components are matrices, additional insights come from plotting the rows of the matrices as curves over year (top) and the columns as curves over age (bottom). The curves over year (top) are colored using a heat color scheme, indexing age (black = 0 through red = 40 to yellow = 95 as shown in the vertical color bar on the bottom left). The curves over age (bottom) are colored using a rainbow color scheme (magenta = 1908 through green = 1960 to red = 2002, shown in the horizontal color bar in the top) and use the vertical axis as domain with horizontal axis as range to highlight the fact that these are column vectors. Additional visual cues to the matrix structure are the horizontal rainbow color bar in the top panel, showing that year indexes columns of the data matrix and the vertical heat color bar (bottom) showing that age indexes rows of the component matrix. Because this is a single component, i.e., a rank-one approximation of the data, each curve is a multiple of a single eigenvector. The corresponding coefficients are shown on the right. In conventional PCA/SVD terminology, the upper block specific coefficients are called \emph{loadings}, and are in fact the entries of the left eigenvectors (colored using the heat color scale on the bottom). Similarly, the lower coefficients are called \emph{scores} and are the entries of the right eigenvectors, colored using the rainbow bar shown in the top. 

The scores plots together with the rows as curves plots in Figure~\ref{fig:mortalityjoint1} indicate a dramatic improvement in mortality over time for both males and females. The scores plots are bimodal indicating rapid overall improvement in mortality around the 1950s. This is also visible as the steepest part in the rows as curves plot. Thus the first mode of joint variation is driven by overall improvement in mortality. In addition to the overall improvement, the rows as curves and scores plots also show two major mortality events, the global flu pandemic of 1918 and the Spanish Civil war in the late 1930s. The loading plots together with the columns as curves plots present the different impacts of this common variation on different age groups for males and females.  The loadings plot for males suggests the improvement in mortality is gradually increasing from older towards younger age groups. In contrast, the female block has a bimodal kernel density estimate of the loadings. This shows that females of child bearing age have received large benefits from improving health care. This effect is similarly visible from comparing the female versus male columns as curves. 

\begin{figure}[ht!]
	\begin{minipage}[b]{0.5\textwidth}
		\centering
		\includegraphics[scale=0.5]{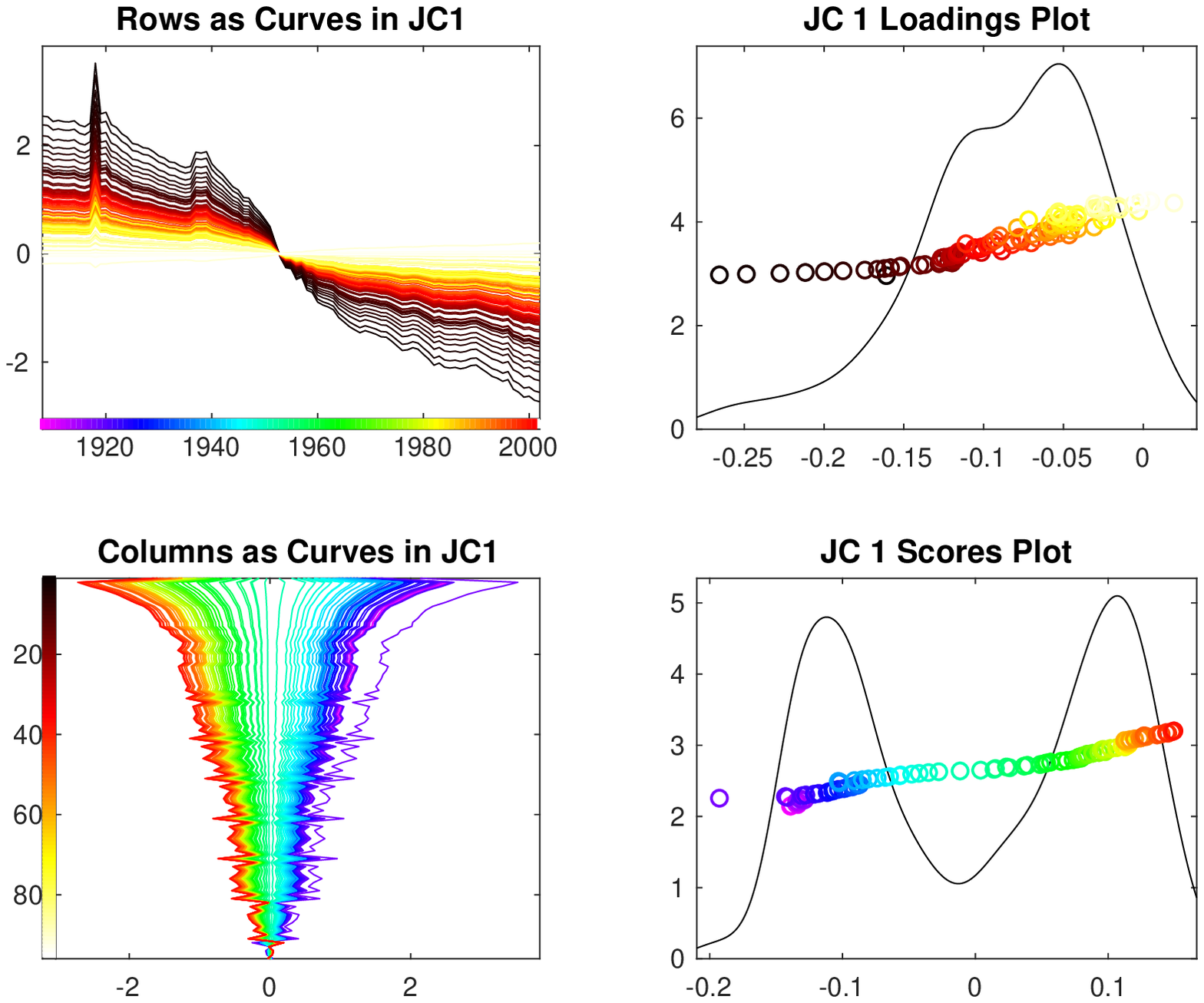}
	\end{minipage}
	\begin{minipage}[b]{0.5\linewidth}
		\centering
		\includegraphics[scale=0.5]{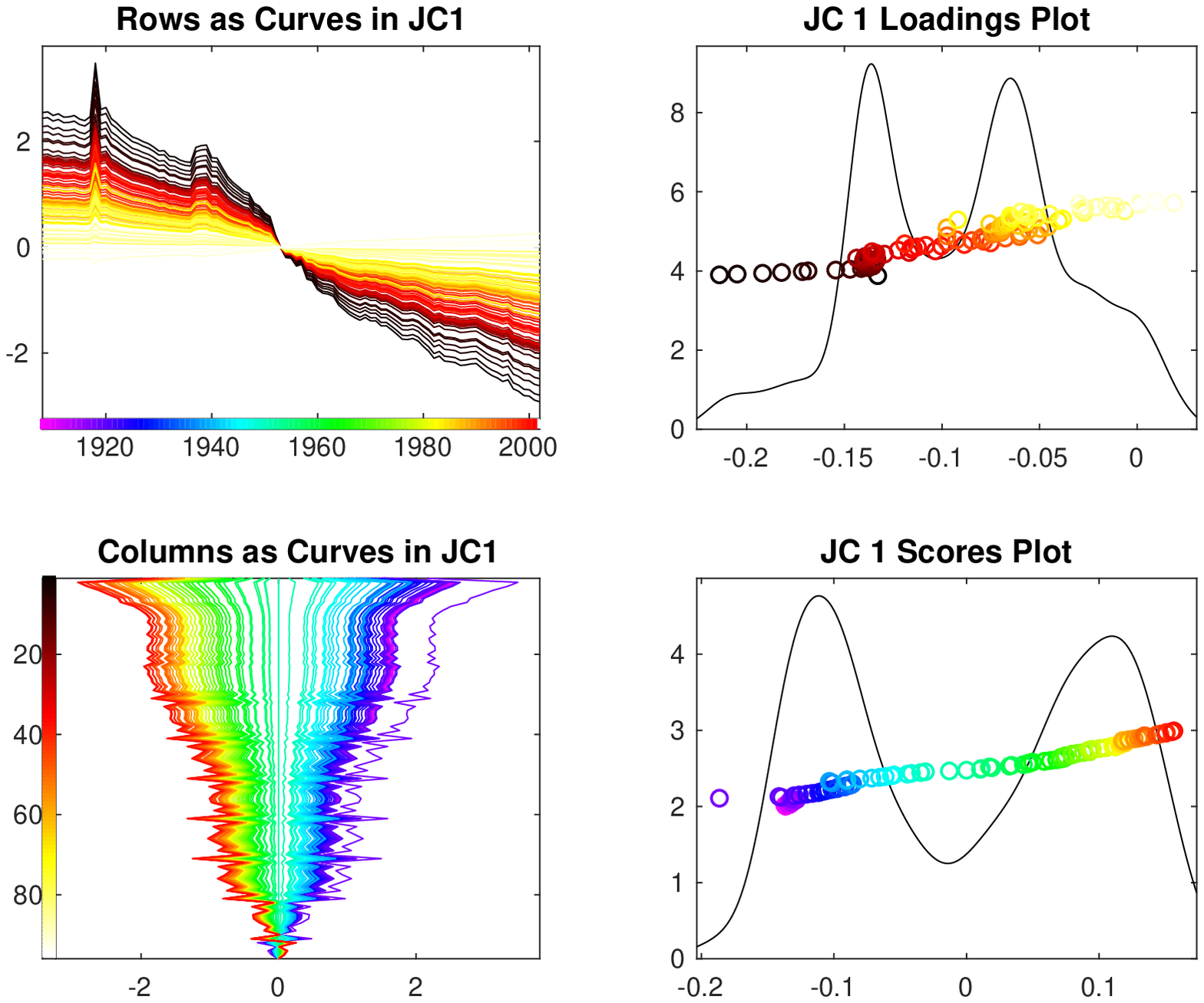}
	\end{minipage}
	\caption{The first block specific joint components of male (left panel) and female (right panel) contain the common modes of variation caused by the overall improvement across different age groups, as can be seen from the scores plots in the right bottom of each panel. The dramatic decrease happened around the 1950s shown in the columns plots. The degree of decrease varies over age groups.}
	\label{fig:mortalityjoint1}
\end{figure}

The second block specific components of joint variation within each gender are similarly visualized in Figure~\ref{fig:mortalityjoint2}. This common variation reflects the contrast between the years around 1950 and the years around 1980 which can be told from the curves in the left top and the colors in the right bottom subplots in both male and female panels. In the scores plots, the green circles, seen on the left end, represent the years around 1950 when automobile penetration started. And the orange to red circles on the right end correspond to recent years, and much improved car and road safety. The loadings plot for males shows that these automobile events had a stronger influence on the 20--45 years old males in terms of both larger values and a second peak in the kernel density estimate. Although this contrast can also be seen in the loadings plot of females, it is not as strong as for the male block. Both loadings plots show an interesting outlier, the babies of age zero. We speculate this shows an improvement in post-natal care that coincidently happened around the same time.  

\begin{figure}[ht!]	
	\begin{minipage}[b]{0.5\textwidth}
		\centering
		\includegraphics[scale=0.5]{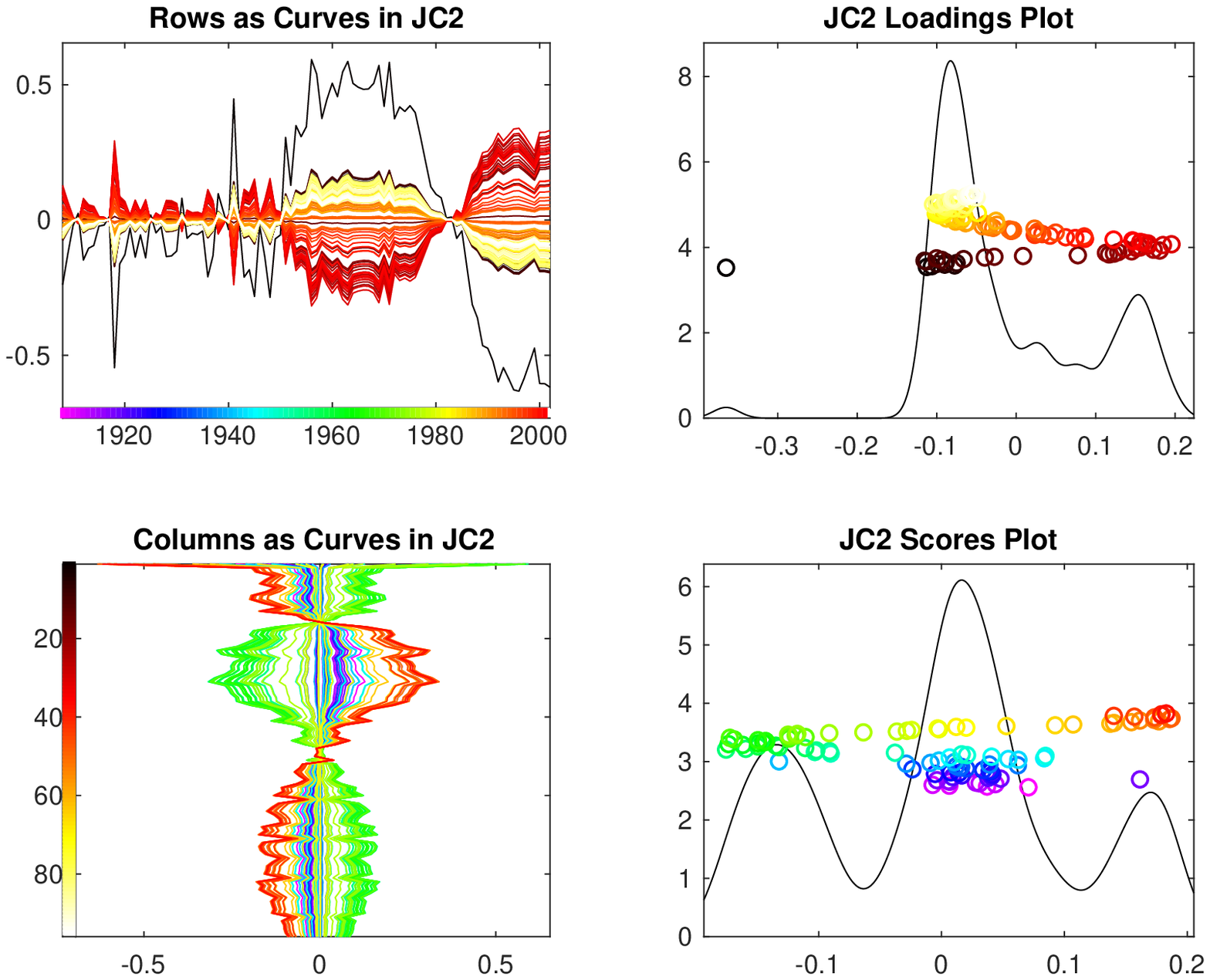}
	\end{minipage}
	\begin{minipage}[b]{0.5\textwidth}
		\centering
		\includegraphics[scale=0.5]{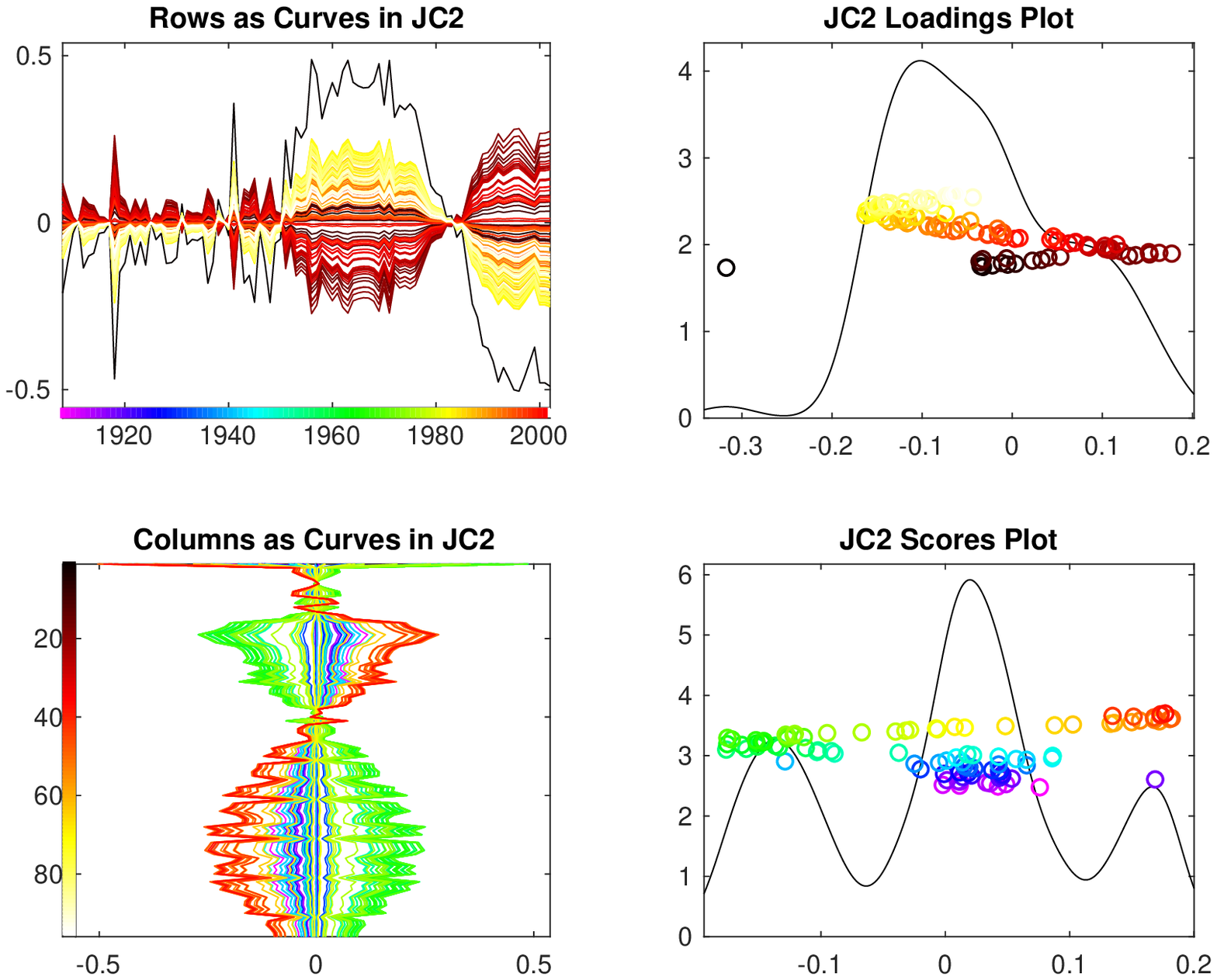}
	\end{minipage}
	\caption{The second joint components of male (left) and female (right) contain the common modes of variation driven by the increase in fatalities caused by automobile penetration and later improvement due to safety improvements. This can be seen from the scores plots in the right bottom. The loadings plots show that this automobile event exerted a significantly stronger impact on the 20--45 males.}
	\label{fig:mortalityjoint2}
\end{figure}

Another interesting result comes from the studying the first individual components for males, shown in Figure~\ref{fig:mortalityindiv}. In the scores plot of males, the blue circles stand out from the rest, corresponding to the years of the Spanish civil war when a significant spike can be seen in the rows as curves plot. Young to middle age groups (typical military age) are affected more than the others as seen in the loadings plot and columns as curves plot.  

\begin{figure}[t!]	
		\centering
		\includegraphics[scale=0.8]{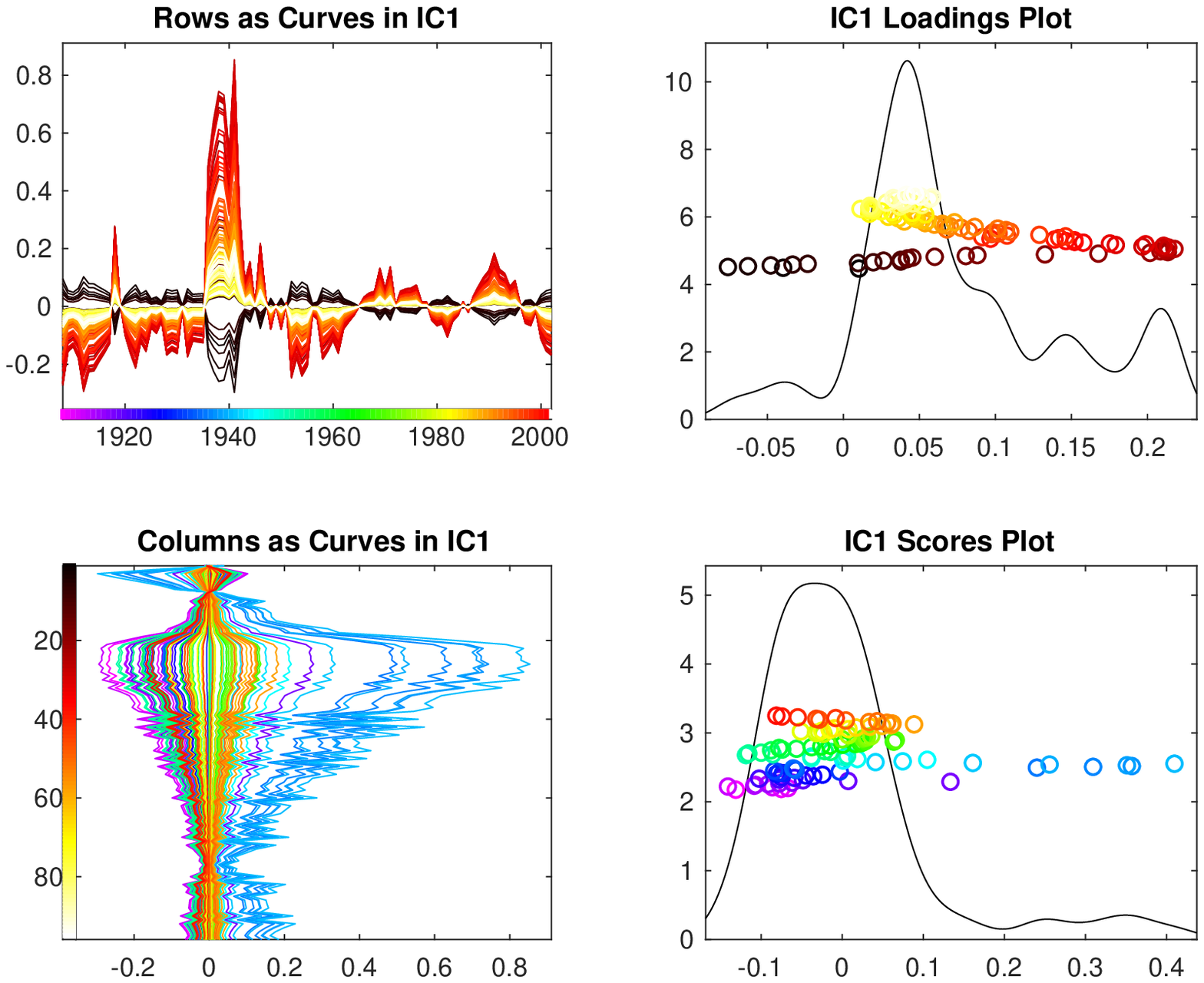}
	\caption{The individual component of male contains the variation driven by the Spanish civil war which can be seen from the blue circles on the right end of the right bottom plot. The Spanish civil war mainly affected the young to middle age males.}
	\label{fig:mortalityindiv}
\end{figure}

\section{Optimization perspective}\label{s:optimization_analysis}

In this section we will investigate how AJIVE compares to PLS, CCA and COBE using the optimization problems that each method is based on. 
Recall that $X_1, \ldots, X_K$ are $(d_k \times n)$ data matrices, with SVD decompositions $X_k = U_{X_k} \Sigma_{X_k} V_{X_k}^\top$, where $ \Sigma_{X_k}$ contains no zeros on its diagonal. To be compatible with AJIVE, we will consider these three algorithms in a non-standard configuration using row spaces. {In Sections~\ref{s:PLS_optimization} and \ref{s:CCAPPA_optimization}, we assume that the matrices $X_k$ are row centered.} We will also use the following notation: for ${{\vec{a}_1 \in \mathbb{R}^{d_1}, \vec{a}_2 \in \mathbb{R}^{d_2}}}$,
\[
\langle \vec{a}_1 X_1, \vec{a}_2 X_2\rangle=\cov(\vec{a}_1 X_1, \vec{a}_2 X_2)=\sqrt{\var(\vec{a}_1 X_1)\var(\vec{a}_2 X_2)}\corr(\vec{a}_1 X_1, \vec{a}_2 X_2).
\]

\subsection{Partial least squares}
\label{s:PLS_optimization}

The PLS finds  linear combinations of rows of $X_1$ and $X_2$ maximizing their sample covariance. More precisely, the PLS identifies a set of pairs of principal vectors, indexed by $i$, obtained sequentially from the following maximization problems:
\begin{equation}
\label{opt:pls}
\begin{aligned}
\{\vec{a}_1^{(i)},\vec{a}_2^{(i)}\}&=\argmax_{\substack{\vec{a}_1 \in \mathbb{R}^{d_1}, \vec{a}_2 \in \mathbb{R}^{d_2}}}\langle \vec{a}_1 X_1, \vec{a}_2 X_2\rangle\\
\mbox{subject to the constraints: } & \|\vec{a}_1\| = 1, \|\vec{a}_2\| = 1,\\
\langle \vec{a}_1 X_1, \vec{a}_1^{(j)} X_1 \rangle = 0,\ 
&\langle \vec{a}_2 X_2, \vec{a}_2^{(j)} X_2 \rangle = 0,\ 
\mbox{for all } j \in \{1, \ldots, i-1\}.
\end{aligned}
\end{equation}
Unlike AJIVE, the directions from PLS are influenced by both variance within data blocks and correlation between the data blocks. In particular, if the signal strength of the individual structure is sufficiently large it might be mistakenly classified as a joint structure by being found ahead of the real joint structure. This phenomenon can be seen in the analysis of the toy example of Section~\ref{c:jive-s:intro-subsec:toy} in \ref{s:toy_details}.

\subsection{Canonical correlation analysis/ Principal angle analysis}
\label{s:CCAPPA_optimization}

Similar to PLS, the  CCA finds  linear combinations of rows of $X_1$ and $X_2$ maximizing their sample correlation. In particular, CCA identifies a set of pairs of canonical vectors obtained sequentially from the optimization problem:
\begin{equation}
\label{opt:cca}
\begin{aligned}
\{\vec{a}_1^{(i)},\vec{a}_2^{(i)}\}&=\argmax_{\substack{\vec{a}_1 \in \mathbb{R}^{d_1}, \vec{a}_2 \in \mathbb{R}^{d_2}}}
\langle \vec{a}_1 X_1, \vec{a}_2 X_2\rangle \\
\mbox{subject to the constraints: }& \|\vec{a}_1 X_1\| = 1, \|\vec{a}_2 X_2\| = 1\\
\langle \vec{a}_1 X_1, \vec{a}_1^{(j)} X_1 \rangle = 0,\ 
& \langle \vec{a}_2 X_2, \vec{a}_2^{(j)} X_2 \rangle = 0,\mbox{ for all }
j \in \{ 1, \ldots, i-1\}.
\end{aligned}
\end{equation}
This form makes the relationship between \eqref{opt:pls} and \eqref{opt:cca} clear and is equivalent to the usual formulation of optimizing the correlation. 

There is an important relationship between CCA and PAA \citep{bjorck1973numerical}, i.e.,  if $\rho_i = \langle \vec{a}_1^{(i)} X_1, \vec{a}_2^{(i)} X_2\rangle$ is the $i$th canonical correlation, $\rho_i = \cos(\theta_{i})$, where $\theta_{i}$ is the $i$th principal angle between row spaces of $X_1$ and $X_2$. The principal vector pairs 
$\{\vec{x}_{1,i}, \vec{x}_{2,i}\}=  \{\vec{a}_1^{(i)} X_1, \vec{a}_2^{(i)} X_2\}$ are often obtained through SVD of $V_{X_1}^\top V_{X_2} $. In particular, let $\vec{u}_{X_1,i}, \vec{u}_{X_2,i}$ be the $i$th left and right singular vectors of $V_{X_1}^\top V_{X_2} $. Then, the $i$th pair of principal vectors are
\[
\vec{x}_{1,i} = \vec{u}_{X_1,i}^\top V_{X_1}^\top ,\quad
\vec{x}_{2,i} = \vec{u}_{X_2,i}^\top V_{X_2}^\top.
\]

An issue with CCA of high-dimensional data is related to the fact that CCA is interested in the canonical vectors $\vec{a}_i$ rather {than} the principal vectors $\vec{x}_i$. In particular, when $d_1 > n, d_2 > n$,  the values of $\vec{a}_i$ in \eqref{opt:cca} are not identifiable due to the singularity of $X_1 X_1^\top$ and $X_2 X_2^\top$. Several approaches have been taken to solve this problem. One approach is to use the Moore--Penrose pseudo inverse to replace the inverse of $X_1 X_1^\top$ and $X_2 X_2^\top$. A second approach is to add a ridge penalty on $X_1 X_1^\top$ and $X_2 X_2^\top$~\citep{vinod1976canonical}. A third approach called penalized CCA is to add penalty functions on $\{\vec{a}_1^{(i)},\vec{a}_2^{(i)}\}$, such as an $\ell_1$ penalty~\citep{parkhomenko2007genome, le2009sparse}, an elastic net~\citep{waaijenborg2008quantifying} or a fused lasso~\citep{witten2009penalized}. Another approach called diagonal penalized CCA is to replace $X_1 X_1^\top$ and $X_2 X_2^\top$ by diag($X_1 X_1^\top$) and diag($X_2 X_2^\top$) \citep{parkhomenko2009sparse, witten2009penalized}.

Another important issue with CCA, which is directly related to AJIVE, is that when $d_1 > n, d_2 > n$, CCA is {generally} driven by noise. \citet{lee2007continuum, samarov2009analysis, lee2016high} study the asymptotic behavior of CCA {in the high-dimension low sample size context} and point out the inconsistency phenomenon in this case. One can solve this issue, like AJIVE and COBE, by replacing {$X_k$} by its low-rank approximation {$\tilde A_1$, $\tilde A_2$}. Recall notation from Eq.~\eqref{equ:step1}.
The $i$th principal vectors are $\vec{p}_i = \tilde{V}_1\vec{u}_{1,i}$, $\vec{q}_i = \tilde{V}_2\vec{u}_{2,i}$, where $\vec{u}_{j,i}$ is the $i$th singular vector of $\tilde U_i$  of the SVD of $\tilde{V}_1^\top \tilde{V}_2$ respectively.

As discussed in Section~\ref{c:jive-s:method}, AJIVE uses an equivalent principal angle calculation based on SVD of $M = [\tilde V_1, \tilde V_2]^\top = U_M \Sigma_M V_M^\top$  \citep{miao1992principal}. 
AJIVE uses the transpose of the $i$th right singular vector, $V_{M, i}^\top$, as the estimated $i$th basis vector of the joint space, provided that the $i$th principal angle is smaller than the threshold derived in Section~\ref{c:jive-s:method-subsec:step2-2blocks}.  Moreover, if the $i$th principal angle has a value distinct from other principal angles, then the $i$th left singular vector of $M$ can be written as $U_{M, i} = [\vec{u}_{1,i}^\top, \vec{u}_{2,i}^\top]^\top/\sqrt{2}$.  Consequently
\[
V_{M, i}^\top = \frac{1}{\sigma_{M,i}} \, U_{M, i}^\top M 
= \frac{1}{\sqrt{2}\sigma_{M,i}} \, {(\vec{u}_{1,i}^\top \tilde{V}_1^\top + \vec{u}_{2,i}^\top \tilde{V}_2^\top)}
= \frac{1}{\sqrt{2}\sigma_{M,i}} \, {(\vec{p}_{i}^\top + \vec{q}_{i}^\top)}.
\]
This shows that the AJIVE direction $V_{M, i}^\top$ is the scaled sum of the $i$th pair of principal vectors.

CCA applied to the low-rank approximations {$\tilde A_k$} and AJIVE are therefore closely related.
However, AJIVE provides one joint vector per two distinct principal vectors that by the virtue of being an average should be a better estimate of the joint space than either of the principal vectors. More importantly, AJIVE uses a theoretically sound threshold of the principal angles that allows us to segment individual and joint variation. 

The AJIVE formulation allows for a natural extension to multi-block situations.
Several approaches of Multiset Canonical Correlation Analysis (mCCA) have been developed as extensions of CCA~\citep{horst1961relations, kettenring1971canonical, nielsen2002multiset}. There is no general consensus on which of these extensions is preferable.  We point out that AJIVE is closely related to one of the mCCA discussed in \citet{nielsen2002multiset}. 

This version of mCCA is defined using the optimization problem for the $i$th set of canonical vectors $\{ \vec{a}_1^{(i)}, \ldots, \vec{a}_K^{(i)} \}$ and corresponding principal vectors (also called canonical {variables}) $\{ \vec{a}_1^{(i)} X_1, \ldots, \vec{a}_K^{(i)} X_K \}$:
\begin{equation}
\label{opt:mcca}
\begin{aligned}
\{\vec{a}_1^{(i)}, \ldots, \vec{a}_K^{(i)} \}&= \argmax_{\substack{\vec{a}_1, \ldots, \vec{a}_K}} \sum_{\substack{1 \leq k,l \leq K}} \langle \vec{a}_k X_k, \vec{a}_l X_l \rangle\\
\mbox{subject to the constraints: }& \sum_{\substack{k = 1}}^{K} \|\vec{a}_k X_k\|_2^2 = 1,\\
\langle \vec{a}_k X_k, \vec{a}_k^{(j)} X_k \rangle = 0, & \quad \mbox{for all } k \in \{ 1, \ldots, K\},\quad  j \in\{ 1,\ldots, i-1\}.
\end{aligned}
\end{equation}
Notice that the constraint in \eqref{opt:mcca} is different than the perhaps more natural $ \|\vec{a}_k X_k\|_2^2 = 1$ for all $k \in \{ 1, \ldots, K\}$.

If the $i$th singular value corresponding to the AJIVE direction $V_{M, i}^\top$ has a value distinct from other singular values in the AJIVE SVD, then calculations similar to the two block case show that the $i$th basis vector of the joint space from AJIVE
\[
V_{M, i}^\top=\frac{1}{\sigma_{M,i}} \sum_{k = 1}^{K} \vec{a}_k^{(i)} X_k
\]
is the scaled sum of the corresponding canonical variables. In fact, $V_{M, i}^\top$ is the $i$th flag mean of the row spaces of $X_1, \ldots, X_K$, as defined by \citet{draper2014flag}, which thus is a building block of AJIVE.

\subsection{Common orthogonal basis extraction}

\citet{zhou2015group}  proposed a compelling optimization problem for finding the common orthogonal basis (COBE). It is based on iteratively solving
\begin{equation}
\label{opt:cobe}
\begin{aligned}
\bar{a}_i=&\argmin_{{\bar{a},\ z_{i,k}, k = 1, \ldots, K}}  
\sum_{k=1}^{K} \|\tilde V_k z_{i,k} - \bar{a}\|^2\\
\mbox{subject to the constraints: } &\|\bar{a}\|_2 = 1, \langle \bar{a}, \bar{a}_j\rangle=0,\mbox{ for all } j \in \{1,\ldots,i-1\}.
\end{aligned}
\end{equation}
To compare COBE to AJIVE we first simplify the objective function of \eqref{opt:cobe} to
$$\begin{aligned}
\sum_{k=1}^{K} \|\tilde V_k z_{i, k}  - \bar{a}_i\|_2^2
= 	\sum_{k=1}^{K} \| \tilde V_k z_{i,k}\|_2^2 + K \|\bar{a}_i\|_2^2 - 2 \sum_{k=1}^{K} \langle  \tilde V_k z_{i,k}, \bar{a}_i\rangle 
= \|z_i\|_2^2 + K \|\bar{a}_i\|_2^2 - 2  z_i^\top M \bar{a}_i.
\end{aligned}$$
where $z_i = [z_{i,1}, \ldots, z_{i,K}]$. 
If we fix the value of $\| z_i\|$ we see that the solution to the optimization problem (\ref{opt:cobe}) is the same as SVD of $M$ with $\bar{a}_i = V_{M,i}$. Moreover this solution is invariant in $\|z_i\|$.

Thus the optimization problem (\ref{opt:cobe}) gives the same result as AJIVE. However, because AJIVE uses well optimized SVD rather than a heuristic iteration algorithm, AJIVE is much faster than the COBE algorithm. Moreover, COBE lacks any principally based standard on how to choose the threshold for selecting the joint space.

To understand why this is a serious issue consider the results of applying COBE to the TCGA data discussed in detail in the next section. To make comparisons fair we provided COBE the same selected first stage ranks for each data block as AJIVE. COBE's default threshold for separating joint and individual structure of $0.01$ is too low to find any joint component. Therefore we tried raising the  default threshold 0.01 to 1, in which case COBE fails to finish on our computer due to its inefficient handling of high dimensional data.

\section*{Acknowledgments}
This research was supported in part by the National Science Foundation under
Grant Nos 1016441, 1512945 and 1633074.

\appendix

\section{Proofs}
\label{s:proofs}
\begin{proof}[Proof of Lemma~\ref{lemma-exist}]
	
	Define the row subspaces respectively for each matrix $A_k$ as $\row(A_k) \subseteq \mathbb{R}^n$.  For non-trivial cases, define a subspace $\row(J) \neq \{\vec{0}\}$ as the intersection of the row spaces $\{\row(A_1), \ldots, \row(A_K)\}$, i.e.,
	$$\row(J) \triangleq \bigcap\limits_{k=1}^{K} \row(A_k).$$
	For each matrix $A_k$, two matrices $J_k$, $I_k$ can be obtained by projection of $A_k$ on $\row(J)$ and its orthogonal complement in the row space $\row(A_k)$. Thus the two matrices satisfy $J_k + I_k = A_k$ and their row subspaces are orthogonal with each other, i.e., $\row(J) \perp \row(I_k)$, for all $k \in \{ 1, \ldots, K\}$. Then the intersection of the row subspaces $\{\row(I_1), \ldots, \row(I_K)\}$, $\bigcap\limits_{k=1}^{K} \row(I_k)$, has a zero projection matrix. Therefore, we have $\bigcap\limits_{k=1}^{K} \row(I_k) = \{\vec{0}\}$ and have obtained a set of matrices simultaneously satisfying the stated constraints.	
	
	On the other hand, it follows from the assumptions that the row space $\row(A_k)$ is spanned by the union of basis vectors of $\row(J_k)$ and $\row(I_k)$, which indicates
	$$ \row(J) = \bigcap\limits_{k=1}^{K} \row(A_k).$$
	Accordingly, the matrices $J_1, \ldots, J_K$ and $I_1, \ldots, I_K$ are also uniquely defined.
	
\end{proof}

\begin{proof}[Proof of Lemma~\ref{lemma-bound}]
	Let $P_{1}$ and $P_{2}$ be the projection matrices onto the individually perturbed joint row spaces. And let $P$ be the projection matrix onto the common joint row space $J$. Thus, we have
	\begin{align*}
	\sin\theta & = \|(I - P_{1})P_{2} \|  \leq \|(I - P_{1})(I - P)P_{2}\| + \|(I - P_{1})PP_{2}\| \\
	& \leq \|(I - P_{1})(I - P)\| \, \|(I - P)P_{2}\| + \|(I - P_{1})P\| \, \|PP_{2}\|,
	\end{align*}
	in which $\|(I - P_{1})P\| = \sin\theta_{1,1}$, $\|(I - P_{1})(I - P)\| = \cos\theta_{1}$, $\|(I - P_{2})P\| = \sin\theta_{2, 1}$ and $\|(I - P_{2})(I - P)\| = \cos\theta_{2, 1}$.
	Therefore,
	\begin{equation*}
	\sin\phi \leq \cos\theta_{1, 1}\sin\theta_{2, 1} + \sin\theta_{1, 1}\cos\theta_{2, 1} = \sin(\theta_{1, 1} + \theta_{2, 1}).
	\end{equation*}
\end{proof}

\begin{proof}[Proof of Lemma~\ref{lemma-bound-multi}]	
	Notation from (\ref{equ:JointM2}) and (\ref{equ:JointM}) is used here. For each singular value $\sigma_{M,i}$, it can be formulated as a sequential optimization problem i.e
	$$ \sigma_{M,i}^2 = \textnormal{max}_{\substack{Q}} \|MQ\|^2_{F} = \textnormal{max}_{\substack{Q}} \sum_{k=1}^K\|\tilde{V}_{1}^\top Q\|^2_{F},$$
	where $Q$ is a rank $1$ projection matrix that is orthogonal to the previous $i-1$ optima, i.e., $Q_{1}, \ldots, Q_{i-1}$. The $Q$ that maximizes the Frobenius norm of $MQ$ is denoted as $Q_i$. 
	
	For an arbitrary component in the theoretical joint score subspace $\row(J)$, write its projection matrix as $P_J^{(1)}$. The Frobenius norm of $M$ projected onto $P_J^{(1)}$ is 
	\begin{equation*}
	\|MP_J^{(1)}\|^2_{F}   =  \begin{bmatrix}
	\tilde{V}_{1}^\top P_J^{(1)} \\
	\vdots \\
	\tilde{V}_{K}^\top P_J^{(1)} \\
	\end{bmatrix}_F^2
	\geq  \begin{bmatrix}
	\cos{\theta_1} \\
	\vdots \\
	\cos{\theta_K} \\
	\end{bmatrix}_F^2
	= \sum_{k=1}^{K} \cos^2\theta_{k}
	\end{equation*}
	
	Considering the mechanism of SVD, $\sigma_{M, 1}^2$ is the maximal norm obtained from the optimal projection matrix $Q_1 \subseteq \bigcup_{k=1}^{K}\row(\tilde{A}_k) \subseteq \mathbb{R}^n$. If all $\tilde{A}_k$ contain all components obtained by noise perturbation of the common row space $\row(J)$, then we have 
	$$
	\sigma_{M, 1}^2 \geq \|MP_J^{(1)}\|^2_{F} \geq \sum_{k=1}^{K} \cos^2\theta_{k}.
	$$
	to be considered as a component of the joint score subspace. 
	
	This argument can be applied sequentially. For the $Q_2 \in Q_1^{\perp} \cap \{\bigcup_{k=1}^{K}\row(\tilde{A}_k)\}$, there exist a non-empty joint subspace ($\subseteq \row(J)$) such that all $ Q_1^{\perp} \cap \row(\tilde{A}_k)$ contain perturbed directions of a joint component other than the one above. Therefore this joint component with projection matrix $P_J^{(2)}$ should have
	$$\sigma_{M, 2}^2 \geq \|MP_J^{(2)}\|^2_{F} \geq \sum_{k=1}^{K} \cos^2\theta_{k}.$$
	Thus the singular values corresponding to the joint components satisfies (\ref{equ:Mbound}) and this procedure can continue through at least $r_J$ steps.
\end{proof}

\section{Details of the toy example}
\label{s:toy_details}

Section~\ref{c:jive-s:intro-subsec:toy} introduces a toy example of two data blocks, $X$ ($100\times 100$) and $Y$ ($\mbox{10,000}\times 100$), with patterns corresponding to joint and individual structures. For details see Figure~\ref{fig:jive:toyrawdata}. 

A naive attempt at integrative analysis can be done by concatenating $X$ and $Y$ on columns and performing a singular value decomposition on this concatenated matrix. Figure~\ref{fig:jive:toysvdoutput} shows the results for three choices of rank. The rank-$2$ approximation essentially captures the joint variation component and the individual variation component of $X$, but the $Y$ components are hard to interpret. The bottom $2000$ rows show the joint variation but the top half of $Y$ reveals signal from the individual component of $X$.  One might hope that the $Y$ individual components would show up in the rank-$3$ and rank-$4$ approximations. However, because the noise in the $X$ matrix is so large, a random noise component from $X$ dominates the $Y$ signal, so the important latter component disappears from this low-rank representation unlike the AJIVE result in Figure~\ref{fig:jive:toyjiveoutput}. In this example, this naive approach completely fails to give a meaningful joint analysis. 

\begin{figure}[t!]
	\centering
	\vspace{0.1in}
	\includegraphics[scale = 0.75]{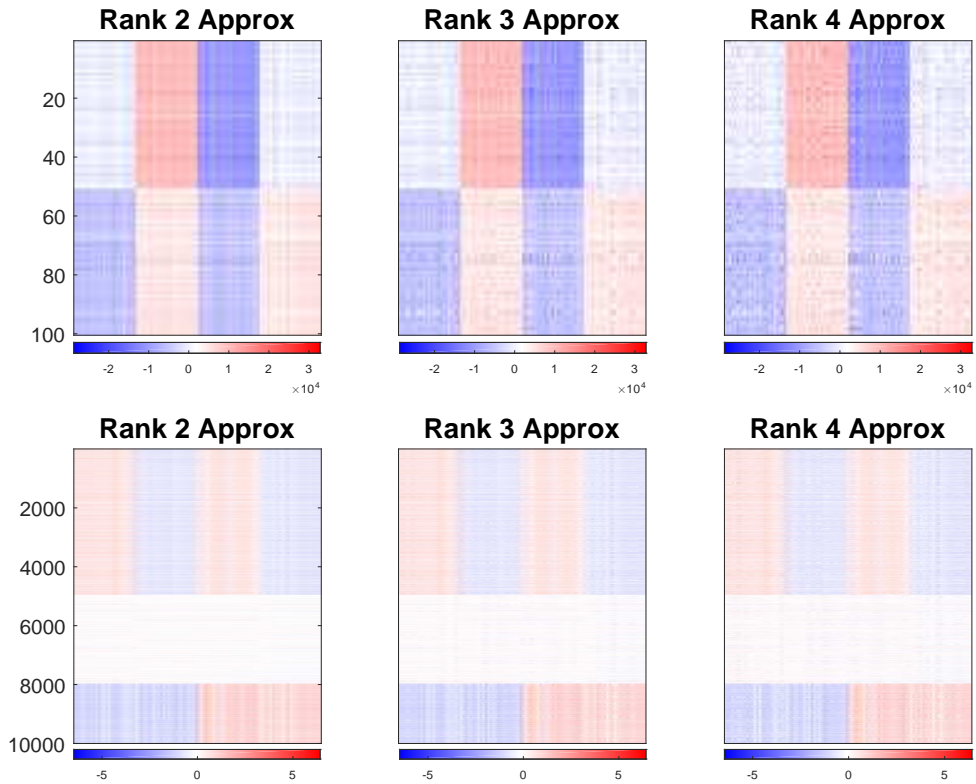}
	\caption[SVD approximations of concatenated toy data blocks]{Shows the concatenation SVD approximation of each block for rank 2 (left), 3 (center) and 4 (right). Although block $X$ has a relatively accurate approximation when the rank is chosen as 2, the individual pattern in block $Y$ has never been captured due to the heterogeneity between $X$ and $Y$.}
	\label{fig:jive:toysvdoutput}
\end{figure}

Figure~\ref{fig:jive:toyplsoutput} presents the PLS approximations with different numbers of components selected. PLS completely fails to separate joint and individual components. Instead it provides mixtures of the joint, and some of the individual components. Increasing the rank of the PLS approximation only includes more noise. 

\begin{figure}[htp!]
	\centering
	\includegraphics[scale = 0.75]{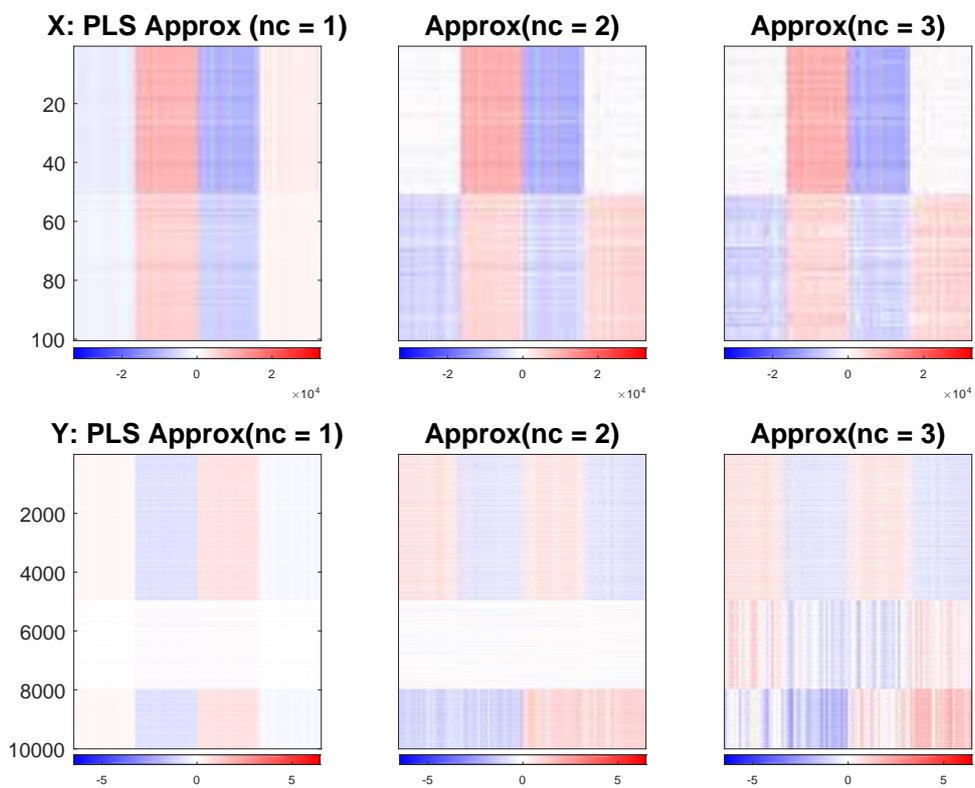}
	\caption[PLS approximations of the toy data]{PLS approximations of each block for numbers of components as 1 (left), 2 (center) and 3 (right). PLS fails to distinguish the joint and individual variation structure.}
	\label{fig:jive:toyplsoutput}
\end{figure}

The \citet{lock2013joint} method, called JIVE here, is applied to this toy data set. We implemented the JIVE algorithm using the \textsf{R} package \texttt{r.jive}~\citep{OConnell:2016du} without the orthogonality constraint. The \texttt{jive} function provides two options for rank selection: using a permutation test and the Bayesian Information Criterion, respectively. However, neither of them segmented joint signal properly. When using the Bayesian Information Criterion approach, no joint signal was identified and the true joint signals were labeled as noise. The permutation test approach gave a reasonable approximation of the total signal variation within each data block as in the left panel of Figure~\ref{fig:jive:toyoldjiveoutput}. However, the \citet{lock2013joint} method gave rank-$2$ approximations to the joint matrices shown in the middle panel. The approximation consists of the real joint component together with the individual component of $X$. Consequently, the approximation of the $X$ individual matrix is a zero matrix and a wrong approximation of the $Y$ individual matrix is shown in the top half of the right panel. We speculate that failure to correctly apportion the joint and individual variation is caused by the fact that the individual spaces are correlated, because the permutation test does not handle correlated individual signals very well. We also manually specified the correct joint and individual ranks for $X$ and $Y$, which results in the correct results.

We finally remark that the \citet{zhou2015group} method COBE correctly segments the toy example. However it takes significantly (39 times) longer time than AJIVE to do so.

\begin{figure}[htp!]
	\centering
	\vspace{0.1in}
	\includegraphics[scale = 0.75]{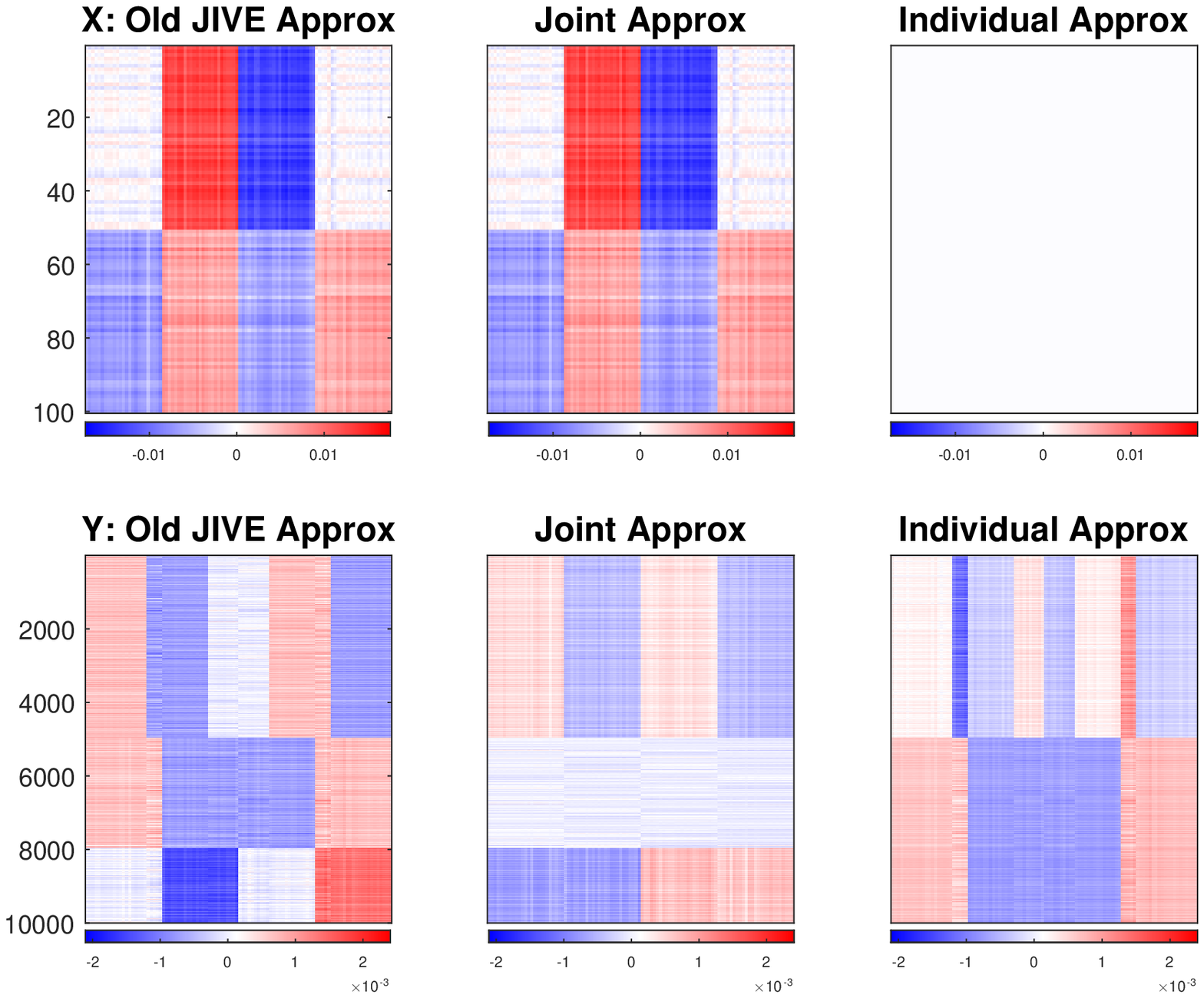}
	\caption[The old JIVE approximation of the toy data]{ The \citet{lock2013joint} JIVE method approximation of the data blocks $X$ and $Y$ in the toy example are shown in the first panel of figures. The joint matrix approximations (middle panel) incorrectly contain the individual component of $X$ because of the failure of the permutation test to correctly select ranks in the presence of correlated individual components.}
	\label{fig:jive:toyoldjiveoutput}
\end{figure}

\end{document}